\documentclass[journal, 10pt, onecolumn]{IEEEtran}

\ifCLASSINFOpdf
\else
\fi

	\usepackage{amssymb,amsbsy,graphicx,amsmath,amsthm,amscd}
	\usepackage[margin=1in]{geometry}
	\usepackage{url}
	
	\usepackage[noadjust]{cite}

	\usepackage{mathptmx}       
	\usepackage{helvet}         
	\usepackage{courier}        
	\usepackage{type1cm}        
	\usepackage{graphicx}     
	
	\usepackage{mathtools, nccmath}

	\usepackage{verbatim}                       
	\usepackage[utf8]{inputenc}
	
	\usepackage{theoremref}
	\usepackage{bbm}
	\usepackage{xcolor}
	\usepackage{algorithm,algpseudocode}

	\allowdisplaybreaks

	\usepackage{fullpage}
	\usepackage{epsfig}
	\usepackage{graphicx}
	\usepackage{amsmath}
	\usepackage{amssymb}
	\usepackage{color}
	\usepackage{pdflscape}
	\usepackage{algorithm,algpseudocode}%
	\usepackage{booktabs}
	\usepackage{comment}
	\usepackage{threeparttable}
	\usepackage{cancel}
	\usepackage{hyperref}
	
	\usepackage{amssymb}
	\usepackage{pifont}

	\usepackage{fullpage}
	\usepackage{epsfig}
	\usepackage{graphicx}
	\usepackage{amsmath}
	\usepackage{amssymb}
	\usepackage{color}
	\usepackage{pdflscape}
	\usepackage{algorithm,algpseudocode}%
	\usepackage{booktabs}
	\usepackage{comment}
	\usepackage{threeparttable}
	\usepackage{cancel}
	\usepackage{hyperref}
	
	\usepackage{amssymb}
	\usepackage{pifont}
	\newcommand{\cmark}{\ding{51}}%
	\newcommand{\xmark}{\ding{55}}%

	\newcommand{\R}{\mathbb R}

	\newcommand{\st}{\text{subject to}}

	\newcommand{\supp}{\mbox{supp}}

	\newcommand{\E}{\mathbb{E}}

	\newcommand{\argmin}[1]{\underset{#1}{\mathbf{argmin}}}

	\newcommand{\inner}[1]{\left\langle #1\right\rangle}
	\newcommand{\red}[1]{{\color{red}#1}}

	\newtheorem{assumption}{Assumption}
	\newtheorem{thm}{Theorem}
	
	\newtheorem{lem}[thm]{Lemma}
	
	\newtheorem{cor}[thm]{Corollary}

	\newtheorem{rem}{Remark}
	
	\algnewcommand{\Input}{%
		\State \textbf{Input:}
	}
	\algnewcommand{\Initialize}{%
		\State \textbf{Initialize:}
	}
	\algnewcommand{\Output}{%
		\State \textbf{Output:}
	}
	\algnewcommand{\ForLoop}{%
		\textbf{For $j=1,2,...,m$ do}
	}
	\algnewcommand{\ForLoopmod}{%
		\textbf{For $L=0,1,...,L-1$ do}
	}
	\algnewcommand{\ForEnd}{%
		\textbf{end for}
	}
	\algnewcommand{\Iterate}{%
		\State \textbf{Iterate:}
	}
	
	\newcommand{\indt}{ \ \ \ \ \ \ \ }
	
	\allowdisplaybreaks
	
\begin{document}
		
		\title{Federated Gradient Matching Pursuit }
		\author{
			Halyun Jeong, Deanna Needell\thanks{H. Jeong and D. Needell are with the Department of Mathematics, the University of California Los Angeles, Los Angeles, CA 90095 (email: hajeong@math.ucla.edu; deanna@math.ucla.edu)},
			\and Jing Qin\thanks{J. Qin is with the Department of Mathematics, University of Kentucky, KY 40506 (email: jing.qin@uky.edu)}}
		\maketitle

		\begin{abstract}
			Traditional machine learning techniques require centralizing all training data on one server or data hub. Due to the development of communication technologies and a huge amount of decentralized data on many clients, collaborative machine learning has become the main interest while providing privacy-preserving frameworks. In particular, federated learning (FL) provides such a solution to learn a shared model while keeping training data at local clients. 
			On the other hand, in a wide range of machine learning and signal processing applications, the desired solution naturally has a certain structure that can be framed as sparsity with respect to a certain dictionary. This problem can be formulated as an optimization problem with sparsity constraints and solving it efficiently has been one of the primary research topics in the traditional centralized setting. 
			In this paper, we propose a novel algorithmic framework, federated gradient matching pursuit (FedGradMP), to solve the sparsity constrained minimization problem in the FL setting. 
			We also generalize our algorithms to accommodate various practical FL scenarios when only a subset of clients participate per round, when the local model estimation at clients could be inexact, or when the model parameters are sparse with respect to general dictionaries. 
			
			Our theoretical analysis shows the linear convergence of the proposed algorithms. A variety of numerical experiments are conducted to demonstrate the great potential of the proposed framework -- fast convergence both in communication rounds and computation time for many important scenarios without sophisticated parameter tuning.
		\end{abstract}
		
		
		
		
		\section{Introduction}
		With the development of technology and science, machine learning for big data processing has become an emerging field with a wide variety of applications. In general, there are several major considerations in dealing with a large amount of data - data storage and privacy, computation, and communication \cite{li2020federated}. To address the limitation of efficiency and scalability of traditional machine learning algorithms for large-scale data, distributed centralized learning allows data and/or model parallelism, where all local data are typically uploaded to a central server but model training is distributed to various clients \cite{verbraeken2020survey}. Different from the traditional centralized learning, federated learning (FL) \cite{brenden2017} is a collaborative learning framework in which many clients work together to solve an optimization problem without sharing local data.
		In order to preserve privacy, and reduce the communication cost between clients and the server, data sets are only stored at the clients locally and can not be transferred to other clients or the server in FL. In other words, it aims to learn a central model using decentralized data sets. The heterogeneous data distributions among the clients pose an additional challenge in federated learning. 
		
		As one of the most popular FL algorithms, Federated Averaging (FedAvg)~\cite{mcmahan2017communication} considers an unconstrained optimization problem where the desired solution has no additional characteristics. FedAvg alternates  gradient descent and averaging of distributed solutions from local clients in an iterative way. However, in a lot of applications, the solution of interest has some special structures, such as sparsity and low-rankness by itself or in some transformed domain, which serve as prior information and can be utilized to address the ill-posedness of the problem and improve performance. Thus, recent interest in FL optimization with additional solution structures has grown, which has been shown to be especially effective when only a few data samples are available at each client but the underlying signal dimension is relatively large \cite{yuan2021federated, tong2020federated}. 
		
		In such cases when the solution to an optimization problem from FL applications possesses a certain structure, for example, sparsity and low-rankness, one can use a regularizer to enforce the desired structure \cite{yuan2021federated}. Federated Dual Averaging (FedDualAvg) by Yuan et al. \cite{yuan2021federated}, different from FedAvg, uses potentially nonsmooth and convex regularizers to promote the structure of the solution. 
		
		When we have more prior information about the solution structure, e.g., the sparsity level, then hard-thresholding based approaches could be often more efficient than the regularization based methods \cite{zhou2018efficient, goodfellow2016deep}. The hard-thresholding based methods aim to solve nonconvex formulations of the problem, which have been successfully applied to many data processing problems lately with improved performance. \cite{needell2009cosamp, blumensath2009iterative, rauhut2017low, grotheer2021iterative, foucart2013invitation}. 
		
		Following this line of research, Tong et al. proposed Federated Hard Thresholding (FedHT) and Federated Iterative Hard Thresholding (FedIterHT) \cite{tong2020federated}, employing hard-thresholding at the aggregation step at the server with potentially additional hard-thresholding after each stochastic gradient step at clients. With a proper choice of step sizes for the stochastic gradients at the clients, these methods guarantee linear convergence up to a neighborhood of the solution to the problem. 
		Despite the fact that these approaches partially inherit the advantages of thresholding based methods over those based on regularization, they necessitate fine-tuning of learning rates at clients, which often have practical limitations and they are not applicable for sparse signals with respect to general dictionaries. Their convergence analysis also requires the mini-batch sizes at the clients to grow exponentially in the number of communication rounds, which further limits the usage in most applications. 
		
		Another popular thresholding based method is the gradient matching pursuit (GradMP) \cite{nguyen2017linear} which is a generalization of the compressive sampling matching pursuit (CoSaMP) \cite{needell2009cosamp}. These methods are known to be more efficient than the others such as regularizer based methods, particularly when the sparsity level of a signal is much smaller than its dimension \cite{foucart2013invitation, shen2017tight}.

		\subsection{Contributions}
		We summarize our contributions below.
		
		\begin{itemize}
			\item We propose the Federated Gradient Matching Pursuit algorithm, abbreviated as FedGradMP, to overcome the aforementioned drawbacks. More precisely, we show that the proposed FedGradMP  enjoys the linear convergence up to a small neighborhood of the solution, without the  restrictions for FedHT/FedIterHT to work. Furthermore, Our analysis has shown that FedGradMP converges linearly up to a statistical bias term for the recovery of sparse signals  under mild conditions. 
			
			\item The majority of FL algorithm analyses have been carried out either under bounded gradient or bounded dissimilarity assumptions, which could be problematic for certain scenarios \cite{wang2021field}. Only a few recent works for FL in unconstrained setup provide theoretical guarantees without this type of assumption, but under the unbounded dissimilarity condition which is considered to be the most general type of heterogeneity assumption \cite{khaled2020tighter,woodworth2020minibatch}. 
			To the best of our knowledge, this is the first work for solving sparsity-constrained FL problems that analyzes the convergence under this general dissimilarity condition. 
			
			\item Thanks to the mechanism of GradMP, FedGradMP does not require intensive tuning of learning rates at the clients for the sparse linear regression problem, which could be often still challenging because of data heterogeneity in the FL setting. 
			Approaches based on the local stochastic gradient at clients including FedHT/FedIterHT, as described in the literature \cite{haddadpour2019convergence,wang2021field,tong2020federated} and demonstrated in our numerical studies, need tweaking the learning rates (step sizes) else they diverge or converge slowly, especially when the data at distinct clients are more heterogeneous. 
			In contrast, FedGradMP which is based on solving low-dimensional sub-optimization problems at clients can be
			often solved efficiently and does not require fine tuning of learning rates. 
			

			\item Most of the signals of practical interest are not sparse by themselves in the standard basis but in a certain dictionary. This observation has led to the development of several sparse recovery methods with general dictionaries in the centralized setting \cite{davenport2013signal,nguyen2017linear,baraniuk2018one}. FedGradMP is a versatile method under a general dictionary framework. One potential problem with using dictionaries in FL methods is the privacy concern if they are correlated with the client data sets. By utilizing dictionaries that are statistically independent with client data sets such as the random Gaussian dictionary, we demonstrate the effectiveness of FedGradMP as an FL method without such concerns.
			
			
		\end{itemize}
		

		\begin{table*}[t]
			\scriptsize
			
			\begin{center}
				\begin{threeparttable}
					\begin{tabular}{lc@{\quad}c@{\quad}c@{\quad}c@{\quad}c }
						\toprule
						& Dictionary sparsity and & Linear convergence to the solution & No client & \\
						& convergence speed-up &  up to  optimal statistical bias & LR tuning 
						& Unbounded dissimilarity \\
						\midrule
						FedAvg ~\cite{mcmahan2017communication} & \xmark & \xmark &  \xmark  &  \xmark   \\
						FedDualAvg ~\cite{yuan2021federated} & \xmark & \xmark & \xmark  &  \xmark   \\
						FedHT/FedIterHT ~\cite{tong2020federated} &\xmark &  \cmark \tnote{*} & \xmark  &  \xmark  \\
						FedGradMP & \cmark  & \cmark & \ding{71}\tnote{**} & \cmark\\\bottomrule 
					\end{tabular}
					\begin{tablenotes}\footnotesize
						\item[*] The convergence analysis of FedHT/FedIterHT, however, requires that the mini-batch sizes increase exponentially in the number of communication rounds, which is generally not practical in many applications. 
						\item[**] FedGradMP does not require learning rate tuning for the sparse linear regression problem that could be still challenging for baseline algorithms based on the stochastic gradient descent, essentially due to heterogeneity in the FL environment as we illustrate Section \ref{section:Challenges in tuning LR}.
					\end{tablenotes}
				\end{threeparttable}
			\end{center}
			\caption{Comparison of our work with related references.} \label{tab:comparasion}
		\end{table*}

		\subsection{Further related works}
		There have been numerous extensions and analyses of FedAvg \cite{mcmahan2017communication, konevcny2016federated, khaled2020tighter, woodworth2020minibatch, wang2021field}, a standard algorithm to train a machine learning model in FL. FedAvg can be considered as a variant of Local SGD, which essentially runs stochastic gradient iterations at each client and averages these locally computed model parameters at a server. 
		In addition to the considerations of Local SGD \cite{mangasarian1995parallel, mcmahan2017communication} 
		for efficient communication in distributed learning, FedAvg aims to handle challenges in the FL settings such as heterogeneous client datasets and  partial client participation \cite{khaled2020tighter, wang2021field}. Thanks to the recent endeavors of researchers \cite{haddadpour2019local, khaled2020tighter, woodworth2020minibatch}, we now have a better understanding of the convergence behavior of FedAvg, especially when the objective function is (strongly) convex. As for the nonconvex case, several works provide the convergence of FedAvg to the stationary points and global convergence under extra assumptions such as Polyak-Lojasiewicz (PL) condition \cite{haddadpour2019convergence}, which is a generalization of the strong convexity condition.
		However, it is worth noting that these assumptions do not imply our main assumptions, the restricted strong convexity/smoothness. 
		
		An important research direction in FL algorithm analysis is characterizing the trade-off between convergence speed and accuracy that stems from client data heterogeneity. As the clients run more local iterations, the estimates of the local solution become more accurate at each client (improving the convergence rate) while they tend to drift away from the global solution (making the actual residual error larger), especially in a highly heterogeneous environment \cite{hsu2019measuring, wang2020tackling,
			karimireddy2020scaffold,
			charles2021convergence, khaled2020tighter}. Our analysis and numerical experiments on the convergence behavior of FedGradMP also reflect this phenomenon, which becomes more noticeable when the client datasets are highly heterogeneous. 
		
		To reduce the communication cost between the server and clients further, techniques such as sparsifying and reducing the dimensionality of the gradients have been proposed in \cite{basu2019qsparse, rothchild2020fetchsgd, mitra2021linear, haddadpour2021federated, song2022sketching}. In FedGradMP, the hard-thresholding operation is applied whenever the computed models are sent from a server or clients, so the models are already sparsified with the effective dimension same as the desired sparsity level. This makes FedGradMP more attractive in terms of saving communication resources.
		
		Another active area in FL research is client sampling or partial participation. Because of the limited connection bandwidth or a large population of clients, it is often not possible for every client to participate at each round in FL. Many methods incorporate this by modeling each client to participate randomly per round according to some distribution \cite{yang2021achieving, charles2021large, wang2021field}. There have been recent attempts to employ more elaborate sampling strategies such as importance sampling \cite{chen2020optimal}, 
		but this requires extra care since it could leak private information of client data. We analyze FedGradMP under the more common assumption, i.e., the random client participation model, and show the more client participate at each round, the faster the convergence rate is. This observation is consistent with recent findings \cite{yang2021achieving, yang2022anarchic} on the FL algorithms for the unconstrained problem. 
		
		\subsection{Organization}
		The rest of the paper is organized as follows. In Section~\ref{sec:review}, we introduce the sparse federated learning problem and make important assumptions that will be used for convergence analysis. In Section~\ref{sec:proposed}, we propose the federated gradient matching pursuit algorithm and discuss the convergence guarantees in detail. Section~\ref{sec:ext} generalizes FedGradMP and its convergence analysis to several practical settings such as the partial client participation environment and inexact estimation at the client side. In addition, we provide theoretical justifications of using a shared random Gaussian dictionary at clients to improve the performance of FedGradMP. 
		Section~\ref{sec:exp} provides a variety of numerical experiments for sparse signal recovery which demonstrate the effectiveness of the proposed approach. 
		Finally, we draw conclusions in Section~\ref{sec:con}.

		\subsection{Notation}
		\label{subsection:Notation}
		We say that a vector is $s$-sparse if it has at most $s$ nonzero entries. 
		We write $\|\cdot\|_2$ to denote the $\ell_2$ norm for a vector. We use $\|\cdot\|_F$ and $\|\cdot\|$ to denote the Frobenius norm and operator norm of a matrix, respectively. For a given positive integer $m$, $[m]$ denotes the set of integers $\{1, 2, \dots, m\}$. For positive semidefinite matrices $A$ and $B$, $A \preceq B$ means that $B - A$ is positive semidefinite.

		\section{Sparse Federated Learning}\label{sec:review}
		Federated learning is a framework to solve machine learning problems collaboratively by multiple clients possibly with a coordination server. While this provides enhanced privacy, it also poses interesting challenges since clients still require to exchange local parameters to other clients or a server in a communication-efficient way.  Moreover, in many heterogeneous learning environments, the local data of each client can be non-identically distributed and/or statistically dependent.  
		
		To formally describe FL, we begin with introducing the setup. Assume that the number of the clients is $N$, and the local objective function at the $i$-th client is denoted by $f_i(x) = \mathbb{E}_{z \sim D_i} [\ell_i(x;z)]$ where  $D_i$ is the data set at the $i$-th client and $\ell_i(x;z)$ is the  loss function about $x$ that depends on the data $z$. 
		
		The optimization problem of interest in FL typically takes the form 
		\[
		\min_{x \in \R^n} f(x) := \sum_{i=1}^N p_i f_i(x) 
		\]
		where $x \in \R^n$ and $p_i \in [0,1]$ is the weight for the $i$-th client satisfying $\sum_{i=1}^N p_i = 1$. 
		This formulation is general enough to cover the majority of machine learning setup, the empirical risk minimization (ERM) by taking the expectation uniformly over the data set $D_i$ and $p_i = |D_i|/\sum_{i=1}^N |D_i|$ \cite{wang2021field, koloskova2022sharper}.

		On the other hand, because of communication efficiency or the nature of many applications, it is natural to assume that the solution we are looking for is sparse with respect to a certain dictionary or an atom set.
		In order to discuss this general notion of sparsity, we consider a finite set of atoms $\mathcal{A} = \{{\bf{a}}_1, {\bf{a}}_2, \dots, {\bf{a}}_d\}$ in which ${\bf{a}}_i \in \R^n$ as defined in \cite{nguyen2017linear, qin2017stochastic}. For example, we recover the standard basis when $\mathcal{A} = \{e_1, \cdots, e_n\}$, where $e_i$s are the standard basis vectors for $\R^n$. We say a vector $x$ is \textbf{$\tau$-sparse with respect to $\mathcal{A}$} if $x$ can be represented as
		\[
		x = \sum_{i=1}^d \alpha_i {\bf{a}}_i
		\] where at most $\tau$ number of the coefficients $\alpha_i$'s are nonzero. 
		Then, the support of $x$ with respect to $\mathcal{A}$ is defined in a natural way, $\supp_{\mathcal{A}}(x) = \{i \in [d]: \alpha_i \neq 0 \}$. 
		We define the $\ell_0$-norm of $x$ with respect to $\mathcal{A}$ as
		\[
		\|x\|_{0, \mathcal{A}} = \min_{\alpha} \left \{|T|: x = \sum_{i \in T} \alpha_i {\bf{a}}_i, T \subseteq [d] \right\},
		\] where $\alpha = (\alpha_1, \alpha_2, \dots, \alpha_d)^T$.

		With a sparsity constraint, sparse FL aims to solve the constrained problem
		\begin{align}	\label{eq:main_problem}
			\min_{x \in \R^n} f(x) = \sum_{i=1}^N p_i f_i(x) \quad \st \quad \|x\|_{0, \mathcal{A}} \le \tau,
		\end{align}
		where $\tau$ is a preassigned sparsity level. We denote the optimal solution to this problem by $x^*$. 
		We also assume further that each local objective function $f_i$ can be expressed by the average of the functions $g_{i,j}: \R^n \rightarrow \R$, i.e.,
		\begin{equation}
			f_i(x) = {1 \over M} \sum_{j=1}^M g_{i,j}(x),
		\end{equation} 
		for some integer $M$. One can interpret $g_{i,j}$ as a loss function associated with the $i$-th client restricted to the $j$-th mini-batch, where the index set of all possible mini-batches is $[M]$. 
		
		
		The objective function of \eqref{eq:main_problem} usually depends on the data distribution at clients. For example, the least squares problem in FL typically sets
		\[
		f_i(x)={1 \over 2|D_i|} \|A_{D_i}x - y_{D_i}\|_2^2,
		\] 
		where $A_{D_i}$ is the client $i$ data matrix whose rows consist of the training input data points for client $i$ and $y_{D_i}$ are corresponding observations. 
		Let $b$ be the size of each mini-batch and  $M$ be the total number of mini-batches.
		Then, $M = { |D_i| \choose b}$ and $g_{i,j}$ is given by
		\[
		g_{i,j}(x) = {1 \over 2b} \sum\limits_{(a_k,b_k) \in S_j} ( y_k - \inner{a_k,x})^2,
		\] where $S_j$ is a subset of $D_i$ associated $j$-th mini-batch of $i$-th client. 
		We use $\mathbb{E}_j^{(i)}\varphi_{i,j}(x)$ to denote the expectation of a function $\varphi_{i,j}(x)$ provided that the $j$-th mini-batch index set is chosen from the set of the all possible mini-batches with size $b$, uniformly at random. Hence, the function $f_i$ can be expressed as 
		$f_i(x) = \mathbb{E}_j^{(i)} g_{i,j}(x)$. 


			Since the problem \eqref{eq:main_problem} is nonconvex in general, it is difficult to find its solution without additional assumptions on the objective function. In this work, we adopt the assumptions from \cite{nguyen2017linear} shown as follows.
			
			\begin{assumption} [$\mathcal{A}$-restricted Strong Convexity ($\mathcal{A}$-RSC)]
				\label{assumption:main_assumptions}
				The local objective function $f_i$ at the $i$-th client satisfies the restricted $\rho^-_\tau(i)$-strongly convexity condition: for each $i \in [N]$ and any $x_1, x_2 \in \R^n$ with $\|x_1 - x_2\|_{0, \mathcal{A}} \le \tau$,  we have
				\begin{equation}
					f_i(x_1) - f_i(x_2) - \inner{\nabla f_i(x_2), x_1 - x_2} \ge {\rho^-_\tau(i)\over 2} \|x_1 - x_2\|_2^2.
				\end{equation}
			\end{assumption}
			
			\begin{assumption}[$\mathcal{A}$-restricted Strong Smoothness ($\mathcal{A}$-RSS)]
				\label{assumption:main_assumptions2}
				The loss function $g_{i,j}$ associated with the $j$-th mini-batch at the $i$-th client satisfies the restricted $\rho^{+}_\tau(i,j)$-strongly smoothness condition: for each $i,j$ and any $x_1, x_2 \in \R^n$ with $\|x_1 - x_2\|_{0, \mathcal{A}} \le \tau$, we have
				\begin{equation}
					\| \nabla g_{i,j}(x_1) - \nabla  g_{i,j} (x_2) \|_2 \le \rho^{+}_\tau(i,j) \|x_1 - x_2\|_2.
				\end{equation}
			\end{assumption}
			
			\begin{rem}
				Assumptions \ref{assumption:main_assumptions} and \ref{assumption:main_assumptions2} are widely used in the optimization community for solving the high-dimensional statistical learning or sparse recovery problems. Note that the local loss function $f_i$ and $g_{i,j}$ may not be convex or smooth in the entire space $\R^n$ since we only need strong convexity and smoothness assumptions for vectors that are sparse with respect to a dictionary. Most convergence analysis for the FL algorithms assumes (strong) convexity of $f_i$ or the Polyak-Lojasiewicz (PL) condition \cite{haddadpour2019convergence}; neither is weaker than Assumption \ref{assumption:main_assumptions}.
			\end{rem}
			
			Several FL algorithms inspired by classical sparse optimization techniques \cite{nguyen2017linear, daubechies2004iterative, duchi2010composite} have been proposed under these assumptions. Such algorithms include FedHT and FedIterHT which are based on IHT \cite{tong2020federated}.  These algorithms use the hard-thresholding operator $\mathcal{H}_\tau(x)$, which keeps the  $\tau$ largest components of the input vector $x$ in magnitude with respect to the standard basis, whereas our algorithm adopts a more general hard-thresholding operator --- the approximate projection operator. 
			To define an approximate projection, we denote by $\mathcal{R}(\mathcal{A}_{\Gamma})$ the subspace of $\R^n$ spanned by the atoms in $\mathcal{A}$ whose indices are restricted to $\Gamma \in [d]$. For $w \in \R^n$, the orthogonal projection of $w$ to $\mathcal{R}(\mathcal{A}_{\Gamma})$ is denoted by $P_\Gamma w$. Then, an approximate projection operator with $\eta>0$, denoted by $\mathrm{approx}_\tau(x)$, constructs an index set $\Gamma$ such that 
			\[
			\|P_\Gamma w - w \|_2 \le \eta \|w - \mathcal{H}_\tau(w)\|_2,
			\] where $\mathcal{H}_\tau(w)$ is the best $\tau$-sparse approximation of $w$ with respect to $\mathcal{A}$, i.e., \[
			\mathcal{H}_\tau(w)
			=\argmin{x=\mathcal{A}\alpha,\|\alpha\|_0 \le \tau}\|w-x\|_2.
			\] Here, $\mathcal{A}$ is the matrix whose columns are the atoms by abusing the notation slightly.

			The local dissimilarity in FL captures how the data distributions among clients are different, which is typically in the following form, especially in the early works in FL \cite{li2018federated,yuan2021federated, wang2021field}:
			\begin{align}
				\label{assump:strong_local_dissimilarity}
				\mathbb{E}_{i \sim \mathcal{P}} \|\nabla f_i(x) - \nabla f(x) \|_2^2 \le \beta^2 \|\nabla f(x)\|_2^2 + \zeta^2, \quad  \forall\, x\in\R^n. 
			\end{align}
			When $\beta = 0$, this condition reduces to the uniform bounded heterogeneity condition that has been used for analyzing the convergence of many popular FL algorithms, such as FedAvg \cite{li2018federated} and FedDualAvg \cite{yuan2021federated, wang2021field}. 
			In this work, the assumption of heterogeneity on the client data is much weaker than \eqref{assump:strong_local_dissimilarity} by assuming heterogeneity only at the solution $x^*$ as follows.
			
			\begin{assumption}
				\label{assump:local_dissimilarity_mod}
				There is a minimizer for \eqref{eq:main_problem}, denoted by $x^*$ with a finite $\zeta_*^2$ defined as below:
				\[
				\zeta_*^2 =\mathbb{E}_{i\sim\mathcal{P}}\|{\nabla f_i(x^*)}\|_2^2= \sum_{i=1}^N p_i \| \nabla f_i(x^*)\|_2^2.
				\]
			\end{assumption}
			Assumption \ref{assump:local_dissimilarity_mod} is the same as the one used for more recent analyses giving sharper convergence guarantees of FL algorithms \cite{khaled2020tighter, woodworth2020minibatch}. This is also a necessary assumption for the FedAvg type of algorithms to converge  \cite{woodworth2020minibatch}. But there are a few places where we state the implication of our results under stronger assumptions such as \eqref{assump:strong_local_dissimilarity}, in order to compare the implications of our results to previous works. 
			
			Other common assumptions in the analysis of federated learning algorithms are the unbiased and bounded variance conditions of local stochastic gradients.
			
			\begin{assumption}
				\label{asump:local_gradients}
				The local stochastic gradient $\nabla g_{i,j}$  associated with the randomly selected $j$-th mini-batch  at the $i$-th client satisfies
				\begin{align}
					\label{asump:local_gradients1}
					\mathbb{E}_j^{(i)} [ \nabla g_{i,j}(x) ] =  \nabla f_i(x) \quad \text{for any $\tau$-sparse vector $x$,}
				\end{align} 
				and
				\begin{align}
					\label{asump:local_gradients2}		\mathbb{E}_j^{(i)} \| \nabla g_{i,j}(x) - \nabla f_i(x) \|_2^2 \le \sigma^2_i \quad \text{for any $\tau$-sparse vector $x$}, 
				\end{align} 
				where $\mathbb{E}_j^{(i)}$ is the expectation taken over the mini-batch index selected from $[M]$ at the client $i$. 
			\end{assumption}
			
			\begin{rem}
				The bounded variance condition of local stochastic gradients associated with  mini-batches  \eqref{asump:local_gradients2} in Assumption \ref{asump:local_gradients} is widely used in FL and stochastic algorithms in general \cite{cho2022towards, wang2021field}, but it may not hold for some settings \cite{khaled2020tighter}. This assumption \eqref{asump:local_gradients2} is actually not essential for our main result to hold (See Appendix for the proof of our convergence theorem without this assumption) but we present our work under the assumption for the sake of simplicity. 
			\end{rem}

			The following lemma is a well-known consequence of the $\mathcal{A}$-RSS property in Assumption \ref{assumption:main_assumptions}.
			\begin{lem}[Descent lemma]
				\label{lem:consequence_RSS}
				Suppose that the function $h(x)$ satisfies the $\mathcal{A}$-RSS property in Assumption \ref{assumption:main_assumptions2} with a constant $\rho^+_{\tau}$. Then for any $x_1, x_2 \in \R^n$ with $\|x_2\|_{0, \mathcal{A}} \le \tau$, it holds
				\begin{align*}
					&\inner{\nabla h(x_1), x_2} \ge h(x_1 + x_2) - h(x_1) - {\rho^+_{\tau} \over 2} \|x_2\|_2^2.
				\end{align*}	
			\end{lem}
			
			\section{Federated Gradient Matching Pursuit}
			\label{sec:proposed}
			
			In this section, we propose Federated Gradient Matching Pursuit (FedGradMP) and discuss its convergence guarantee. We start with describing FedGradMP in Algorithm \ref{alg:FedGradMP}. 
			
			In the FedGradMP framework, the StoGradMP algorithm \cite{nguyen2017linear} is implemented at the client side and the server aggregates the resulting locally computed models after each round followed by a projection onto a subspace of dimension at most $\tau$. Each iteration of StoGradMP at a client consists of the following five steps: 
			\begin{enumerate}
				\item 
				Randomly select a mini-batch from the client batch. 
				\item Compute the stochastic gradient associated with the selected mini-batch. \item Merge the subspace associated with the previously estimated local model with the closest subspace of dimension at most $2\tau$ to the stochastic gradient from Step 2. 
				\item Solve the minimization problem for the local objective function at the client over the merged subspace from Step 3. \item Identify the closest subspace of dimension at most $\tau$ to the solution at Step 4. 
			\end{enumerate}
			Note that in Step 4, the clients are not minimizing the local objective function $f_i$ over the sparsity constraint, but over the subspace associated with the estimated sparsity pattern of the solution in Step 3. This subproblem can be often solved efficiently since $f_i$ are strongly convex/smooth with respect to such subspaces by Assumptions \ref{assumption:main_assumptions} and \ref{assumption:main_assumptions2}, especially when the dimension of the subspace is small or $f_i$ are quadratic \cite{ beck2017first, nesterov2018lectures}. Nevertheless, it could be expensive to solve this subproblem in general, so we discuss how to obtain its approximate solution by computationally cheap methods in the next section.
			
			
			\begin{algorithm} 
				\caption{FedGradMP}
				\begin{algorithmic} 
					\Input The number of rounds $T$, the number of clients $N$, the number of local iterations $K$, weight vector $p$, the estimated sparsity level $\tau$, $\eta_1, \eta_2, \eta_3$.
					
					\Output $\hat{x} = x_T$.
					
					\Initialize	$x_0 = 0$, $\Lambda = \emptyset$.
					
					\textbf{for $t=0, 1, \dots, T-1$ do}
					
					\indt \textbf{for client $i=1, 2, \dots, N$ do}
					
					\indt \indt $x_{t,1}^{(i)} = x_t$
					
					\indt \indt \textbf{for $k=1$ to $K$ do}
					
					\indt \indt \indt Select a mini-batch index set $j_k := i_{t,k}^{(i)}$ uniformly at random from $\{1,2,\dots, M\}$
					
					\indt \indt \indt Calculate the stochastic gradient $r^{(i)}_{t,k} = \nabla g_{i,j_k} \left( x_{t,k}^{(i)} \right)$
					
					\indt \indt \indt $\Gamma = $ approx$_{2\tau} (r^{(i)}_{t,k}, \eta_1)$
					
					\indt \indt \indt $\widehat{\Gamma} = \Gamma \cup \Lambda$
					
					\indt \indt \indt $b^{(i)}_{t,k} = \argmin{x} f_i(x), \quad x \in \mathcal{R}(\mathcal{A}_{\widehat{\Gamma}})$
					
					\indt \indt \indt $\Lambda = $ approx$_{\tau} (b^{(i)}_{t,k}, \eta_2)$
					
					\indt \indt \indt $x_{t,k+1}^{(i)} = P_{\Lambda} (b^{(i)}_{t,k})$
					
					\indt \indt	\ForEnd
					
					\indt \ForEnd
					
					\indt $\Lambda_s = $     approx$_{\tau} \left( \sum_{i=1}^N p_i x_{t,{K+1}}^{(i)}, \eta_3 \right)$
					
					\indt $x_{t+1} = P_{\Lambda_s} \left( \sum_{i=1}^N p_i x_{t,{K+1}}^{(i)} \right)$

					\ForEnd
					
				\end{algorithmic}
				\label{alg:FedGradMP}
			\end{algorithm}

			\subsection{Linear Convergence of FedGradMP}
			This subsection is devoted to proving the linear convergence of FedGradMP in the number of communication rounds. The first step of the proof for our main theorem is similar to the one for Theorem 3.1 in \cite{tong2020federated} but also utilizes several lemmas below from \cite{nguyen2017linear} and \cite{qin2017stochastic} after some modifications to accommodate the FL setting. 
			
			\begin{lem}  [{\cite[Lemma~1]{nguyen2017linear}}]
				\label{lem:GradMP_lem1}
				The approximation error between the $(k+1)$-th local iterate  $x_{t,k+1}^{(i)}$ and $x^*$ is bounded by
				\[
				\|x_{t,k+1}^{(i)} - x^*\|_2^2 \le (1 + \eta_2)^2 \|b^{(i)}_{t,k} - x^*\|_2^2.
				\]
			\end{lem}
			
			For notational  convenience, we define the following two quantities:
			\begin{align*}
				&\rho^{+{(i)}}_\tau = \max_j \rho^{+}_\tau(i,j), \quad \bar{\rho}^{+{(i)}}_\tau = {1 \over M} \sum_{j=1}^M \rho^{+}_\tau(i,j). 
			\end{align*}
			
			\begin{lem}  [{\cite[Lemma~5.7]{qin2017stochastic}}]
				\label{lem:GradMP_lem2}
				Let $\widehat{\Gamma}$ be the set obtained from the $k$-th iteration at client $i$. Then, we have
				\[
				\mathbb{E}^{(i)}_{J_k}  \|b^{(i)}_{t,k} - x^*\|_2^2 \le \beta_1(i) \mathbb{E}^{(i)}_{J_k}  \|P^\perp_{\widehat{\Gamma}}( b^{(i)}_{t,k} - x^*)\|_2^2 + \xi_1(i),
				\] where 
				\begin{align*}
					\beta_1(i) &= { \bar{\rho}^{+(i)}_{4\tau} \over 2\rho^{-}_{4\tau}(i) -  \bar{\rho}^{+(i)}_{4\tau} }, \quad
					\xi_1(i) = {2 \mathbb{E}^{(i)}_{J_k, j} \|P_{\widehat{\Gamma}} \nabla g_{i,j}(x^*)\|_2^2 \over \bar{\rho}^{+(i)}_{4\tau} (2\rho^{-}_{4\tau}(i) - \bar{\rho}^{+(i)}_{4\tau} ) }.
				\end{align*}	
				Here $J_k$ denotes the set of all previous mini-batch indices $j_1, \dots, j_k$ randomly selected in or before the $k$-th step of the local iterations at the $i$-th client and $\mathbb{E}^{(i)}_{J_k}$ is the expectation taken over $J_k$. 
			\end{lem}
			
			The following lemma is an extended version of Lemma 3 in \cite{nguyen2017linear}, whose proof further utilizes Young's inequality to control the trade-off between contraction and residual error due to the noise of the stochastic gradient for the FL setting. It also provides a refinement for the exact projection operator. Since the proof of the lemma is substantially different from the original version due to nontrivial changes to accommodate FL setting, we include its proof. 
			
			\begin{lem}	\label{lem:GradMP_lem3}
				Let $\widehat{\Gamma}$ be the set obtained from the $k$-th iteration at client $i$. Then, for any $\theta > 0$, we have
				\[
				\mathbb{E}^{(i)}_{j_k}  \|P^\perp_{\widehat{\Gamma}}( b^{(i)}_{t,k} - x^*)\|_2^2 \le \beta_2(i) \| x^{(i)}_{t,k} - x^* \|_2^2 + \xi_2(i) 
				\] where 
				\begin{align*}
					\beta_2(i) &= \left(2{\left( {\bar{\rho}^{+(i)}_{4\tau} } +  {1 \over \theta^2} \right)  - \eta_1^2 \rho^-_{4\tau}(i)   \over \eta_1^2 \rho_{4\tau}^{-}(i)} +  {2\sqrt{\eta_1^2-1} \over \eta_1 \rho_{4\tau}^{-}(i) }(3 \mathbb{E}_{j_k} (\rho^{+}_\tau(i,j_k))^2+1)  \right), \\
					\xi_2(i)  &=
					\begin{cases}
						{8 \over (\rho_{4\tau}^{-}(i) )^2}  \max\limits_{\substack{ \Omega \subset [d] \\ |\Omega| = 4\tau}}  \| P_{\Omega}  \nabla f_i(x^*) \|_2^2 + {1 \over \rho_{4\tau}^{-}(i) } \left[\left({2\theta^2 } + {6\sqrt{\eta_1^2-1} \over \eta_1} \right) \sigma_i^2 + {6\sqrt{\eta_1^2-1} \over \eta_1}\| \nabla f_i \left(x^* \right) \|_2^2 \right] \quad  \text{if $\eta_1 > 1$},\\
						{8 \over (\rho_{4\tau}^{-}(i) )^2}  \max\limits_{\substack{ \Omega \subset [d] \\ |\Omega| = 4\tau}}  \| P_{\Omega}  \nabla f_i(x^*) \|_2^2 + {2{\theta^2} \sigma_i^2 \over \rho_{4\tau}^{-}(i) } \quad \text{if $\eta_1 = 1$ (when the projection operator is exact)}.
					\end{cases}
				\end{align*}
				Here $\mathbb{E}^{(i)}_{j_k}$ is the expectation taken over the randomly selected mini-batch index $j_k$ for the stochastic gradient in the $k$-th step of the local iterations at the $i$-th client. 
			\end{lem}
			
			\begin{proof}
				We start with by noticing $P^\perp_{\widehat{\Gamma}} b^{(i)}_{t,k} = 0$ and $P^\perp_{\widehat{\Gamma}} x^{(i)}_{t,k}  = 0$ since both $b^{(i)}_{t,k}$ and $x^{(i)}_{t,k}$ belong to the span of $\mathcal{A}_{\widehat{\Gamma}}$. Let $\Delta := x^* - x^{(i)}_{t,k}  $ and the set $\supp_{\mathcal{A}}(\Delta)$ be denoted by $R$. Note that $|R| \le 2\tau$.  Hence, we have
				\begin{align*}
					\|P^\perp_{\widehat{\Gamma}}(b^{(i)}_{t,k}  - x^*)\|_2 
					&= \|P^\perp_{\widehat{\Gamma}}(b^{(i)}_{t,k} - x^{(i)}_{t,k} + x^{(i)}_{t,k}   - x^*)\|_2 \\
					&\le  \|P^\perp_{\widehat{\Gamma}}(b^{(i)}_{t,k} - x^{(i)}_{t,k})\|_2 + \|P^\perp_{\widehat{\Gamma}}(x^{(i)}_{t,k}  - x^*)\|_2\\
					&= \|P^\perp_{\widehat{\Gamma}}(x^{(i)}_{t,k}  - x^*)\|_2\\
					&\le \|P^\perp_{\Gamma}(x^{(i)}_{t,k}  - x^*)\|_2\\
					&=\|\Delta - P_{\Gamma}\Delta\|_2.
				\end{align*}  Here the second inequality follows from the definitions of the sets $\Gamma$ and $\widehat{\Gamma}$, $\Gamma \subset \widehat{\Gamma}$.
				
				Now we estimate $\|\Delta - P_{\Gamma}\Delta\|_2^2$. With a slight abuse of notation, $\mathbb{E}^{(i)}_{j_k}$ will be denoted by $\mathbb{E}_{j_k}$ throughout the proof.
				First, from the $\mathcal{A}$-RSC property of $f_i$, we have 
				\begin{align}
					\nonumber
					&f_i(x^*) - f_i(x^{(i)}_{t,k}) -  {\rho_{4\tau}^{-}(i) \over 2} \| x^* - x^{(i)}_{t,k} \|_2^2 \\
					\nonumber
					&\quad \ge \inner{\nabla f_i(x^{(i)}_{t,k}),  x^* - x^{(i)}_{t,k}} \\
					\nonumber
					&\quad = \mathbb{E}_{j_k} \inner{\nabla g_{i,j_k}(x^{(i)}_{t,k}),  x^* - x^{(i)}_{t,k}} \\
					\nonumber
					&\quad = \mathbb{E}_{j_k} \inner{P_R \nabla g_{i,j_k}(x^{(i)}_{t,k}),  x^* - x^{(i)}_{t,k}} \\
					\label{eq:bound1_Lemma}
					&\quad \ge -  \mathbb{E}_{j_k} \| P_R \nabla g_{i,j_k}(x^{(i)}_{t,k}) \| \|  \Delta \|_2 
				\end{align}
				By applying the inequality (15) in \cite{nguyen2017linear} and from the fact that $\Gamma$ is the support set after the projection of the stochastic gradient $\nabla g_{i,j_k}(x^{(i)}_{t,k})$ in Algorithm \ref{alg:FedGradMP}, we have
				\[
				\| P_R \nabla g_{i,j_k}(x^{(i)}_{t,k}) \|_2 \le \|P_{\Gamma} \nabla g_{i,j_k}(x^{(i)}_{t,k})\|_2 + {\sqrt{\eta_1^2-1} \over \eta_1} \|P^\perp_{\Gamma} \nabla g_{i,j_k}(x^{(i)}_{t,k})\|_2.
				\]
				To make the notation simple, we define \[
				z := - {P_{\Gamma} \nabla g_{i,j_k}(x^{(i)}_{t,k}) \over \|P_{\Gamma} g_{i,j_k}(x^{(i)}_{t,k}) \|_2} \|\Delta\|_2.
				\]
				Then, the term $ -  \mathbb{E}_{j_k} \| P_R \nabla g_{i,j_k}(x^{(i)}_{t,k}) \|_2 \|  \Delta \|_2 $  can be further bounded as follows.
				\begin{align*}
					& -   \mathbb{E}_{j_k} \| P_R \nabla g_{i,j_k}(x^{(i)}_{t,k}) \|_2 \|  \Delta \|_2\\
					&\ge -  \mathbb{E}_{j_k}\| P_\Gamma \nabla g_{i,j_k}(x^{(i)}_{t,k}) \|_2 \|  \Delta \|_2 - {\sqrt{\eta_1^2-1} \over \eta_1} \mathbb{E}_{j_k} \|P^\perp_{\Gamma} \nabla g_{i,j_k}(x^{(i)}_{t,k}) \|_2 \|  \Delta \|_2\\
					&= \mathbb{E}_{j_k} \inner {P_\Gamma \nabla g_{i,j_k}(x^{(i)}_{t,k}),  z } - {\sqrt{\eta_1^2-1} \over \eta_1} \mathbb{E}_{j_k} \|P^\perp_{\Gamma} \nabla g_{i,j_k}(x^{(i)}_{t,k}) \|_2 \|  \Delta \|_2  \\
					&= \mathbb{E}_{j_k} \inner { \nabla g_{i,j_k}(x^{(i)}_{t,k}),  z } - {\sqrt{\eta_1^2-1} \over \eta_1} \mathbb{E}_{j_k} \|P^\perp_{\Gamma} \nabla g_{i,j_k}(x^{(i)}_{t,k}) \|_2 \|  \Delta \|_2   \\
					&\ge \mathbb{E}_{j_k} \inner { \nabla g_{i,j_k}(x^{(i)}_{t,k}),  z } - {\sqrt{\eta_1^2-1} \over 2 \eta_1} \mathbb{E}_{j_k} \left( \|P^\perp_{\Gamma} \nabla g_{i,j_k}(x^{(i)}_{t,k}) \|_2^2 + \|  \Delta \|_2^2  \right)   \\
					&= \mathbb{E}_{j_k} \inner { \nabla f_i \left(x^{(i)}_{t,k} \right),  z } +  \mathbb{E}_{j_k} \inner { \nabla g_{i,j_k}(x^{(i)}_{t,k}) - \nabla f_i \left(x^{(i)}_{t,k} \right),  z } - {\sqrt{\eta_1^2-1} \over 2 \eta_1} \mathbb{E}_{j_k} \left( \|P^\perp_{\Gamma} \nabla g_{i,j_k}(x^{(i)}_{t,k}) \|_2^2 + \|  \Delta \|_2^2  \right),
				\end{align*}
				where the second inequality above follows from     the AM-GM inequality,  
				$ab \le (a^2 + b^2)/2$ for nonnegative real numbers $a, b$.
				
				Now, we apply the Young's inequality for the inner product space to the second term in the last line to obtain
				\[
				\left| \inner { \nabla g_{i,j_k}(x^{(i)}_{t,k}) - \nabla f_i \left(x^{(i)}_{t,k} \right),  z } \right| \le {\theta^2 \over 2} \|\nabla g_{i,j_k}(x^{(i)}_{t,k}) - \nabla f_i \left(x^{(i)}_{t,k} \right)\|_2^2 + {1 \over 2\theta^2} \|z\|_2^2
				\] for any nonzero $\theta$. 
				Hence, we have
				\begin{align*}
					& -   \mathbb{E}_{j_k} \| P_R \nabla g_{i,j_k}(x^{(i)}_{t,k}) \|_2 \|  \Delta \|_2\\
					& \ge \mathbb{E}_{j_k} \inner { \nabla f_i\left(x^{(i)}_{t,k} \right),  z } -  {\theta^2 \over 2} \mathbb{E}_{j_k} \|\nabla g_{i,j_k}(x^{(i)}_{t,k}) - \nabla f_i \left(x^{(i)}_{t,k} \right)\|_2^2 - {1 \over 2\theta^2} \mathbb{E}_{j_k} \|z\|_2^2 \\
					& \qquad - {\sqrt{\eta_1^2-1} \over 2 \eta_1} \left(\mathbb{E}_{j_k}  \|P^\perp_{\Gamma} \nabla g_{i,j_k}(x^{(i)}_{t,k}) \|_2^2 + \mathbb{E}_{j_k} \|  \Delta \|_2^2  \right)\\
					& \ge \mathbb{E}_{j_k} \inner { \nabla f_i\left(x^{(i)}_{t,k} \right),  z } -  {\theta^2 \over 2} \sigma_i^2 - {1 \over 2\theta^2} \mathbb{E}_{j_k} \|z\|_2^2  - {\sqrt{\eta_1^2-1} \over 2 \eta_1} \left(\mathbb{E}_{j_k}  \|P^\perp_{\Gamma} \nabla g_{i,j_k}(x^{(i)}_{t,k}) \|_2^2 + \mathbb{E}_{j_k} \|  \Delta \|_2^2  \right),
				\end{align*}
				where we have used \eqref{asump:local_gradients2} in Assumption  \ref{asump:local_gradients} in the last inequality above.
				
				Next, we obtain the upper bound for $\mathbb{E}_{j_k}  \|P^\perp_{\Gamma} \nabla g_{i,j_k}(x^{(i)}_{t,k}) \|_2^2$ as follows.
				\begin{align*}
					&\mathbb{E}_{j_k}  \|P^\perp_{\Gamma} \nabla g_{i,j_k}(x^{(i)}_{t,k}) \|_2^2 \\
					&\le \mathbb{E}_{j_k}  \|\nabla g_{i,j_k}(x^{(i)}_{t,k}) \|_2^2 \\
					&\le 3 \mathbb{E}_{j_k} \|\nabla g_{i,j_k}(x^{(i)}_{t,k}) - \nabla g_{i,j_k}(x^*)  \|_2^2 
					+ 3 \mathbb{E}_{j_k} \|\nabla g_{i,j_k}(x^*) - \nabla f_i \left(x^* \right)\|_2^2 
					+ 3 \mathbb{E}_{j_k} \| \nabla f_i \left(x^* \right) \|_2^2 \\
					&\le 3  \mathbb{E}_{j_k} (\rho^{+}_\tau(i,j_k))^2 \|x^{(i)}_{t,k}  - x^* \|_2^2 + 3 \sigma_i^2 +  3  \| \nabla f_i \left(x^* \right) \|_2^2\\ 
					&= 3  \mathbb{E}_{j_k} (\rho^{+}_\tau(i,j_k))^2 \| \Delta \|_2^2  + 3 \sigma_i^2 +  3  \| \nabla f_i \left(x^* \right) \|_2^2,
				\end{align*}
				where we have used the inequality  $\|a+b+c\|_2^2 \le 3\|a\|_2^2 + 3\|b\|_2^2 + 3\|c\|_2^2$ in the second inequality, and Assumption \ref{assumption:main_assumptions2} and  Assumption \ref{asump:local_gradients} in the third inequality above.
				
				Combining this bound with inequality \eqref{eq:bound1_Lemma} yields
				\small
				\begin{align}
					\nonumber
					&f_i(x^*) - f_i \left(x^{(i)}_{t,k} \right) -  {\rho_{4\tau}^{-}(i) \over 2} \| x^* - x^{(i)}_{t,k} \|_2^2 \\
					\label{eq:bound2_Lemma}
					& \ge  \mathbb{E}_{j_k} \inner { \nabla f_i\left(x^{(i)}_{t,k} \right),  z } -  {\theta^2 \over 2} \sigma_i^2 - {1 \over 2\theta^2} \mathbb{E}_{j_k} \|z\|_2^2  - {\sqrt{\eta_1^2-1} \over 2 \eta_1} \left(\mathbb{E}_{j_k}  \|P^\perp_{\Gamma} \nabla g_{i,j_k}(x^{(i)}_{t,k}) \|_2^2 + \mathbb{E}_{j_k} \|  \Delta \|_2^2  \right) \\
					\nonumber
					& \ge  \mathbb{E}_{j_k} \inner { \nabla f_i\left(x^{(i)}_{t,k} \right),  z } -  {\theta^2 \over 2} \sigma_i^2 - {1 \over 2\theta^2} \mathbb{E}_{j_k} \|z\|_2^2  - {\sqrt{\eta_1^2-1} \over 2 \eta_1} \left[(3 \mathbb{E}_{j_k} (\rho^{+}_\tau(i,j_k))^2+1) \| \Delta \|_2^2 + 3 \sigma_i^2 +  3  \| \nabla f_i \left(x^* \right) \|_2^2\right].
				\end{align}	
				\normalsize	
				On the other hand, using Lemma \ref{lem:consequence_RSS} for $g_{i,j}$, a consequence of the $\mathcal{A}$-RSS property, we have
				\[
				\inner { \nabla g_{i,j}(x^{(i)}_{t,k}),  z } \ge   g_{i,j} 	\left(x^{(i)}_{t,k} +  z \right) -  g_{i,j} \left(x^{(i)}_{t,k} \right) -  {\rho^{+}_{4\tau}(i,j) \over 2} \| z\|_2^2
				\] for all $j \in [M]$. 
				By taking the average over all $g_{i,j}$ over $j \in [M]$ on both sides of the inequality above and from the definitions of $f_i$ and $\rho^{+(i)}_{4\tau}$, we obtain
				\[
				\inner { \nabla  f_i(x^{(i)}_{t,k}),  z } \ge  f_i \left(x^{(i)}_{t,k} +  z \right) -  f_i \left(x^{(i)}_{t,k} \right) -   {\bar{\rho}^{+(i)}_{4\tau} \over 2}  \| z\|_2^2.
				\] Here we have used \eqref{asump:local_gradients1} in Assumption \ref{asump:local_gradients}.
				We then take the expectation $\mathbb{E}_{j_k}$ on both sides of the inequality. 
				\[
				\mathbb{E}_{j_k}	\inner { \nabla  f_i(x^{(i)}_{t,k}),  z } \ge \mathbb{E}_{j_k} f_i \left(x^{(i)}_{t,k} +  z \right) -  f_i \left(x^{(i)}_{t,k} \right) -   {\bar{\rho}^{+(i)}_{4\tau} \over 2} \mathbb{E}_{j_k} \| z\|_2^2.
				\]
				
				After applying this bound to inequality \eqref{eq:bound2_Lemma}, we obtain
				\begin{align*}
					f_i(x^*) - f_i \left( x^{(i)}_{t,k} \right) &-  {\rho_{4\tau}^{-}(i) \over 2} \| \Delta \|_2^2
					\ge \mathbb{E}_{j_k} f_i \left(x^{(i)}_{t,k} +  z \right) -  f_i \left(x^{(i)}_{t,k} \right) -  {\bar{\rho}^{+(i)}_{4\tau} \over 2} \mathbb{E}_{j_k} \| z\|_2^2 \\
					& -{\theta^2 \over 2} \sigma_i^2 - {1 \over 2\theta^2} \mathbb{E}_{j_k} \|z\|_2^2  - {\sqrt{\eta_1^2-1} \over 2 \eta_1} \left[(3 \mathbb{E}_{j_k} (\rho^{+}_\tau(i,j_k))^2+1) \| \Delta \|_2^2 + 3 \sigma_i^2 +  3  \| \nabla f_i \left(x^* \right) \|_2^2\right].
				\end{align*}
				
				Thus, we have
				\begin{align}   \label{ineq:GradMP_lem3_bound1}
					&\left( {\bar{\rho}^{+(i)}_{4\tau} \over 2} +  {1 \over 2\theta^2} \right) \mathbb{E}_{j_k} \| z\|_2^2 - {1 \over 2} \left(\rho_{4\tau}^{-}(i)  -  {\sqrt{\eta_1^2-1} \over \eta_1}(3 \mathbb{E}_{j_k} (\rho^{+}_\tau(i,j_k))^2+1) \right) \|\Delta\|_2^2 \nonumber\\
					& + \left({\theta^2 \over 2} + {3\sqrt{\eta_1^2-1} \over 2\eta_1} \right) \sigma_i^2 + {3\sqrt{\eta_1^2-1} \over 2\eta_1}\| \nabla f_i \left(x^* \right) \|_2^2    \nonumber\\
					& \ge \mathbb{E}_{j_k} f_i \left(x^{(i)}_{t,k} +  z \right) -  f_i \left(x^* \right)  \nonumber \\
					& \ge {\rho^-_{4\tau}(i) \over 2} \mathbb{E}_{j_k} \|x^{(i)}_{t,k} +  z - x^* \|_2^2 + \mathbb{E}_{j_k}
					\inner{\nabla f_i(x^*) ,x^{(i)}_{t,k} +  z - x^*}\\
					\nonumber
					& = {\rho^-_{4\tau}(i) \over 2}  \mathbb{E}_{j_k} \|\Delta - z  \|_2^2 + \mathbb{E}_{j_k} \inner{\nabla f_i(x^*) , z - \Delta}\\
					\nonumber
					& = {\rho^-_{4\tau}(i) \over 2}  \mathbb{E}_{j_k} \|\Delta - z  \|_2^2 + \mathbb{E}_{j_k} \inner{\nabla f_i(x^*) , P_{\Gamma \cup R} (z - \Delta)} \tag{$\star$}\\
					\nonumber
					& = {\rho^-_{4\tau}(i) \over 2}  \mathbb{E}_{j_k} \|\Delta - y  \|_2^2 + \mathbb{E}_{j_k} \inner{ P_{\Gamma \cup R} \nabla f_i(x^*) , (z - \Delta)}\\
					\nonumber
					& \ge {\rho^-_{4\tau}(i) \over 2}  \mathbb{E}_{j_k} \|\Delta - z  \|_2^2  - \mathbb{E}_{j_k} \| P_{\Gamma \cup R}  \nabla f_i(x^*)\|_2   \| z - \Delta \|_2\\
					\label{ineq:GradMP_lem3_bound2}
					& \ge {\rho^-_{4\tau}(i) \over 2}  \|\Delta - z  \|_2^2  - \max_{\substack{ \Omega \subset [d] \\ |\Omega| = 4\tau}}  \| P_{\Omega}  \nabla f_i(x^*) \|_2 \mathbb{E}_{j_k} \| \Delta -z \|_2.
				\end{align}
				Here, the inequality \eqref{ineq:GradMP_lem3_bound1} follows from $\mathcal{A}$-RSC. In $(\star)$ of the above inequality chain, we have used the fact that $z = - {P_{\Gamma} \nabla g_{i,j_k}(x^{(i)}_{t,k}) \over \|P_{\Gamma} g_{i,j_k}(x^{(i)}_{t,k}) \|} \|\Delta\|_2$. 
				
				
				Let $u = \mathbb{E}_{j_k} \|\Delta - z\|_2 $, $a =  \rho^-_{4\tau}(i) $, $b = \max\limits_{\substack{ \Omega \subset [d] \\ |\Omega| = 4\tau}}  \| P_{\Omega}  \nabla f_i(x^*) \|_2$, and 
				\[\begin{aligned}
					c =& \left( \bar{\rho}^{+(i)}_{4\tau} +  {1 \over \theta^2} \right) \mathbb{E}_{j_k} \| z\|_2^2 -  \left(\rho_{4\tau}^{-}(i)  -  {\sqrt{\eta_1^2-1} \over \eta_1}(3 \mathbb{E}_{j_k} (\rho^{+}_\tau(i,j_k))^2+1) \right)\|\Delta\|_2^2 \nonumber\\
					& + \left(\theta^2 + {3\sqrt{\eta_1^2-1} \over \eta_1} \right) \sigma_i^2 + {3\sqrt{\eta_1^2-1} \over \eta_1}\| \nabla f_i \left(x^* \right) \|_2^2.
				\end{aligned}
				\]
				
				Then the inequality \eqref{ineq:GradMP_lem3_bound2} can be rewritten as $au^2 - 2bu - c \le 0$ which gives
				\[
				\mathbb{E}_{j_k} \|\Delta - z\|_2 \le \sqrt{c \over a} + {2b \over a}.
				\]
				
				Moreover, we have
				\begin{align*}
					\|\Delta - P_\Gamma \Delta \|_2^2 \le \|\Delta - z \|_2^2.
				\end{align*}
				
				Combining the previous two bounds yields
				\begin{align*}
					\mathbb{E}_{j_k}\|\Delta - P_\Gamma \Delta \|_2^2 
					&\le  \left(\sqrt{c \over a} + {2b \over a} \right)^2 
					\le {2c \over a} + {8b^2 \over a^2}.
				\end{align*}

				On the other hand, since 
				\begin{align*}
					\|z\|_2^2 = \left \|- {P_{\Gamma} \nabla g_{i,j_k}(x^{(i)}_{t,k}) \over \|P_{\Gamma} g_{i,j_k}(x^{(i)}_{t,k}) \|_2} \|\Delta\|_2 \right \|_2^2 = \|\Delta\|_2^2,
				\end{align*}
				we have 
				\begin{align*}
					& c \le \left(\bar{\rho}^{+(i)}_{4\tau}  +  {1 \over \theta^2} - \rho_{4\tau}^{-}(i)  +  {\sqrt{\eta_1^2-1} \over \eta_1}(3 \mathbb{E}_{j_k} (\rho^{+}_\tau(i,j_k))^2+1) \right)\|\Delta\|_2^2 \\
					& \qquad + \left(\theta^2  + {3\sqrt{\eta_1^2-1} \over \eta_1} \right) \sigma_i^2 + {3\sqrt{\eta_1^2-1} \over \eta_1}\| \nabla f_i \left(x^* \right) \|_2^2.  
				\end{align*}
				
				Thus, 
				\begin{align*}
					&\mathbb{E}_{j_k}\|\Delta - P_\Gamma \Delta \|_2^2 \\
					& \quad \le \left(2{\left( {\bar{\rho}^{+(i)}_{4\tau} } +  {1 \over \theta^2} \right)  - \eta_1^2 \rho^-_{4\tau}(i)   \over \eta_1^2 \rho_{4\tau}^{-}(i)} +  {2\sqrt{\eta_1^2-1} \over \eta_1 \rho_{4\tau}^{-}(i) }(3 \mathbb{E}_{j_k} (\rho^{+}_\tau(i,j_k))^2+1)  \right) \|\Delta\|_2^2 \\
					&\qquad + {8 \over (\rho_{4\tau}^{-}(i) )^2}   \max_{\substack{ \Omega \subset [d] \\ |\Omega| = 4\tau}}  \| P_{\Omega}  \nabla f_i(x^*) \|_2^2 + {1 \over \rho_{4\tau}^{-}(i) } \left[\left({2\theta^2 } + {6\sqrt{\eta_1^2-1} \over \eta_1} \right) \sigma_i^2 + {6\sqrt{\eta_1^2-1} \over \eta_1}\| \nabla f_i \left(x^* \right) \|_2^2 \right] .
				\end{align*}
				
			\end{proof}

			Equipped with these lemmas, we are ready to prove our main result for the linear convergence of FedGradMP. 
			\begin{thm}
				\label{thm:main_theorem1}
				Let $x^*$ be the solution to \eqref{eq:main_problem} and $x_0$ be the initial feasible solution. Assume that the local objective $f_i$ satisfies $\mathcal{A}$-RSC with constant $\rho^-_{4\tau}(i)$ in Assumption \ref{assumption:main_assumptions} and all of the functions $g_{i,j}$ associated with mini-batches satisfy $\mathcal{A}$-RSS with constant $\rho^+_{4\tau}(i,j)$ in Assumption \ref{assumption:main_assumptions2}. We further assume that $g_{i,j}$ satisfies the bounded variance condition of local stochastic gradients in Assumption \ref{asump:local_gradients} with variance bound $\sigma^2_i$. Let $K$ be the number of local iterations at each client. 
				
				Then, for any $\theta > 0$, the expectation of the recovery error at the $(t+1)$-th round of FedGradMP described in Algorithm \ref{alg:FedGradMP} obeys
				\[
				\mathbb{E} \|x_{t+1} - x^*\|_2^2 
				\le  \kappa^{t+1}  \|x_0 - x^*\|_2^2  +  {(2 \eta_3^2 + 2) \nu \over 1-\kappa} \sum_{i=1}^N p_i { 1 - \mu(i)^K   \over 1 - \mu(i)},	
				\]
				where 
				\[
				\kappa = (2 \eta_3^2 + 2) \sum_{i=1}^N p_i \left[(1 + \eta_2)^2  \beta_1(i)\beta_2(i) \right]^K \qquad \text{and} \qquad \mu(i) =  (1 + \eta_2)^2  \beta_1(i)\beta_2(i). 
				\]
				Here 
				\begin{align*}
					&\beta_1(i) = { \bar{\rho}^{+(i)}_{4\tau} \over 2\rho^{-}_{4\tau}(i) -  \bar{\rho}^{+(i)}_{4\tau} }, \qquad \beta_2(i) 
					=  \left(2{\left( {\bar{\rho}^{+(i)}_{4\tau} } +  {1 \over \theta^2} \right)  - \eta_1^2 \rho^-_{4\tau}(i)   \over \eta_1^2 \rho_{4\tau}^{-}(i)} +  {2\sqrt{\eta_1^2-1} \over \eta_1 \rho_{4\tau}^{-}(i) }(3 \mathbb{E}_{j_k} (\rho^{+}_\tau(i,j_k))^2+1)  \right),
				\end{align*}
				\begin{align*}
					&\nu  = 
					\begin{cases}
						(1 + \eta_2)^2  \max_i \left(  {8\beta_1(i) \over (\rho_{4\tau}^{-}(i) )^2} + {4 \over \bar{\rho}^{+(i)}_{4\tau} (2\rho^{-}_{4\tau}(i) - \bar{\rho}^{+(i)}_{4\tau} ) }  + {\beta_1(i) \over \rho_{4\tau}^{-}(i) }{6\sqrt{\eta_1^2-1} \over \eta_1}   \right) \zeta_*^2 \\
						\qquad+  (1 + \eta_2)^2  \sum\limits_{i=1}^N p_i \left[{\beta_1(i) \over \rho_{4\tau}^{-}(i) }  \left({2\theta^2 } + {6\sqrt{\eta_1^2-1} \over \eta_1} \right) + {4 \over \bar{\rho}^{+(i)}_{4\tau} (2\rho^{-}_{4\tau}(i) - \bar{\rho}^{+(i)}_{4\tau} ) } \right]\sigma_i^2  \qquad \text{if $\eta_1  > 1$},\\
						(1 + \eta_2)^2 \max_i \left(  {8\beta_1(i) \over (\rho_{4\tau}^{-}(i) )^2} + {4 \over \bar{\rho}^{+(i)}_{4\tau} (2\rho^{-}_{4\tau}(i) - \bar{\rho}^{+(i)}_{4\tau} ) } \right) \sum\limits_{i=1}^N p_i  \max\limits_{\substack{ \Omega \subset [d] \\ |\Omega| = 4\tau}}  \| P_{\Omega}  \nabla f_i(x^*) \|_2^2 \\
						\qquad+ (1 + \eta_2)^2  \sum_{i=1}^N p_i \left[2{\beta_1(i) \over \rho_{4\tau}^{-}(i) }  {\theta^2 }  + {4 \over \bar{\rho}^{+(i)}_{4\tau} (2\rho^{-}_{4\tau}(i) - \bar{\rho}^{+(i)}_{4\tau} ) } \right]\sigma_i^2 \qquad \text{if $\eta_1  = 1$}.
					\end{cases}
				\end{align*}
				
			\end{thm}
			
			\begin{rem}
				\label{rem:main_remark1}
				There are key factors that impact the rate of convergence  of FedGradMP in the number of communication rounds and the residual error in Theorem \ref{thm:main_theorem1} 
				as we discuss below. 
				
				\begin{enumerate}
					\item {\bf{Impact of parameters  $\beta_1(i)$ and $\beta_2(i)$}.} For a fixed number of local iterations $K$, one can see that as the product of the two parameters $\beta_1(i)$ and $\beta_2(i)$ becomes small, the convergence rate $\kappa$ improves (decreases). The product $\beta_1(i)\beta_2(i)$ becomes small as the $\mathcal{A}$-RSC constant $\rho^{-}_{4\tau}(i)$ increases or the $\mathcal{A}$-RSS constant
					$\bar{\rho}^{+(i)}_{4\tau}$ decreases. Choosing a proper dictionary could often improve the RSC/RSS constants as shown in Section \ref{sec:Improving_RSC_using_random_dictionary} consequently leading to better convergence, which is also numerically demonstrated in Section \ref{sec:Improving_RSC_for_FEMNIST}.

					\item  {\bf{Impact of the local iteration number $K$ on the convergence rate.}} Increasing the local iteration number $K$ also makes the convergence rate $\kappa$ decrease when all the terms $(1 + \eta_2)^2  \beta_1(i)\beta_2(i)$ are less than $1$. To see this, recall that the convergence rate is given by 	
					\[\kappa = (2 \eta_3^2 + 2) \sum_{i=1}^N p_i \left[(1 + \eta_2)^2  \beta_1(i)\beta_2(i) \right]^K.\] 
					In this case, increasing $K$ leads to the decay of each term in 
					$\sum_{i=1}^N p_i \left[(1 + \eta_2)^2  \beta_1(i)\beta_2(i) \right]^K$, improving the convergence rate. It is possible, however, some of the terms in the sum $(1 + \eta_2)^2  \beta_1(i)\beta_2(i)$ exceed $1$, while the sum $\sum_{i=1}^N p_i \left[(1 + \eta_2)^2  \beta_1(i)\beta_2(i) \right]^K$ is still less than $1$ for small $K$, making the convergence rate less than $1$.  In this case, as we increase the local iteration number $K$, the largest term starts dominating the sum which could increase the convergence rate (even make it greater than $1$ for large $K$), degrading the performance of FedGradMP. 
					
					\item {\bf{Impact of the local iteration number $K$ on the residual error.}}
					
					The residual error 
					${(2 \eta_3^2 + 2) \nu \over 1-\kappa} \sum_{i=1}^N p_i { 1 - \mu(i)^K   \over 1 - \mu(i)}$ depends on the local iteration number $K$ in a more complicated way. Even when all terms $(1 + \eta_2)^2  \beta_1(i)\beta_2(i)$ are less than $1$, in which case increasing $K$ makes the convergence rate $\kappa$ decrease (making the factor ${1 \over 1-\kappa}$ in the residual error decrease), but the factor ${ 1 - \mu(i)^K  \over 1 - \mu(i)}$ increases in $K$. Hence, the dependence of residual error on the local iteration number $K$ may not be simple and this is actually what we observe in the numerical experiment in Section \ref{section:Impact of the number of local iterations}. This is consistent with commonly accepted knowledge on the effect of the local iteration number on the residual error in the FL literature \cite{charles2021convergence,wang2021field, khaled2020tighter, yuan2021federated,  tong2020federated}: taking more local steps at clients makes the local estimates closer to the local solutions while the local estimates could deviate from the global solution in the FL environment in general.

				\end{enumerate}
				
			\end{rem}
			
			\begin{rem} [Interpretation of Theorem \ref{thm:main_theorem1}]
				\label{rem:main_remark2}
				Theorem \ref{thm:main_theorem1} states that the iterates of FedGradMP converge linearly up to the residual error of the solution $x^*$ as long as $\kappa < 1$. 
				The size of the residual error is proportional to $\nu$. In particular, from the expression for $\nu$ in Theorem \ref{thm:main_theorem1}, one can see that $\nu = 0$ if the heterogeneity parameter $\zeta_* = 0$ and stochastic gradient noise $\sigma_i = 0$ for all $i \in [N]$. We take a look at the related scenarios in more detail below. 
				\begin{itemize}
					\item $\zeta_* = 0$ or $\nabla f_i(x^*) = 0$ for all the client objective function $f_i$. For example, the function $f_i$ could be the square loss for the noiseless observations of $x^*$ with sparsity level $\tau$.
					
					\item $\sigma_i = 0$ for all $1 \le i \le N$ holds if and only if $\nabla g_{i,j}(x) = \nabla f_i(x)$ almost surely for all $\tau$-sparse vectors.
					This happens when the full batch of each client is used instead of mini-batches. In particular, when the projection operator is exact ($\eta_1 = 1$) and under a slightly strong heterogeneity assumption, the residual error is statistically optimal. 
					More precisely, under slightly strong heterogeneity assumptions only at the solution $x^*$ such as  
					\begin{align}
						\label{assump:sligtly_strong_dissimilarity}
						\|P_{\Omega}  \nabla f_i(x^*) - P_{\Omega} \nabla f(x^*) \|_2^2 \le \beta^2 \|P_{\Omega}  \nabla f(x^*)\|_2^2,  
					\end{align}
					where $\beta > 0$ and $\Omega$ is any subset of $[d]$ with size $4\tau$, 
					one can see that we have 
					\[
					\max\limits_{\substack{ \Omega \subset [d] \\ |\Omega| = 4\tau}} \| P_{\Omega}  \nabla f_i(x^*) \|_2^2 \le 2(1+ \beta^2) \max\limits_{\substack{ \Omega \subset [d] \\ |\Omega| = 4\tau}} \|P_{\Omega}  \nabla f(x^*)\|_2^2.
					\]
					From Theorem \ref{thm:main_theorem1}, after a sufficient number of rounds, we have
					\[
					\mathbb{E} \|x_{t+1} - x^*\|_2
					\le \left[\mathbb{E} \|x_{t+1} - x^*\|_2^2\right]^{1/2}
					\le  O \Bigg( \bigg(\sum\limits_{i=1}^N p_i \max\limits_{\substack{ \Omega \subset [d] \\ |\Omega| = 4\tau}} \| P_{\Omega}  \nabla f_i(x^*) \|_2^2 \bigg)^{1/2} \Bigg), 
					\]
					which is bounded from above by  $O \Bigg(  \max\limits_{\substack{ \Omega \subset [d] \\ |\Omega| = 4\tau}} \| P_{\Omega}  \nabla f(x^*) \|_2  \Bigg)$.
					This is the optimal statistical bias for commonly used FL data including sub-Gaussian data sets of size $|D|$ that are independently generated for each client, which is of order of $O \left(  \sqrt{\tau \log n \over N|D|} \right)$ for the sparse linear regression problem (when $f_i$ are the square loss functions). 
					The uniform bounded heterogeneity condition, which is much stronger than \eqref{assump:sligtly_strong_dissimilarity} is used to show the optimal statistical recovery of Fast FedDualAvg \cite{bao2022fast}. 
					See \cite{nguyen2017linear, tong2020federated, bao2022fast} for more details. 
					
					\item The parameter $\theta$ provides  a trade-off between the convergence rate and the residual error due to the stochastic gradient. In particular, when the full batch is used ($\sigma_i = 0$), then one can set $\theta = \infty$ giving the fastest convergence. 
					
					\item When $\sigma_i \neq 0$, the second term of the residual error $\nu$ is not vanishing in the number of rounds $t$. The similar term for FedHT/FedIterHT \cite{tong2020federated} decreases in $t$, but this requires that the mini-batch size at each client goes to infinity in $t$, which could severely restrict the number of communication rounds for the applicability of their theory. The idea of increasing the mini-batch size in the number of iterations is not new and has been used in \cite{zhou2018efficient, friedlander2012hybrid}. However, the settings for these works are not for FL and the rate of mini-batch size growth is moderate, whereas the growth rates of FedHT/FedIterHT need to be generally much higher -- they grow exponentially in the number of local iterations at clients.  This potential issue in FL methods based on exponentially increasing mini-batch sizes is also pointed out in \cite{haddadpour2019local}.
				\end{itemize}
			\end{rem}
			
			\begin{rem}
				[Comparisons between our results and previous works]
				
				Under strong convexity/smoothness assumptions for local objective functions $f_i$, it has been proved that the methods based on \emph{FedAvg} cannot converge linearly to the solution $x^*$ in general \cite{khaled2020tighter, woodworth2020minibatch}. As for general convex and smoothness conditions, the convergence to a neighborhood of $x^*$ is proven to be at most sublinear, whereas our work shows the linear convergence up to a small neighborhood under restricted convexity/smoothness assumptions. 
				We want to make clear that we are not claiming that FedGradMP converges linearly to the $x^*$ but a small neighborhood of $x^*$ unless $\zeta_* = 0$  and $\sigma_i = 0$ for all $i$. 
			\end{rem}

			\begin{proof} [Proof of Theorem \ref{thm:main_theorem1}]
				Let $\mathcal{F}^{(t)}$ be the filtration by all the randomness up to the $t$-th communication round, which is all the selected mini-batch indices at all the client up to the $t$-th round. 
				We begin with analyzing $\mathbb{E} \left[ \|x_{t+1} - x^*\|_2^2 | \mathcal{F}^{(t)} \right]$, the expected error of the global iterate $x_{t+1}$ at the $(t+1)$-th round and $x^*$ conditioned on $\mathcal{F}^{(t)}$. Because we will work with this conditional expectation until the very end of the proof, by abusing the notation slightly, $\mathbb{E} \left[ \cdot | \mathcal{F}^{(t)} \right]$ will be denoted by $\mathbb{E}[\cdot]$.
				
				\begin{align}
					\nonumber
					\mathbb{E} \|x_{t+1} - x^*\|_2^2  
					&  = \mathbb{E} \left \| P_{\Lambda_s} \left( \sum_{i=1}^N p_i x_{t,{K+1}}^{(i)} \right) - \sum_{i=1}^N p_i x_{t,{K+1}}^{(i)} + \sum_{i=1}^N p_i x_{t,{K+1}}^{(i)}  - x^* \right \|_2^2 \\ \nonumber
					&  \le 2 \mathbb{E} \left \| P_{\Lambda_s} \left( \sum_{i=1}^N p_i x_{t,{K+1}}^{(i)} \right) - \sum_{i=1}^N p_i x_{t,{K+1}}^{(i)} \right \|_2^2 +  2 \mathbb{E} \left \| \sum_{i=1}^N p_i x_{t,{K+1}}^{(i)}  - x^* \right \|_2^2 \\ \nonumber
					&  \le 2 \eta_3^2 \mathbb{E} \left \| \mathcal{H}_\tau \left( \sum_{i=1}^N p_i x_{t,{K+1}}^{(i)} \right) - \sum_{i=1}^N p_i x_{t,{K+1}}^{(i)} \right \|_2^2 +  2 \mathbb{E} \left \| \sum_{i=1}^N p_i x_{t,{K+1}}^{(i)}  - x^* \right \|_2^2 \\ \nonumber
					&  \le (2 \eta_3^2 + 2) \mathbb{E} \left \| \sum_{i=1}^N p_i x_{t,{K+1}}^{(i)}  - x^* \right \|_2^2 \\ \nonumber
					& = (2 \eta_3^2 + 2)  \mathbb{E} \left \|  \sum_{i=1}^N p_i x_{t,{K+1}}^{(i)}   -  \sum_{i=1}^N p_i x^*  \right \|_2^2 \\ 
					\label{eq:main_thm_bound1}
					& \le  (2 \eta_3^2 + 2)   \sum_{i=1}^N p_i \mathbb{E}^{(i)}_{J_K} \left \|  x_{t,{K+1}}^{(i)}   - x^*  \right \|_2^2 
				\end{align} where the second inequality follows from the definition of the approximation projector operator,  the third follows from the fact that both $x^*$ and $\mathcal{H}_\tau \left( \sum_{i=1}^N p_i x_{t,{K+1}}^{(i)} \right)$ are $\tau$-sparse but $\mathcal{H}_\tau \left( \sum_{i=1}^N p_i x_{t,{K+1}}^{(i)} \right)$ is the best $\tau$-sparse approximation of $\sum_{i=1}^N p_i x_{t,{K+1}}^{(i)}$ with respect to the dictionary $\mathcal{A}$, and the last one is obtained by applying the Jensen's inequality.
				
				After applying Lemma \ref{lem:GradMP_lem1}, \ref{lem:GradMP_lem2}, and \ref{lem:GradMP_lem3} sequentially to \eqref{eq:main_thm_bound1}, we obtain
				
				\begin{align}
					\label{eq:main_thm_bound_1.5}
					& \sum_{i=1}^N p_i \mathbb{E}^{(i)}_{J_K} \left \|  x_{t,{K+1}}^{(i)}   - x^*  \right \|_2^2   \\
					\nonumber
					& \le (1 + \eta_2)^2 \sum_{i=1}^N p_i \mathbb{E}^{(i)}_{J_K}\left \|b^{(i)}_{t,K} - x^* \right \|_2^2 \\
					\nonumber
					& \le  (1 + \eta_2)^2 \sum_{i=1}^N p_i \left[  \beta_1(i)  \mathbb{E}^{(i)}_{J_K}  \|P^\perp_{\widehat{\Gamma}}( b^{(i)}_{t,K} - x^*)\|_2^2 + \xi_1(i) \right] \\
					\nonumber
					& =  (1 + \eta_2)^2  \sum_{i=1}^N p_i \beta_1(i) \mathbb{E}^{(i)}_{J_{K-1}, j_K}  \|P^\perp_{\widehat{\Gamma}}( b^{(i)}_{t,K} - x^*)\|_2^2  +   (1 + \eta_2)^2 \sum_{i=1}^N p_i \xi_1(i) \\
					\nonumber
					& \le  (1 + \eta_2)^2 \sum_{i=1}^N p_i  \beta_1(i) \left[\beta_2(i) \mathbb{E}^{(i)}_{J_{K-1}} \| x^{(i)}_{t,K}- x^* \|_2^2 + \xi_2(i) \right] +   (1 + \eta_2)^2 \sum_{i=1}^N p_i \xi_1(i) \\
					\label{eq:main_thm_bound2}
					& = (1 + \eta_2)^2 \sum_{i=1}^N p_i  \beta_1(i) \beta_2(i) \mathbb{E}^{(i)}_{J_{K-1}}  \| x^{(i)}_{t,K} - x^* \|_2^2 +   (1 + \eta_2)^2  \sum_{i=1}^N p_i \left(  \beta_1(i) \xi_2(i) + \xi_1(i)  \right).
				\end{align}
				
				First, the term $\xi_1(i)$ can be bounded as follows:
				\begin{align*}
					\xi_1(i) 
					&= {2 ( \mathbb{E}^{(i)}_{J_K, j} \|P_{\widehat{\Gamma}} \nabla g_{i,j}(x^*)\|_2^2) \over \bar{\rho}^{+(i)}_{4\tau} (2\rho^{-}_{4\tau}(i) - \bar{\rho}^{+(i)}_{4\tau} ) } \le {4 \over \bar{\rho}^{+(i)}_{4\tau} (2\rho^{-}_{4\tau}(i) - \bar{\rho}^{+(i)}_{4\tau} ) } \left(  \max\limits_{\substack{ \Omega \subset [d] \\ |\Omega| = 4\tau}}  \| P_{\Omega}  \nabla f_i(x^*) \|_2^2 + \sigma_i^2 \right) 
				\end{align*} This is from the property of the projection operator and \eqref{asump:local_gradients2} in Assumption \ref{asump:local_gradients} implying $$\mathbb{E}^{(i)}_{j}\|P_{\widehat{\Gamma}} \nabla g_{i,j}(x) - P_{\widehat{\Gamma}} \nabla f_i(x) \|_2^2 \le \mathbb{E}^{(i)}_{j} \| \nabla g_{i,j}(x) - \nabla f_i(x) \|_2^2 \le \sigma^2_i,$$
				so we have
				\[
				\mathbb{E}^{(i)}_{j} \|P_{\widehat{\Gamma}} \nabla g_{i,j}(x^*)\|_2^2 \le 2(\mathbb{E}^{(i)}_{j} \|P_{\widehat{\Gamma}} \nabla f_i(x^*) \|_2^2 + \sigma^2_i) \le 2 \left( \max_{\substack{ \Omega \subset [d] \\ |\Omega| = 4\tau}}  \| P_{\Omega}  \nabla f_i(x^*) \|_2^2 + \sigma_i^2 \right).
				\]
				
				Hence, each term $\beta_1(i) \xi_2(i) + \xi_1(i)$ in \eqref{eq:main_thm_bound2} can be bounded as below. 
				\begin{align*}
					&\beta_1(i) \xi_2(i) + \xi_1(i)\\
					&\le \beta_1(i) \left(  {8 \over (\rho_{4\tau}^{-}(i) )^2}  \max\limits_{\substack{ \Omega \subset [d] \\ |\Omega| = 4\tau}}  \| P_{\Omega}  \nabla f_i(x^*) \|_2^2 + {1 \over \rho_{4\tau}^{-}(i) } \left[\left({2\theta^2 } + {6\sqrt{\eta_1^2-1} \over \eta_1} \right) \sigma_i^2 + {6\sqrt{\eta_1^2-1} \over \eta_1}\| \nabla f_i \left(x^* \right) \|_2^2 \right] \right)\\
					& \qquad + {4 \over \bar{\rho}^{+(i)}_{4\tau} (2\rho^{-}_{4\tau}(i) - \bar{\rho}^{+(i)}_{4\tau} ) } \left(  \max\limits_{\substack{ \Omega \subset [d] \\ |\Omega| = 4\tau}}  \| P_{\Omega}  \nabla f_i(x^*) \|_2^2 + \sigma_i^2 \right)\\
					&\le  \left(  {8\beta_1(i) \over (\rho_{4\tau}^{-}(i) )^2} + {4 \over \bar{\rho}^{+(i)}_{4\tau} (2\rho^{-}_{4\tau}(i) - \bar{\rho}^{+(i)}_{4\tau} ) } \right)  \max\limits_{\substack{ \Omega \subset [d] \\ |\Omega| = 4\tau}}  \| P_{\Omega}  \nabla f_i(x^*) \|_2^2 + {\beta_1(i) \over \rho_{4\tau}^{-}(i) }{6\sqrt{\eta_1^2-1} \over \eta_1}\| \nabla f_i \left(x^* \right) \|_2^2  \\
					& \qquad + \left[{\beta_1(i) \over \rho_{4\tau}^{-}(i) }  \left({2\theta^2 } + {6\sqrt{\eta_1^2-1} \over \eta_1} \right) + {4 \over \bar{\rho}^{+(i)}_{4\tau} (2\rho^{-}_{4\tau}(i) - \bar{\rho}^{+(i)}_{4\tau} ) } \right]\sigma_i^2.
				\end{align*}
				
				By plugging the above bound to \eqref{eq:main_thm_bound2}, we have
				
				\small
				\begin{align}
					\nonumber
					& \sum_{i=1}^N p_i \mathbb{E}^{(i)}_{J_K} \left \|  x_{t,K+1}^{(i)}   - x^*  \right \|_2^2   \\
					\label{eq:main_convergence_result}
					& \le  (1 + \eta_2)^2 \sum_{i=1}^N p_i  \beta_1(i) \beta_2(i)  \mathbb{E}^{(i)}_{J_{K-1}} \| x^{(i)}_{t,K} - x^* \|_2^2 \\
					\nonumber
					& +  (1 + \eta_2)^2  \sum_{i=1}^N p_i \left(  \left(  {8\beta_1(i) \over (\rho_{4\tau}^{-}(i) )^2} + {4 \over \bar{\rho}^{+(i)}_{4\tau} (2\rho^{-}_{4\tau}(i) - \bar{\rho}^{+(i)}_{4\tau} ) } \right)  \max\limits_{\substack{ \Omega \subset [d] \\ |\Omega| = 4\tau}}  \| P_{\Omega}  \nabla f_i(x^*) \|_2^2 + {\beta_1(i) \over \rho_{4\tau}^{-}(i) }{6\sqrt{\eta_1^2-1} \over \eta_1}\| \nabla f_i \left(x^* \right) \|_2^2   \right)\\
					\nonumber
					& + (1 + \eta_2)^2  \sum_{i=1}^N p_i \left[{\beta_1(i) \over \rho_{4\tau}^{-}(i) }  \left({2\theta^2 } + {6\sqrt{\eta_1^2-1} \over \eta_1} \right) + {4 \over \bar{\rho}^{+(i)}_{4\tau} (2\rho^{-}_{4\tau}(i) - \bar{\rho}^{+(i)}_{4\tau} ) } \right]\sigma_i^2. 
					\nonumber
				\end{align} 
				\normalsize
				
				Now consider first the case when $\eta_1 > 1$. Then, since $\max\limits_{\substack{ \Omega \subset [d] \\ |\Omega| = 4\tau}}  \| P_{\Omega}  \nabla f_i(x^*) \|_2^2 \le \| \nabla f_i \left(x^* \right) \|_2^2$, we have
				
				\begin{align*}
					\nonumber
					& \sum_{i=1}^N p_i \mathbb{E}^{(i)}_{J_K} \left \|  x_{t,K+1}^{(i)}   - x^*  \right \|_2^2   \\
					& \le  (1 + \eta_2)^2 \sum_{i=1}^N p_i  \beta_1(i) \beta_2(i)  \mathbb{E}^{(i)}_{J_{K-1}} \| x^{(i)}_{t,K} - x^* \|_2^2 \\
					& +  (1 + \eta_2)^2  \sum_{i=1}^N p_i \left(  {8\beta_1(i) \over (\rho_{4\tau}^{-}(i) )^2} + {4 \over \bar{\rho}^{+(i)}_{4\tau} (2\rho^{-}_{4\tau}(i) - \bar{\rho}^{+(i)}_{4\tau} ) }  + {\beta_1(i) \over \rho_{4\tau}^{-}(i) }{6\sqrt{\eta_1^2-1} \over \eta_1}   \right) \| \nabla f_i \left(x^* \right) \|_2^2 \\
					& + (1 + \eta_2)^2  \sum_{i=1}^N p_i \left[{\beta_1(i) \over \rho_{4\tau}^{-}(i) }  \left({2\theta^2 } + {6\sqrt{\eta_1^2-1} \over \eta_1} \right) + {4 \over \bar{\rho}^{+(i)}_{4\tau} (2\rho^{-}_{4\tau}(i) - \bar{\rho}^{+(i)}_{4\tau} ) } \right]\sigma_i^2 \\
					& \le  (1 + \eta_2)^2 \sum_{i=1}^N p_i  \beta_1(i) \beta_2(i)  \mathbb{E}^{(i)}_{J_{K-1}} \| x^{(i)}_{t,K} - x^* \|_2^2 \\
					& +  (1 + \eta_2)^2  \max_i \left(  {8\beta_1(i) \over (\rho_{4\tau}^{-}(i) )^2} + {4 \over \bar{\rho}^{+(i)}_{4\tau} (2\rho^{-}_{4\tau}(i) - \bar{\rho}^{+(i)}_{4\tau} ) }  + {\beta_1(i) \over \rho_{4\tau}^{-}(i) }{6\sqrt{\eta_1^2-1} \over \eta_1}   \right) \sum_{i=1}^N p_i  \| \nabla f_i \left(x^* \right) \|_2^2 \\
					& + (1 + \eta_2)^2  \sum_{i=1}^N p_i \left[{\beta_1(i) \over \rho_{4\tau}^{-}(i) }  \left({2\theta^2 } + {6\sqrt{\eta_1^2-1} \over \eta_1} \right) + {4 \over \bar{\rho}^{+(i)}_{4\tau} (2\rho^{-}_{4\tau}(i) - \bar{\rho}^{+(i)}_{4\tau} ) } \right]\sigma_i^2.
				\end{align*} 
				
				Let $\mu(i) =  (1 + \eta_2)^2  \beta_1(i) \beta_2(i)$ and 
				\begin{align*}
					&\nu  =   (1 + \eta_2)^2  \max_i \left(  {8\beta_1(i) \over (\rho_{4\tau}^{-}(i) )^2} + {4 \over \bar{\rho}^{+(i)}_{4\tau} (2\rho^{-}_{4\tau}(i) - \bar{\rho}^{+(i)}_{4\tau} ) }  + {\beta_1(i) \over \rho_{4\tau}^{-}(i) }{6\sqrt{\eta_1^2-1} \over \eta_1}   \right) \zeta_*^2 \\
					& + (1 + \eta_2)^2  \sum_{i=1}^N p_i \left[{\beta_1(i) \over \rho_{4\tau}^{-}(i) }  \left({2\theta^2 } + {6\sqrt{\eta_1^2-1} \over \eta_1} \right) + {4 \over \bar{\rho}^{+(i)}_{4\tau} (2\rho^{-}_{4\tau}(i) - \bar{\rho}^{+(i)}_{4\tau} ) } \right]\sigma_i^2.
				\end{align*}
				
				Hence, when $\eta_1 > 1$, we have
				\begin{align}
					\label{eq:main_theorem_bound2}
					&  \sum_{i=1}^N p_i \mathbb{E}^{(i)}_{J_K} \left \|  x_{t,K+1}^{(i)}   - x^*  \right \|_2^2 \le \sum_{i=1}^N p_i \mu(i) \mathbb{E}^{(i)}_{J_{K-1}} \| x^{(i)}_{t,K} - x^* \|_2^2 + \nu. 
				\end{align}
				
				The case when the projection operator for the gradient is exact ($\eta_1 = 1$) follows the same argument. Setting $\eta_1 = 1$ in the inequality \eqref{eq:main_convergence_result} reduces the inequality to
				\begin{align*}
					&  \sum_{i=1}^N p_i \mathbb{E}^{(i)}_{J_K} \left \|  x_{t,K+1}^{(i)}   - x^*  \right \|_2^2   \\
					& \le  (1 + \eta_2)^2 \sum_{i=1}^N p_i  \beta_1(i) \beta_2(i)  \mathbb{E}^{(i)}_{J_{K-1}} \| x^{(i)}_{t,K} - x^* \|_2^2 \\
					& +  (1 + \eta_2)^2  \sum_{i=1}^N p_i  \left(  {8\beta_1(i) \over (\rho_{4\tau}^{-}(i) )^2} + {4 \over \bar{\rho}^{+(i)}_{4\tau} (2\rho^{-}_{4\tau}(i) - \bar{\rho}^{+(i)}_{4\tau} ) } \right)  \max\limits_{\substack{ \Omega \subset [d] \\ |\Omega| = 4\tau}}  \| P_{\Omega}  \nabla f_i(x^*) \|_2^2 \\
					& + (1 + \eta_2)^2  \sum_{i=1}^N p_i \left[2{\beta_1(i) \over \rho_{4\tau}^{-}(i) }  {\theta^2 }  + {4 \over \bar{\rho}^{+(i)}_{4\tau} (2\rho^{-}_{4\tau}(i) - \bar{\rho}^{+(i)}_{4\tau} ) } \right]\sigma_i^2\\
					& \le  (1 + \eta_2)^2 \sum_{i=1}^N p_i  \beta_1(i) \beta_2(i)  \mathbb{E}^{(i)}_{J_{K-1}} \| x^{(i)}_{t,K} - x^* \|_2^2 \\
					& +  (1 + \eta_2)^2 \max_i \left(  {8\beta_1(i) \over (\rho_{4\tau}^{-}(i) )^2} + {4 \over \bar{\rho}^{+(i)}_{4\tau} (2\rho^{-}_{4\tau}(i) - \bar{\rho}^{+(i)}_{4\tau} ) } \right) \sum_{i=1}^N p_i  \max\limits_{\substack{ \Omega \subset [d] \\ |\Omega| = 4\tau}}  \| P_{\Omega}  \nabla f_i(x^*) \|_2^2 \\
					& + (1 + \eta_2)^2  \sum_{i=1}^N p_i \left[2{\beta_1(i) \over \rho_{4\tau}^{-}(i) }  {\theta^2 }  + {4 \over \bar{\rho}^{+(i)}_{4\tau} (2\rho^{-}_{4\tau}(i) - \bar{\rho}^{+(i)}_{4\tau} ) } \right]\sigma_i^2. 
				\end{align*} 
				
				The bound for $\nu$ for the exact projection case ($\eta_1 = 1$) is given as below.
				\begin{align*}
					&\nu  =     (1 + \eta_2)^2 \max_i \left(  {8\beta_1(i) \over (\rho_{4\tau}^{-}(i) )^2} + {4 \over \bar{\rho}^{+(i)}_{4\tau} (2\rho^{-}_{4\tau}(i) - \bar{\rho}^{+(i)}_{4\tau} ) } \right) \sum_{i=1}^N p_i  \max\limits_{\substack{ \Omega \subset [d] \\ |\Omega| = 4\tau}}  \| P_{\Omega}  \nabla f_i(x^*) \|_2^2 \\
					& + (1 + \eta_2)^2  \sum_{i=1}^N p_i \left[2{\beta_1(i) \over \rho_{4\tau}^{-}(i) }  {\theta^2 }  + {4 \over \bar{\rho}^{+(i)}_{4\tau} (2\rho^{-}_{4\tau}(i) - \bar{\rho}^{+(i)}_{4\tau} ) } \right]\sigma_i^2.
				\end{align*}

				Then, applying the bound \eqref{eq:main_theorem_bound2} to \eqref{eq:main_thm_bound1} repeatedly using the induction argument on $K$, we have
				\begin{align*}
					\mathbb{E} \|x_{t+1} - x^*\|_2^2 
					& \le  (2 \eta_3^2 + 2) \left(  \sum_{i=1}^N p_i \left[  \mathbb{E}^{(i)} \mu(i) ^K\|x^{(i)}_{t,1} - x^*\|_2^2 \right] +  \sum_{i=1}^N p_i {\nu (1 - \mu(i)^K) \over 1 - \mu(i)} \right) \\
					&  = (2 \eta_3^2 + 2) \mathbb{E} \left(  \sum_{i=1}^N p_i \left[   \mu(i) ^K\|x_t- x^*\|_2^2 \right] +  \sum_{i=1}^N p_i {\nu (1 - \mu(i)^K) \over 1 - \mu(i)} \right) \\
					& =  \kappa  \|x_t - x^*\|_2^2  +  (2 \eta_3^2 + 2) \nu \sum_{i=1}^N p_i {(1 - \mu(i)^K)  \over 1 - \mu(i)} ,
				\end{align*} 
				where the first equality follows from $x^{(i)}_{t,1} = x_t$ and the second follows from 
				\[
				\kappa = (2 \eta_3^2 + 2) \sum_{i=1}^N p_i \mu(i)^K = (2 \eta_3^2 + 2) \sum_{i=1}^N p_i \left[(1 + \eta_2)^2  \beta_1(i)\beta_2(i) \right]^K. 
				\]
				
				Now, taking the unconditional expectation on both sides of the above yields
				\begin{align*}
					\mathbb{E} \left[ \|x_{t+1} - x^*\|_2^2 \right] 
					&= \mathbb{E} \left[\mathbb{E} \left[ \|x_{t+1} - x^*\|_2^2 | \mathcal{F}^{(t)} \right] \right]\\
					&\le  \kappa  \mathbb{E} \left[\mathbb{E} \left[ \|x_t - x^*\|_2^2  | \mathcal{F}^{(t-1)} \right] \right]  +  (2 \eta_3^2 + 2) \nu \sum_{i=1}^N p_i {(1 - \mu(i)^K)  \over 1 - \mu(i)}.
				\end{align*} 
				By applying this result repeatedly, we obtain 
				\begin{align*}
					\mathbb{E} \|x_{t+1} - x^*\|_2^2 
					\le  \kappa^{t+1}  \|x_0 - x^*\|_2^2  +  {(2 \eta_3^2 + 2) \nu \over 1-\kappa} \sum_{i=1}^N p_i { 1 - \mu(i)^K   \over 1 - \mu(i)}.
				\end{align*}
				
				
			\end{proof}
			
			\begin{cor}\label{cor:7}
				Under the same conditions and notation in Theorem \ref{thm:main_theorem1}, we have
				\begin{align*}
					\mathbb{E} f(x_{t+1}) 
					&\le f(x^*) + {1\over 2\rho} \|\nabla f(x^*)\|_2^2 + \rho \left[\kappa^{t+1} \|x_0 - x^*\|_2^2  +  {(2 \eta_3^2 + 2) \nu \over 1-\kappa} \sum_{i=1}^N p_i { 1 - \mu(i)^K   \over 1 - \mu(i)} \right],
				\end{align*} where $\rho = \sum_{i=1}^N p_i \bar{\rho}^{+{(i)}}_\tau$.  
			\end{cor}
			See Appendix for the proof of Lemma \ref{cor:7}.
			
			\section{Discussion and Extensions}
			\label{sec:ext}

			\subsection{Inexact FedGradMP}
			
			\begin{algorithm} 
				\caption{Inexact FedGradMP with partial participation }
				\begin{algorithmic} 
					\Input The number of rounds $T$, the number of clients $N$, the cohort size $L$, the number of local iterations $K$, weight vector $p$, the estimated sparsity level $\tau$, $\eta_1, \eta_2, \eta_3, \delta$.
					
					\Output $\hat{x} = x_T$.
					
					\Initialize	$x_0 = 0$, $\Lambda = \emptyset$.
					
					\textbf{for $t=0, 1, \dots, T-1$ do}
					
					\indt \textbf{Randomly select a subset $S_t$ of clients with size $L$} 
					
					\indt \textbf{for each client $i$ in $S_t$,  do}
					
					\indt \indt $x_{t,1}^{(i)} = x_t$
					
					\indt \indt \textbf{for $k=1$ to $K$ do}
					
					\indt \indt \indt Select a mini-batch index set $j_k := i_{t,k}^{(i)}$ uniformly at random from $\{1,2,\dots, M\}$
					
					\indt \indt \indt Calculate the stochastic gradient $r^{(i)}_{t,k} = \nabla g_{i,j_k} \left( x_{t,k}^{(i)} \right)$
					
					\indt \indt \indt $\Gamma = $ approx$_{2\tau} (r^{(i)}_{t,k}, \eta_1)$
					
					\indt \indt \indt $\widehat{\Gamma} = \Gamma \cup \Lambda$
					
					\indt \indt \indt Solve $b^{(i)}_{t,k} = \argmin{x} f_i(x), \quad x \in \mathcal{R}(\mathcal{A}_{\widehat{\Gamma}})$ up to accuracy $\delta$
					
					\indt \indt \indt $\Lambda = $ approx$_{\tau} (b^{(i)}_{t,k}, \eta_2)$
					
					\indt \indt \indt $x_{t,k+1}^{(i)} = P_{\Lambda} (b^{(i)}_{t,k})$
					
					\indt \indt \indt $\Big(x_{t,k+1}^{(i)} \gets \Pi_R \left(x_{t,k+1}^{(i)}\right)\Big)$ \qquad [Optional projection onto a ball]
					
					\indt \indt	\ForEnd
					
					\indt \ForEnd
					
					\indt $\Lambda_s = $     approx$_{\tau} \left( \sum_{i=1}^N p_i x_{t,K+1}^{(i)}, \eta_3 \right)$
					
					\indt $x_{t+1} = P_{\Lambda_s} \left( \sum_{i=1}^N p_i x_{t,K+1}^{(i)} \right)$
					
					\indt $\Big(x_{t+1} \gets \Pi_R \left(x_{t+1}\right)\Big)$ \qquad [Optional projection onto a ball]
					
					\ForEnd
					
				\end{algorithmic}
				\label{alg:InexactFedGradMP}
			\end{algorithm}
			
				
				
				
			
			In FedGradMP, each client solves the minimization problem $\argmin{x} f_i(x)$ over a subspace $ \mathcal{R}(\mathcal{A}_{\widehat{\Gamma}})$ to update the support estimate of the solution $x^*$. When a closed-form solution exists to the minimization problem such as the least squares problem and the sparsity level $\tau$ is relatively small compared with the signal dimension, an exact minimizer can be obtained efficiently. This can be achieved, for example, by computing the pseudo-inverse with respect to a $\tau$-dimensional subspace $\mathcal{R}(\mathcal{A}_{\widehat{\Gamma}})$ or by algorithms based on Cholesky, QR factorizations, or SVD for the least squares problem \cite{trefethen1997numerical, golub2013matrix}. 
			
			But for the other cases, although the sub-optimization problem is typically convex due to the $\mathcal{A}$-RSC assumption, one may still want to reduce the  computational cost in the optimization. By solving it only approximately but with a desired accuracy, we can save computational resources further.
			Because the local loss function $f_i$ for the $i$-th client satisfies the $\mathcal{A}$-RSC and $\mathcal{A}$-RSS properties with the respective constants  $\rho^-_\tau(i)$ and $\bar{\rho}^{+{(i)}}_\tau$, $f_i$ is strongly convex/smooth with the same constants on the domain of the minimization problem, the linear subspace $ \mathcal{R}(\mathcal{A}_{\widehat{\Gamma}})$. Recall that $|D_i|$ is the number of data points at the $i$-th client. 
			We define a $\delta$-approximate solution to $\argmin{x} f_i(x)$ with $x \in \mathcal{R}(\mathcal{A}_{\widehat{\Gamma}})$ as a vector $b$ such that $\|b - b_{\text{opt}}\|_2^2 \le \delta^2$ where $b_{\text{opt}}$ is its exact solution. 
			The number of steps required to achieve a $\delta$-approximate solution at client $i$ using popular standard algorithms is shown as follows:
			\begin{itemize}
				\item Gradient descent (GD): $O \left(|D_i| \left( {\bar{\rho}^{+{(i)}}_\tau  \over  \rho^-_\tau(i)} \right) \log \left(1 \over \delta \right) \right)$ \cite{nesterov2018lectures}.
				
				\item Stochastic gradient descent with variance reduction such as SAG or SVRG: $O \left(|D_i|  + {\bar{\rho}^{+{(i)}}_\tau  \over  \rho^-_\tau(i)} \log \left(1 \over \delta \right) \right)$ \cite{roux2012stochastic, johnson2013accelerating}.
			\end{itemize}
			
			Since the domain is a $\tau$-dimensional space, the computational complexity per data point of the above algorithms is $O(\tau)$ for the squared or logistic loss, so the overall complexity of the local step to compute a  $\delta$-approximate solution is $O \left(|D_i| \tau \left( {\bar{\rho}^{+{(i)}}_\tau  \over  \rho^-_\tau(i)} \log \left(1 \over \delta \right) \right) \right)$ for GD. For Newton's method, the total computational cost to achieve a $\delta$-approximate solution is roughly $O\left( (|D_i| \tau^2 + \tau^3) \log \left(1 \over \delta \right) \right)$ \cite{nesterov2018lectures}. Hence, if the sparsity level $\tau$ is much smaller than the signal dimension $n$, the  subproblem in FedGradMP can be solved efficiently up to accuracy $\delta$. As a comparison, most FL methods run (stochastic) gradient descent to solve $\argmin{x} f_i(x)$ over the whole space $\R^n$, which would cost computationally more to acquire its $\delta$-approximate solution.
			
			\begin{thm}
				\label{thm:main_theorem1_inexact_solver}
				Under the same notation and assumptions as in Theorem \ref{thm:main_theorem1}, for any $\theta > 0$, the expectation of the recovery error at the $(t+1)$-th round of inexact FedGradMP described in Algorithm \ref{alg:InexactFedGradMP}  obeys
				\[
				\mathbb{E} \|x_{t+1} - x^*\|_2^2 
				\le  \kappa^{t+1}  \|x_0 - x^*\|_2^2  +  {(2 \eta_3^2 + 2) (\nu + \delta^2) \over 1-\kappa} \sum_{i=1}^N p_i { 1 - \mu(i)^K   \over 1 - \mu(i)},	
				\]
				where 
				\[
				\kappa = (2 \eta_3^2 + 2) \sum_{i=1}^N p_i \left[2(1 + \eta_2)^2  \beta_1(i)\beta_2(i) \right]^K \qquad \text{and} \qquad \mu(i) =  2(1 + \eta_2)^2  \beta_1(i)\beta_2(i).
				\]
				Here, the parameters $\beta_1(i), \beta_2(i)$, and $\nu$ are the same as in Theorem \ref{thm:main_theorem1}.
			\end{thm}
			See Appendix for the proof of Theorem \ref{thm:main_theorem1_inexact_solver}.
			

			
			

			\subsection{Client sampling and the impact of cohort size}
			In practical FL scenarios, it may not be possible for all of the clients to participate in each communication round. This could  particularly stand out when there are a large population of clients or the communication bandwidth of connections between the server and clients is limited. A common theoretical assumption to capture this partial client participation is that participating clients for each communication round are drawn randomly according to some probability distribution, independent with other rounds. It could be considered as client sampling as noted in \cite{wang2021field}. 
			One could also consider more sophisticated sampling strategies such as importance sampling, but it seems to be not easy to implement such sampling techniques for FL since it could leak the private information of the client data sets \cite{chen2020optimal}. Furthermore, in many real-world scenarios, client availability (which is usually random) solely controls participation rather than the server, ruling out the potential of using such methods \cite{bonawitz2019towards}. 
			
			For simplicity, we assume that the weight $p_i$ is $1/N$ in the global objective function $f$ and a fixed number of clients (the cohort size) participate per round as in \cite{gower2019sgd} to study the impact of the cohort size.

			\begin{thm}
				\label{thm:main_theorem2}
				Assume the uniform weights $p_i = 1/N$ and $L$ participating clients are drawn uniformly at random over the client set without replacement per round.  
				Then, under the same assumptions and notation in Theorem \ref{thm:main_theorem1}, for any $\theta > 0$,  the expectation of the recovery error is bounded from above by
				\[
				\mathbb{E} \|x_{t+1} - x^*\|_2^2 
				\le  \kappa^{t+1} \|x_0 - x^*\|_2^2  +  {(2 \eta_3^2 + 2) \tilde{\nu} (1 - \mu^K)   \over (1 - \mu)(1-\kappa)},	
				\]
				where 
				\[
				\kappa =  (2 \eta_3^2 + 2) \max_{\substack{ S \subset [N] \\ |S| = L}} {1 \over L}  \sum_{i \in S}  \left[(1 + \eta_2)^2   \beta_1(i) \beta_2(i) \right]^K, \qquad 
				\mu =  \max_{i \in [N]}  \left[(1 + \eta_2)^2   \beta_1(i) \beta_2(i) \right]^K,
				\]
				and
				
				\begin{align*}
					&\tilde{\nu}  = 
					\begin{cases}
						(1 + \eta_2)^2  \max_i \left(  {8\beta_1(i) \over (\rho_{4\tau}^{-}(i) )^2} + {4 \over \bar{\rho}^{+(i)}_{4\tau} (2\rho^{-}_{4\tau}(i) - \bar{\rho}^{+(i)}_{4\tau} ) }  + {\beta_1(i) \over \rho_{4\tau}^{-}(i) }{6\sqrt{\eta_1^2-1} \over \eta_1}   \right) \zeta_*^2 \\
						\qquad+  (1 + \eta_2)^2 {1 \over L} \sum\limits_{i=1}^N \left[{\beta_1(i) \over \rho_{4\tau}^{-}(i) }  \left({2\theta^2 } + {6\sqrt{\eta_1^2-1} \over \eta_1} \right) + {4 \over \bar{\rho}^{+(i)}_{4\tau} (2\rho^{-}_{4\tau}(i) - \bar{\rho}^{+(i)}_{4\tau} ) } \right]\sigma_i^2, \quad \text{if $\eta_1 > 1$},\\
						(1 + \eta_2)^2  \max\limits_i \left(  {8\beta_1(i) \over (\rho_{4\tau}^{-}(i) )^2} + {4 \over \bar{\rho}^{+(i)}_{4\tau} (2\rho^{-}_{4\tau}(i) - \bar{\rho}^{+(i)}_{4\tau} ) }    \right) \left( {1 \over N} \sum\limits_{i=1}^N  \max\limits_{\substack{ \Omega \subset [d] \\ |\Omega| = 4\tau}} \| P_{\Omega}  \nabla f_i(x^*) \|_2^2  \right) \\
						\qquad+  (1 + \eta_2)^2  {1 \over L} \sum\limits_{i=1}^N  \left[ \beta_1(i) {2{\theta^2}  \over \rho_{4\tau}^{-}(i) } + {4 \over \bar{\rho}^{+(i)}_{4\tau} (2\rho^{-}_{4\tau}(i) - \bar{\rho}^{+(i)}_{4\tau} ) } \right]\sigma_i^2 \qquad \qquad  \qquad \text{if $\eta_1  = 1$}.
					\end{cases}.
				\end{align*}
			\end{thm}
			See Appendix for the proof of Theorem \ref{thm:main_theorem2}.
			
			
			\begin{rem} 
				Note that the convergence rate $\kappa$ of FedGradMP improves (decreases) as the cohort size $L$ increases in Theorem \ref{thm:main_theorem2}, aligning with some of the previous works about the impact of cohort size on the convergence speed \cite{yang2021achieving}.  
				Our numerical experiments in Section \ref{section:cohort_size} also validate our theory about the impact of cohort size on the convergence rates. 
				On the other hand, it appears that the residual error in Theorem \ref{thm:main_theorem2} is pessimistic and the actual behavior of FL algorithms depends on the cohort size in a more complicated way. See Section \ref{section:cohort_size} for the numerical experiment and discussion. 
				
			\end{rem}

			\subsection{FedGradMP with a constraint}
			\label{sec:constrained_FedGradMP}
			
			Many machine learning problems can be formulated as an $\ell_2$-norm constrained optimization problem \cite{loh2011high, shen2017tight, yuan2021federated}. Since we focus on the FL setting with a sparse structure, our goal is to solve the following problem: 
			\begin{align}
				\label{eq:main_problem_constrained}
				\min_{x \in \R^n} f(x) = \sum_{i=1}^N p_i f_i(x) \quad \st \quad \|x\|_{0, \mathcal{A}} \le \tau,  \quad \|x\|_2\le R,
			\end{align} for some $R > 0$, which is our main optimization problem \eqref{eq:main_problem} with the additional $\ell_2$ constraint $\|x\|_2\le R$. The $\ell_2$ constraint ensures the global minimum exists in the domain. Another advantage of using the $\ell_2$-norm constraint is that its orthogonal projection computationally is cheaper than projections of other constraints such as the $\ell_1$-norm \cite{shen2017tight}.
			We denote by $\Pi_R$ the orthogonal projection of a vector to the set $\{ \|x\|_2\le R \}$, which is implemented as follows. For any vector $u \in \R^N$,
			\begin{align*}
				\Pi_R(u)
				=\begin{cases}
					u, \qquad \qquad \text{if} \; \; \|u\|_2 \le R;\\
					Ru/\|u\|_2, \; \; \text{otherwise}. 
				\end{cases}
			\end{align*}
			
			Let $x^*$ be a minimizer of the problem \eqref{eq:main_problem_constrained} and the heterogeneity at the solution $x^*$ is defined as in Assumption \ref{assump:local_dissimilarity_mod}. By executing additional steps, the projection to a $\ell_2$-norm ball in Algorithm \ref{alg:InexactFedGradMP}, FedGradMP converges to the solution $x^*$ under the same conditions in Theorem \ref{thm:main_theorem1},  \ref{thm:main_theorem1_inexact_solver}, and \ref{thm:main_theorem2}. The proof follows a simple modification of the proofs of the theorems due to the fact that the orthogonal projection of a vector $u$ to a ball with radius $R$ does not increase the $\ell_2$-norm distance between $u$ and $v$ for a vector $v$ in the ball. 
			For instance, we replace 
			\eqref{eq:main_thm_bound_1.5} in the proof of Theorem \ref{thm:main_theorem1} as follows. 
			\begin{align*}
				& \sum_{i=1}^N p_i \mathbb{E}^{(i)}_{J_K} \left \|  x_{t,K+1}^{(i)}   - x^*  \right \|_2^2   \\
				&= \sum_{i=1}^N p_i \mathbb{E}^{(i)}_{J_K} \left \| \Pi_R \left( P_{\Lambda_s} \left(b^{(i)}_{t,K}\right) \right)  - x^*  \right \|_2^2   \\
				&\le \sum_{i=1}^N p_i \mathbb{E}^{(i)}_{J_K} \left \|  P_{\Lambda_s} \left(b^{(i)}_{t,K}\right)   - x^*  \right \|_2^2   \\
				& \le (1 + \eta_2)^2 \sum_{i=1}^N p_i \mathbb{E}^{(i)}_{J_K}\left \|b^{(i)}_{t,K} - x^* \right \|_2^2,
			\end{align*} where we have used the fact that $x^*$ belongs the $\ell_2$-norm ball with radius $R$ and the aforementioned property of $\Pi_R$ in the first inequality above. 
			
			Similarly, note that all the local iterates satisfy $\|x_{t,K+1}^{(i)} \|_2 \le R$ because of the projection to the ball in Algorithm \ref{alg:InexactFedGradMP}. Thus, their convex combination $ \sum_{i=1}^N p_i x_{t,K+1}^{(i)}$ also belongs to the ball. 
			
			Now we apply the same argument to the first step of the proof of Theorem \ref{thm:main_theorem1} 
			\begin{align*}
				\mathbb{E} \|x_{t+1} - x^*\|_2^2  
				&  = \mathbb{E} \left \| \Pi_R \left(P_{\Lambda_s} \left( \sum_{i=1}^N p_i x_{t,K+1}^{(i)} \right)\right) - \sum_{i=1}^N p_i x_{t,K+1}^{(i)} + \sum_{i=1}^N p_i x_{t,K+1}^{(i)}  - x^* \right \|_2^2 \\ 
				&  \le 2 \mathbb{E} \left \| \Pi_R \left(P_{\Lambda_s} \left( \sum_{i=1}^N p_i x_{t,K+1}^{(i)} \right)\right) - \sum_{i=1}^N p_i x_{t,K+1}^{(i)} \right \|_2^2 +  2 \mathbb{E} \left \| \sum_{i=1}^N p_i x_{t,K+1}^{(i)}  - x^* \right \|_2^2\\
				&  \le 2 \mathbb{E} \left \| P_{\Lambda_s} \left( \sum_{i=1}^N p_i x_{t,K+1}^{(i)} \right) - \sum_{i=1}^N p_i x_{t,K+1}^{(i)} \right \|_2^2 +  2 \mathbb{E} \left \| \sum_{i=1}^N p_i x_{t,K+1}^{(i)}  - x^* \right \|_2^2. 
			\end{align*}
			After these modifications, we proceed as in the rest of the proof of Theorem \ref{thm:main_theorem1}. 
			
			\subsection{Impact of dictionary choice}\label{sec:Improving_RSC_using_random_dictionary}
			Recall that our convergence guarantees depend on the restricted convexity/smoothness ($\mathcal{A}$-RSC, $\mathcal{A}$-RSS) constants $\rho^-_{4\tau}(i)$ and $\bar{\rho}^{+(i)}_{4\tau}$ as many works for sparse recovery \cite{needell2009cosamp, nguyen2017linear, foucart2013invitation, tong2020federated}. In particular, the product $\beta_1(i)\beta_2(i)$ in Theorems \ref{thm:main_theorem1}, \ref{thm:main_theorem1_inexact_solver}, and \ref{thm:main_theorem2} critically impact the convergence rate $\kappa$; for faster convergence, $\beta_1(i)$ and $ \beta_2(i)$ should be small as stated in Remark \ref{rem:main_remark1}.
			This can be achieved especially when the  $\mathcal{A}$-RSS/$\mathcal{A}$-RSC constants $\rho^-_{4\tau}(i)$ and $\bar{\rho}^{+(i)}_{4\tau}$ are close to each other or their ratio (the restricted condition number) is close to $1$. 
			
			\subsubsection*{Sparse linear regression}
			When the local objective function is the square loss function associated with the local data set at the client, the $\mathcal{A}$-RSC and $\mathcal{A}$-RSS constants essentially reduce to the restricted isometry property ($\mathcal{A}$-RIP) \cite{davenport2013signal, baraniuk2018one}. Indeed, let the square loss function be given by $h(x) = {1 \over 2l}  \|Bx - y\|_2^2$ where the rows of matrix $B \in \R^{l \times m}$ are the input data vectors denoted by $b_i$ and $y$ is the observation vector. Assume that $\|b_i\|_2 = 1$ for all $1 \le i \le l$, which can be done by normalizing the data vector $b_i$ and corresponding $y_i$. Since the function $h$ is 
			the square loss function, by looking into its Hessian, we study the restricted strong convexity (RSC) and smoothness (RSS) properties.  The Hessian $\nabla^2 h$ of $h$ is given by
			\[
			{1 \over l} B^T B.
			\] The RSC and RSS constants are the largest $c \ge 0$ and the smallest $d \ge 0$ such that $c \|w-z\|_2^2 \le (w-z)^T \left({1 \over l} B^T B \right) (w-z) \le  d\|w-z\|_2^2$ for all vectors $w$ and $z$ such that $\|w-z\|_0 \le \tau$. 
			
			This observation and the definition of RIP \cite{foucart2013invitation} imply that if the RIP constant of ${1 \over \sqrt{l}}B$ is at least $\delta$ then $(1-\delta) \|z\|_2^2 \le z^T \left({1 \over l} B^T B \right)z \le   (1+\delta)\|z\|_2^2$ for all $\tau$-sparse vectors $z$, making it satisfy the RSC/RSS with constants $1-\delta$ and $1+\delta$ respectively. 
			
			It could be possible, however, that the data matrix $B$ whose rows consist of the local data at each client may not satisfy RSC with respect to the standard basis, but RSC with respect to a certain dictionary $\mathcal{A}$. Put it differently, if $h$ is not restricted strong convex for $\tau$-sparse vectors, then $B(w-z) = 0$ for some vectors $w$ and $z$ such that $\|w-z\|_0 \le \tau$ or $Bu = 0$ for some $\tau$-sparse vector $u$, i.e., $B$ is not $\tau$-RIP. 
			
			We present our idea of simply using a random Gaussian dictionary $A$ to improve the ratio RSS to RSC constants of the associated new loss function ${1 \over 2l}  \|BAx - y\|_2^2$ (or improve the $\mathcal{A}$-RIP constant of $B$ with respect to a dictionary $A$) with high probability. 
			
			Our idea to improve the RIP with a random dictionary is based on a recent development in high dimensional geometry. More specifically, we use the following theorem from \cite{jeong2021sub}.
			
			\begin{thm} [Theorem 1.1 in \cite{jeong2021sub}]
				\label{theorem_mainBA}
				Let $B\in\R^{l \times m}$ be a fixed matrix, let $A\in\R^{m\times n}$ be a mean zero, isotropic and sub-Gaussian matrix with sub-Gaussian parameter $K$ and let $T\subset \R^n$ be a bounded set. Then
				$$ \E\sup_{x\in T} \Big|\|BAx\|_2-\|B\|_F\|x\|_2\Big|\leq CK\sqrt{\log K}\,\|B\|\left[w(T)+\mathrm{rad}(T)\right], $$
				and with probability at least $1-3e^{-u^2}$,
				$$ \sup_{x\in T} \Big|\|BAx\|_2-\|B\|_F\|x\|_2\Big|\leq CK\sqrt{\log K}\,\|B\| \left[w(T)+u\cdot\mathrm{rad}(T)\right]. $$
				Here $w(T)$ is the Gaussian width for the set $T$, $\mathrm{rad}(T) = \sup\limits_{y \in T} \|y\|_2$, and $C$ is an absolute constant.
			\end{thm}
			
			The following is an immediate consequence of the above theorem and the well-known fact that $w(T) \le C  r \sqrt{\tau \log (n/\tau)}$ for the set $T$ of all $\tau$-sparse vectors $x$ with $\|x\| \le r$ for some universal constant $C > 0$.
			
			\begin{cor}
				\label{cor:RIP_with_random_dictionary}
				Let $r > 0$ and $\mathbb{B}$ be the closed unit ball in $\R^n$. 
				For the set $T$ of all $\tau$-sparse vectors in $r\mathbb{B}$ and Gaussian random matrix $A$, we have
				$$ \E\sup_{x\in T} \left|\|BAx\|_2-\|B\|_F\|x\|_2\right|\leq C\|B\|\left[ r\sqrt{\tau \log (n/\tau)} + r \right], $$
				and with probability at least $1-3e^{-u^2}$,
				$$ \sup_{x\in T} \left|\|BAx\|_2-\|B\|_F\|x\|_2\right|\leq C\|B\| \left[ r\sqrt{\tau \log (n/\tau)} + ru \right]. $$
			\end{cor}
			
			Since both terms in the bounds in Corollary \ref{cor:RIP_with_random_dictionary} are homogeneous in $r$ for all $x$ in $T$ with $\|x\|_2 = r$, 
			the corollary implies that matrix $ \frac{1}{\|B\|_F} BA$ satisfies the RIP with a constant $\delta_\tau = { C_1{\|B\|^2 \over \|B\|_F^2} \tau \log (n/\tau)}$ with high probability, whenever the stable rank of $B$
			$$ 
			\mathrm{sr}(B):=\frac{\|B\|_F^2}{\|B\|^2} \geq C_2 \tau \log (n/\tau)
			$$ for a sufficiently large constant $C_2 > 0$.
			
			The above corollary can be readily applied to the data matrix $B_{D_i}$ in the local objective function  $f_i = {1 \over 2|D_i|}  \|B_{D_i} Ax - y\|_2^2$ for client $i$. First, recall that $\|b_i\|_2 = 1$ and by the definition of the Frobenious norm, $\|B_{D_i}\|_F =  \sqrt{|D_i|}$. The data matrix $B_{D_i}$ may not satisfy the RIP in general but $\frac{1}{\|B_{D_i}\|_F} B_{D_i}A = {1 \over \sqrt{|D_i|}} B_{D_i}A$ does with RIP constant 
			\[
			\delta_{\tau} = { C{\|B_{D_i}\|^2 \over \|B_{D_i}\|_F^2} \tau \log (n/\tau)} =  { C{\|B_{D_i}\|^2 \over |D_i|} \tau \log (n/\tau)}.
			\]
			Thus, with high probability, ${1 \over 2|D_i|}  \|B_{D_i} Ax - y\|_2^2$ is $\mathcal{A}$-RSC and $\mathcal{A}$-RSS with the constant ratio ${1 + \delta_{\tau}  \over 1 - \delta_{\tau}}$ with respect to the Gaussian random dictionary $A$, under the stable rank condition for $B_{D_i}$ (which could be a mild condition for many data matrices). Since  ${1 + \delta_{\tau}  \over 1 - \delta_{\tau}}$ is close to $1$ whenever the RIP constant $\delta_{\tau}$ is close to $0$, this makes $ \beta_1(i)$ and $ \beta_2(i)$ small, improving the convergence rate in Theorems \ref{thm:main_theorem1} \ref{thm:main_theorem1_inexact_solver}, and \ref{thm:main_theorem2} as we discussed before. Furthermore, note that since the Gaussian random matrix is statistically independent of the client data sets, there is no privacy leakage. 
			
			
			\subsubsection*{Sparse binary logistic regression}
			The previous analysis for the square loss can be extended to the logistic losses. First, we consider the binary logistic loss function $h(x) = {1 \over l} \sum_{i=1}^l \log(1 + \exp(-2y_j b_j^T x))$ with input data vector $b_j$ and labels $y_j \in \{-1,  1\}$. Assume that $\|b_i\|_2 = 1$ for all $1 \le i \le l$ and $x$ is a $\tau$-sparse  vector with $\|x\| \le r$. Since the function $h$ is twice-differentiable, we can study the 
			RSC and RSS by investigating its Hessian.
			We denote the sigmoid function by $s(z) = {1 \over 1+\exp(z)}$. By a direct computation or from the lecture notes \url{https://www.cs.mcgill.ca/~dprecup/courses/ML/Lectures/ml-lecture05.pdf}, one can verify that Hessian $\nabla^2 h$ of the logistic loss function $h$ is given by
			\[
			\nabla^2 h (x) =  {1 \over l} B^T \Lambda(x) B,
			\] 
			where $B$ is the matrix whose rows $b_i$ consist of a client data set and $\Lambda(x)$ is the diagonal matrix whose $j$-th diagonal entry is $4s(2b_j^T x) (1 - s(2b_j^T x))$. 
			
			First, it is easy to check that $h$ is $L$-smooth \cite{shalev2014understanding} with 
			\[
			L \le {1 \over l} \sum_{i=1}^l \max_{x}  4s(2b_j^T x) (1 - s(2b_j^T x)) \le 1.
			\] Next, since $x$ is a $\tau$-sparse vector with $\|x\|_2 \le r$, from the definition of the sigmoid function $\tau$, we deduce  $[\Lambda(x)]_{jj} \ge {4 \over (1+\exp(r))^2}$. Then, we have
			\[
			{4 \over (1+\exp(r))^2} \cdot {B^T  B \over l} \preceq \nabla^2 h = {1 \over l} B^T \Lambda(x) B \preceq   {B^T  B \over l}.
			\]
			Note that the above bound does not  imply that $h$ is RSC since it is possible that $Bx = 0$ for some $\tau$-sparse vector $x$. However, if we use a random Gaussian dictionary $A$, then a similar derivation gives 
			\[
			{4 \over (1+\exp(r))^2} \cdot {A^TB^T  BA \over l} \preceq \nabla^2 h = {1 \over l} A^TB^T \Lambda(x) BA \preceq  {A^TB^T  BA \over l}.
			\]
			Collorary \ref{cor:RIP_with_random_dictionary} implies that
			\[
			\|B\|_F \|x\| - C\|B\| \|x\|\left[ \sqrt{\tau \log (n/\tau)} + u \right] \le \sqrt{x^TA^TB^T  BAx} \le \|B\|_F \|x\| + C\|B\| \|x\|\left[ \sqrt{\tau \log (n/\tau)} + u \right]
			\] for all $\tau$-sparse vectors  with probability at least $1 - 3e^{-u^2}$. Finally, it is easy to check that applying this bound to the previous bound on $\nabla^2 h$ yields that with high probability, $h$ is $\mathcal{A}$-RSS and $\mathcal{A}$-RSC with the constant ratio 
			\[
			{(1+\exp(r))^2 \over 4} \cdot \left( \|B\|_F  + C\|B\| \left[ \sqrt{\tau \log (n/\tau)} + u \right] \over \|B\|_F- C\|B\|\left[ \sqrt{\tau \log (n/\tau)} + u \right] \right)^2,
			\] which is close to ${(1+\exp(r))^2 \over 4}$
			as long as the stable rank 
			$$ 
			\mathrm{sr}(B) =\frac{\|B\|_F^2}{\|B\|^2} \gg \tau \log (n/\tau).
			$$ 
			This implies that for any $\tau$-sparse vector $x$ with $\|x\|_2 \le r$, the logistic loss function is $\mathcal{A}$-RSC/$\mathcal{A}$-RSS with respect to a random Gaussian dictionary with constant $\approx {(1+\exp(r))^2 \over 4}$ under a mild condition, even if the function is not RSC/RSS in the standard basis (for example, the ratio is infinite if the RSC constant in the standard basis is $0$). We apply the above argument to each binary logistic loss function $f_i$. Note that since the RSC/RSS ratio can be understood as a restricted condition number that controls the convergence rates by Theorems \ref{thm:main_theorem1}, \ref{thm:main_theorem1_inexact_solver}, and \ref{thm:main_theorem2}, a random Gaussian dictionary is appropriate for FedGradMP with an $\ell_2$-norm constraint that is discussed in Section \ref{sec:constrained_FedGradMP}.
			
			\subsubsection*{Sparse multiclass logistic regression}
			We only highlight the difference between the multiclass and binary logistic regression cases since the arguments are very similar to each other. 
			Consider the multinomial logistic regression function with $K$ classes. The label $y_{ij}$ is $1$ if the $j$-th training input belongs to the class $i$ and $0$ otherwise, $b_j$ are normalized data vectors (i.e., $\|b_i\|_2 = 1$), and $x^{(i)}$ are $\tau$-sparse classifier vectors with $\|x^{(i)}\|_2 \le r$.
			
			The corresponding loss function is given as
			\[
			h(x^{(1)},x^{(2)},\dots,x^{(K)}) = \sum_{j=1}^l \left[\sum_{i=1}^K  -y_{ij}  b_j^Tx^{(i)} + \ln \left(\exp \left( \sum_{i=1}^K   b_j ^Tx^{(i)} \right)\right) \right].
			\]
			Similar to the binary logistic regression case, the direct computation of the Hessian of $h$ gives
			\[
			\nabla^2_{x^{(i)}} h = {1 \over l} B^T \Lambda(x^{(i)}) B.
			\] Here $\Lambda(x)$ is a diagonal matrix whose diagonal entries are defined as $[\Lambda(x)]_{jj} =  s(b_j^T x) (1 - s(b_j^T x))$, where
			\[
			s(b_j^T x) = {\exp(b_j^T x) \over 1+ \sum_{i=1}^K \exp(b_i^T x)}.
			\]
			By the same argument used for the sparse binary logistic regression, $h$ is $\mathcal{A}$-RSS  and $\mathcal{A}$-RSC with a constant ratio 
			\[{ (1+K\exp(2 r))^2 } \cdot \left(\|B\|_F  + C\|B\| \left[ \sqrt{\tau \log (n/\tau)} + u \right] \over \|B\|_F- C\|B\|\left[ \sqrt{\tau \log (n/\tau)} + u \right]  \right)^2
			.\]
			This again indicates that for any $\tau$-sparse vector $x$ with $\|x\| \le r$, the multiclass logistic loss function is $\mathcal{A}$-RSC/$\mathcal{A}$-RSS with respect to a random Gaussian dictionary even if it may not be RSC/RSS in the standard basis. As we saw in the binary logistic regression, this shows that it is beneficial to use a random Gaussian dictionary in logistic regression for $\ell_2$-norm constrained FedGradMP, which is also verified in our numerical experiments in Section \ref{sec:Improving_RSC_for_FEMNIST}.

			\begin{rem} [Random dictionary]
				The idea of using a Gaussian random dictionary to improve the restricted condition number should be distinguished from the sketching in the FL literature \cite{haddadpour2020fedsketch, rothchild2020fetchsgd, song2022sketching}. Our formulation and analysis are fundamentally different from those for sketching schemes that focus on compressing the gradient to save communication cost between a server and clients. In these work \cite{haddadpour2020fedsketch, rothchild2020fetchsgd, song2022sketching}, the sketching mappings (commonly random matrices) developed for numerical linear algebra \cite{woodruff2014sketching} are applied after the clients computed the gradients to compress the information, whereas our Gaussian random mappings are used to transform the domain of the solution space to improve the restricted condition number. 
			\end{rem}
			
			\begin{rem} [Sharing the dictionary among clients]
				The server either broadcasts the dictionary to clients or the shared memory can be used to share the dictionary among clients as suggested in \cite{gu2021fast,mo2021ppfl}. When the latter option is available, the server does not need to send the dictionary to the clients. 
			\end{rem}
			
			\section{Numerical Experiments}
			\label{sec:exp}
			
			In this section, we provide numerical experiments validating our theory and showing the effectiveness of the proposed algorithm. 
			
			
			\subsection{FedGradMP for sparse linear regression}
			\label{section:FedGradMP on synthetic Heterogeneous data sets}
			
			\subsubsection{Synthetic dataset}
			\subsubsection*{Experiment settings}
			The first numerical experiment uses synthetic data sets. We run FedGradMP (Algorithm \ref{alg:FedGradMP}) with the square loss function. More precisely, we consider the component function of the form $f_i = {1 \over 2\|D_i|} \|A_{D_i}x - y_{D_i}\|_2^2$ where  $A_{D_i}$ is the client $i$ data matrix in $\R^{100 \times 1000}$ whose elements are synthetically generated according to the normal distribution $\mathcal{N}(\mu_i, 1/i^{1.1})$ with the mean value $\mu_i$ that is randomly generated from the mean-zero Gaussian with variance  $\alpha$. 
			Here, $y_{D_i}$ are observations with $y_{D_i} = A_{D_i} x^\#$ and $x^\# \in \R^{1000}$ is a randomly generated vector that is $10$-sparse with respect to the standard basis whose $10$ nonzero components are drawn from the unit sphere $\mathbb{S} \subset \R^{10}$. Since the random mean $\mu_i$ obeys the normal distribution $\mathcal{N}(\alpha, 0)$, the parameter $\alpha$ modulates the degree of client data heterogeneity: as $\alpha$ increases, the more likely $\mu_i$ vary wildly which in turn makes the client dataset distributions more different. This type of model is commonly used in FL numerical experiments to generate synthetic datasets \cite{wang2021field, tong2020federated, yuan2021federated} since randomly generated mean $\mu_i$ and decreasing  variance $1/i^{1.1}$ make the client data set heterogeneous. 
			
			
			
			The number of clients is $50$, the number of data points of each client is $100$, and the mini-batch size of each client for FedGradMP is $40$.
			
			\subsubsection*{Simulation results}
			Figure \ref{fig:FedGradMP_noniid_rounds} shows that FedGradMP converges linearly for various heterogeneity level $\alpha$, validating Theorem \ref{thm:main_theorem1}. Note that the higher $\alpha$ is, the larger the variance of random mean shift $\mu_i$ or the higher the degree of heterogeneity is. 
			The curves on the left panel are the relative error of FedGradMP for the noiseless case and the curves on the right are for the Gaussian noise case.  
			We observe that FedGradMP still converges for highly heterogeneous data sets but with slower convergence rates in both cases.
			\begin{figure}
				\centering
				\begin{minipage}{0.45\textwidth}
					\centering 
					\includegraphics[ width=1.0 \textwidth]{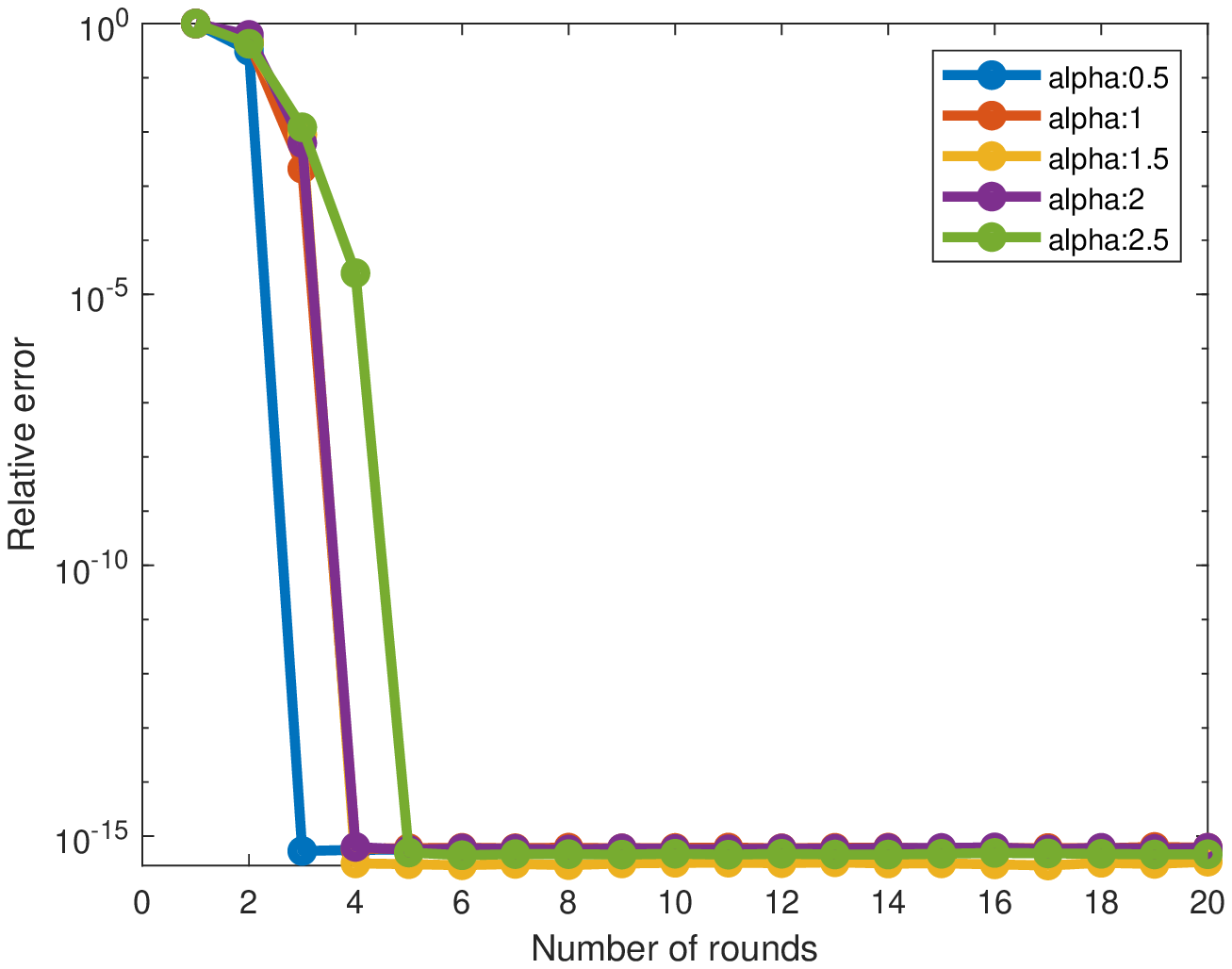}
				\end{minipage}
				\begin{minipage}{0.45\textwidth}
					\centering 
					\includegraphics[ width=1.0 \textwidth]{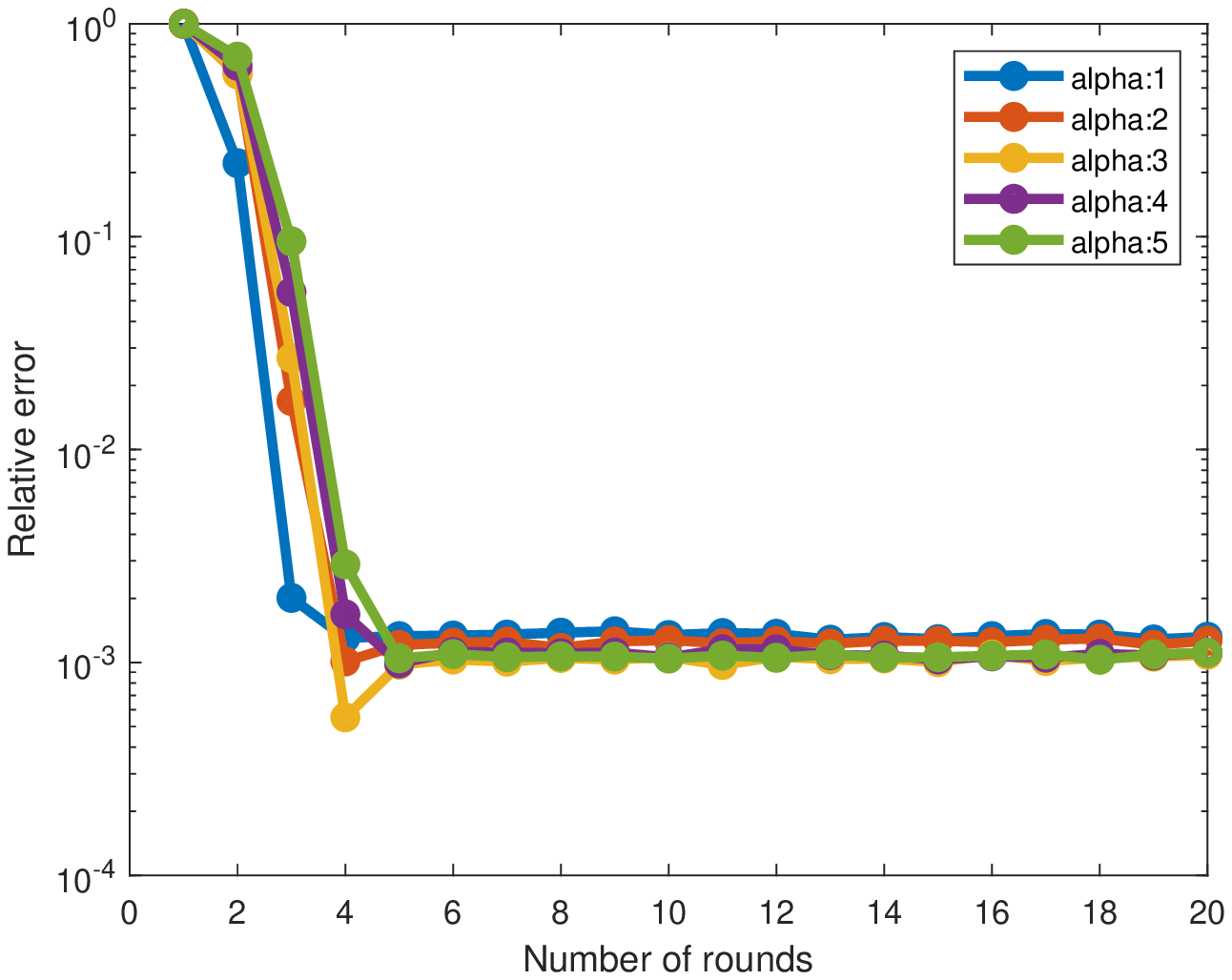}
				\end{minipage}
				\caption[] {Linear convergence of FedGradMP with for data sets with various heterogeneity levels. 
				}
				\label{fig:FedGradMP_noniid_rounds}
			\end{figure}
			
			\subsubsection{Real data set: sparse video recovery}
			
			In this experiment, we test FedGradMP on video frame recovery from a real-world dataset. Our dataset is a xylophone video consisting of $120$ frames from YouTube   \url{https://www.youtube.com/watch?v=ORipY6OXltY}, which can be also downloaded from the MathWorks website \url{https://www.mathworks.com/help/matlab/ref/videoreader.html}. Each frame is of size $240  \times 320$ after the conversion to gray-scale frames. 
			We reshape the $82$-th frame as a vector in $\mathbb{R}^{76800}$ and our goal is to recover this frame. 
			
			For this experiment, we use the K-SVD algorithm \cite{aharon2006k} to generate a dictionary $\Psi \in \mathbb{R}^{76800 \times 50}$ consisting of $50$ atoms that are trained over the first $80$ frames. 
			
			The number of clients to reconstruct this video  frame is $50$ and  non i.i.d. random matrix of size $30 \times 76800$ is used for each client. More specifically, it is generated according to the normal distribution $\mathcal{N}(\mu_i, 1/i^{0.9})$ where $\mu_i \sim \mathcal{N}(0,\alpha = 0.5)$, similar to the one in Sections \ref{section:FedGradMP on synthetic Heterogeneous data sets} and \ref{section:Comparison of FedGradMP}.
			
			Figure \ref{fig:FedGradMP_xylophone} shows one frame of the input image sequence on the left, the image recovered by FedGradMP + K-SVD in the middle, and the difference on the right.
			Considering that the sensing matrices for clients are highly heterogeneous, the recovered image quality is reasonably satisfactory. 
			
			\begin{figure}
				\centering
				\setlength{\tabcolsep}{-20pt}
				\begin{tabular}{ccc}	\includegraphics[width=0.4\textwidth]{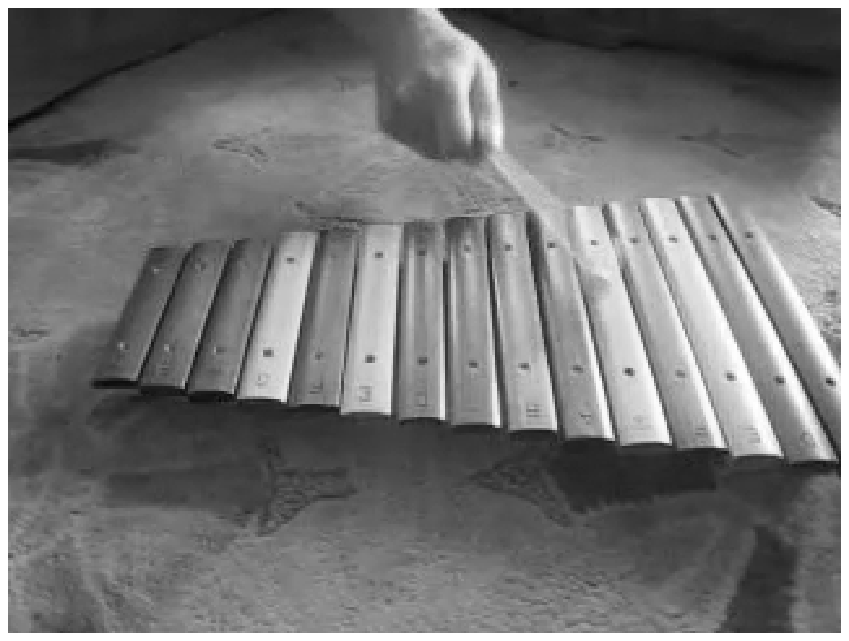}
					&\includegraphics[width=0.4\textwidth]{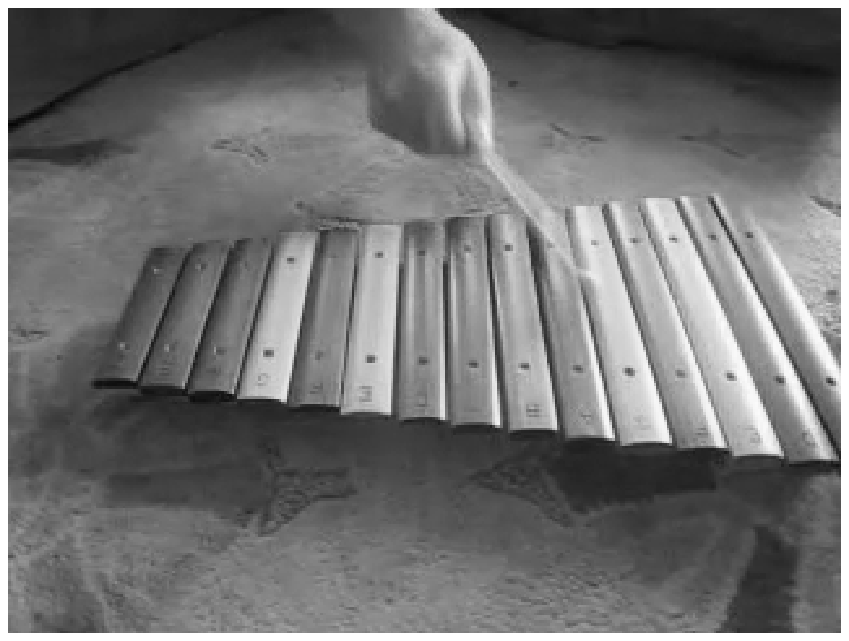}
					&\includegraphics[width=0.4\textwidth]{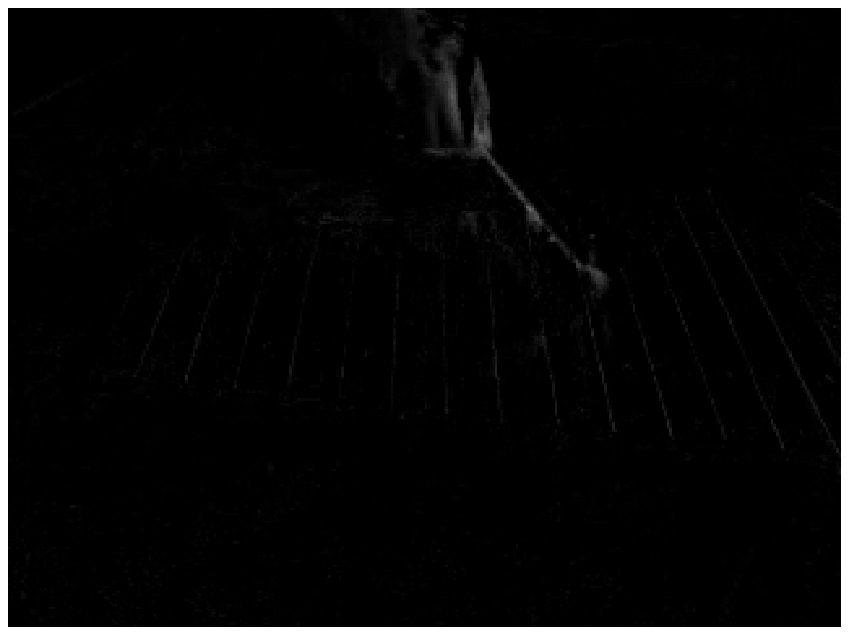}
				\end{tabular}
				\vspace{-20pt}
				\caption[]{Input image on the left:  82-th frame of the xylophone video. The output image of FedGradMP with K-SVD dictionary in the middle. The difference of the two images is displayed on the right. }
			\label{fig:FedGradMP_xylophone}
		\end{figure}

		\subsection{Comparison of FedGradMP with other FL algorithms}
		
		\subsubsection{Federeated sparse linear regression}
		\label{section:Comparison of FedGradMP}
		The next experiments illustrate FedGradMP outperforms other FL algorithms in both low and highly heterogeneous data environments. 
		
		\subsubsection*{Experiment settings}
		We compare FedGradMP with FedAvg, FedIterHT, FedMid, and FedDualAvg for the sparse linear regression or compressed sensing. The $\ell_1$ regularization hyperparameter for FedMid and FedDualAvg is $0.5^5$. The client learning rates for FedAvg, FedIterHT, FedMid, and FedDualAvg are chosen by grid search with grid $\{0.01, 0.005, 0.001, 0.0005, 0.0001, 0.00005, 0.00001,0.000005\}$ to achieve their best performance. The number of clients is $50$ and the mini-batch size of each client is $50$. The loss function for client $i$ is given by $f_i = {1 \over 2|D_i|} \|A_{D_i}x - y_{D_i}\|_2^2$, where $y_{D_i} = A_{D_i}x^{\#} + e$ are noisy measurements of a $15$-sparse vector $x^{\#}$ and $e$ is a noise vector whose components are generated according to $\mathcal{N}(0,0.005)$. 
		
		In the low-heterogeneity data experiments for Figure \ref{fig:low_heterogenity}, the $100 \times 1000$ data matrices $A_{D_i}$ are generated by the randomly shifted mean Gaussian model used for the experiments for Figure \ref{fig:FedGradMP_noniid_rounds} with whose elements are synthetically generated according to $\mathcal{N}(\mu_i, 1/i^{0.2})$ where $\mu_i \sim \mathcal{N}(0,\alpha = 0.2)$.
		
		On the other hand, under the same setting as before but a higher value of the parameter $\alpha = 0.5$ is used to generate the data matrices  $A_{D_i}$ to obtain a more heterogeneous client data set for the experiment for Figure \ref{fig:Comparison_noniid}. 
		
		The previous two experiments are conducted for relatively low-sparsity level signals. The relative error curves in Figure \ref{fig:Comparison_noniid_high_sparsity} are obtained for a signal $x^{\#}$ that $400$-sparse under the same heterogeneous model as in Figure \ref{fig:Comparison_noniid}. Because of the high sparsity level (about the same order as the ambient dimension $1000$), we run the Inexact-FedGradMP (Algorithm \ref{alg:InexactFedGradMP}) with gradient descent to solve the sub-optimization problem more efficiently as we have discussed in Section \ref{sec:ext}.
		
		\subsubsection*{Simulation results}
		The plots for Figure \ref{fig:low_heterogenity} demonstrate FedGradMP converges faster than other methods in the number of communication rounds for a low heterogeneous environment both in the number of rounds and wall-clock time. FedIterHT converges linearly as shown in  \cite{tong2020federated}, but with a slower convergence rate than FedGradMP. FedMid and FedDualAvg also appear to converge as their theory suggest \cite{li2018federated, wang2021field, yuan2021federated} 
		but slower than FedGradMP. FedAvg is the slowest among all algorithms we tested and it generally does not produce
		a sparse solution. We also notice that FedGradMP offers the smallest residual error evidencing our theory that FedGradMP guarantees the optimal statistical bias in Remark \ref{rem:main_remark2}.

		In the highly heterogenous environment setting, FedGradMP still performs well whereas other algorithms start degrading significantly, as we observe in the plots in Figure \ref{fig:Comparison_noniid}. 
		
		As for the signals with higher sparsity levels, from the plots in Figure \ref{fig:Comparison_noniid_high_sparsity} show, we see that FedGradMP performs better than other baseline algorithms in terms of both criteria. 
		
		\begin{figure}
			\centering
			\begin{minipage}{0.45\textwidth}
				\centering 
				\includegraphics[ width=1.0 \textwidth]{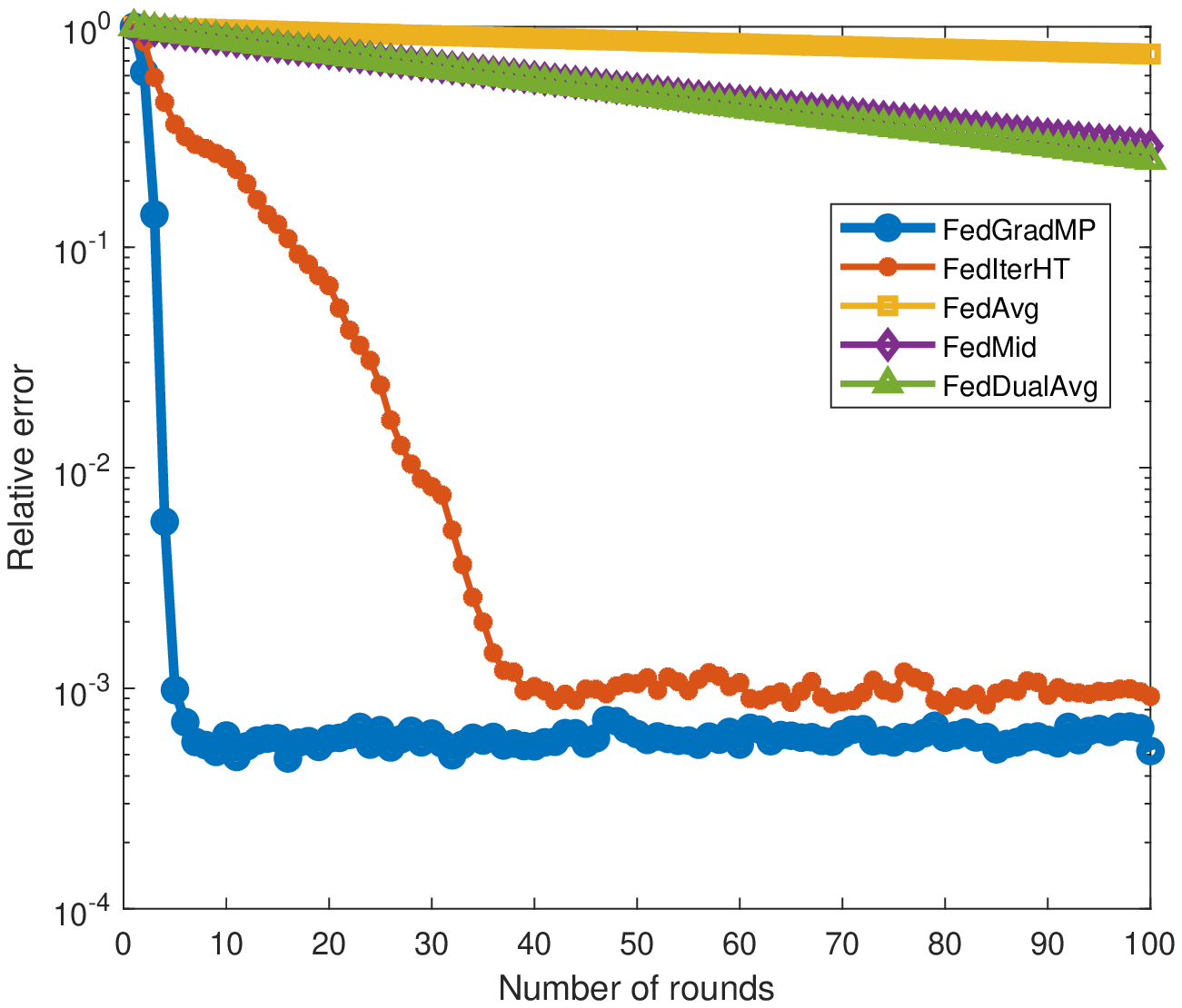}
			\end{minipage}
			\begin{minipage}{0.45\textwidth}
				\centering 
				\includegraphics[ width=1.0 \textwidth]{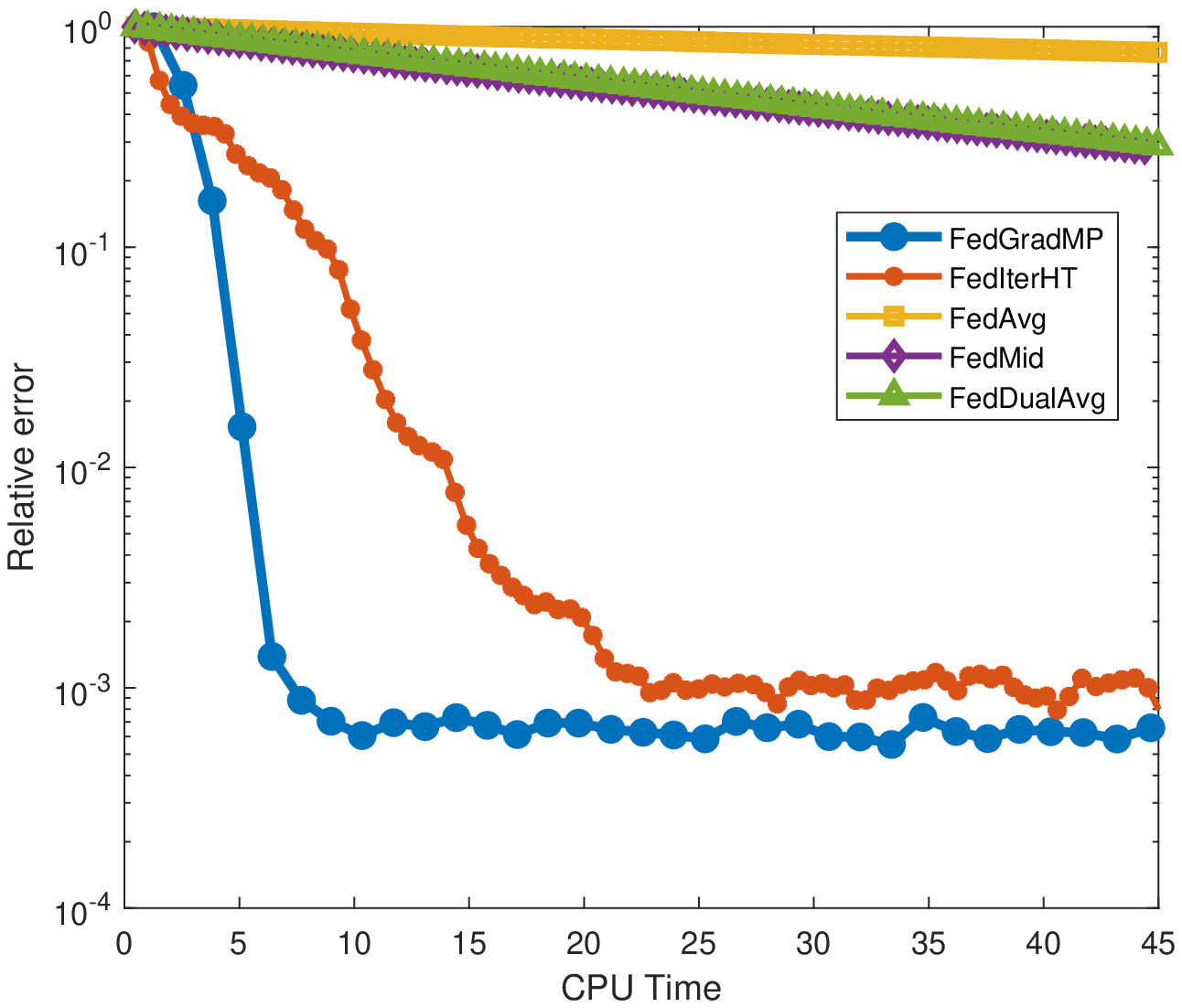}
			\end{minipage}
			\caption[] {FedGradMP outperforms other methods in a low data heterogeneous environment. 
			}
			\label{fig:low_heterogenity}
		\end{figure}

		\begin{figure}
			\centering
			\begin{minipage}{0.45\textwidth}
				\centering 
				\includegraphics[ width=1.0 \textwidth]{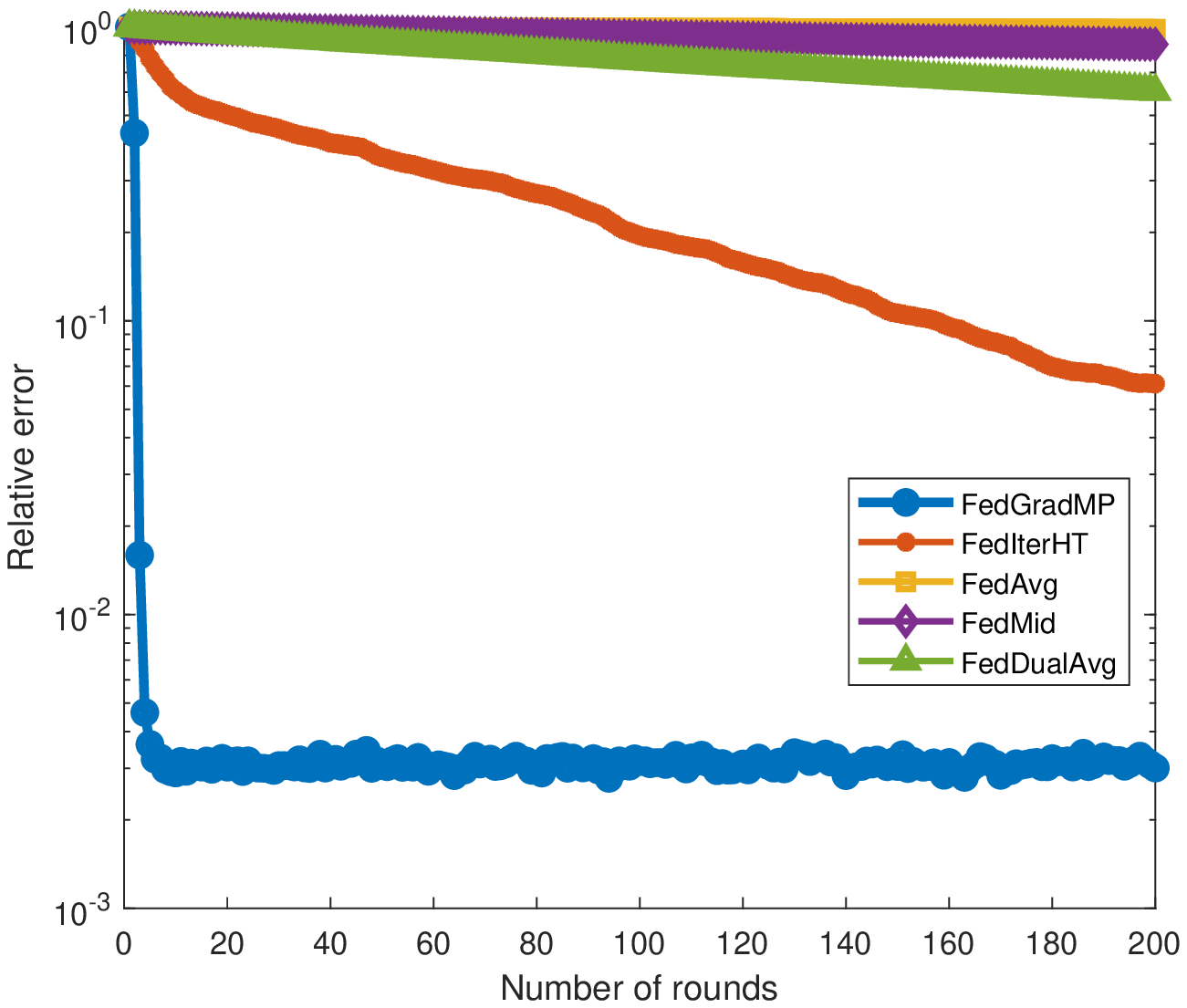}
			\end{minipage}
			\begin{minipage}{0.45\textwidth}
				\centering 
				\includegraphics[ width=1.0 \textwidth]{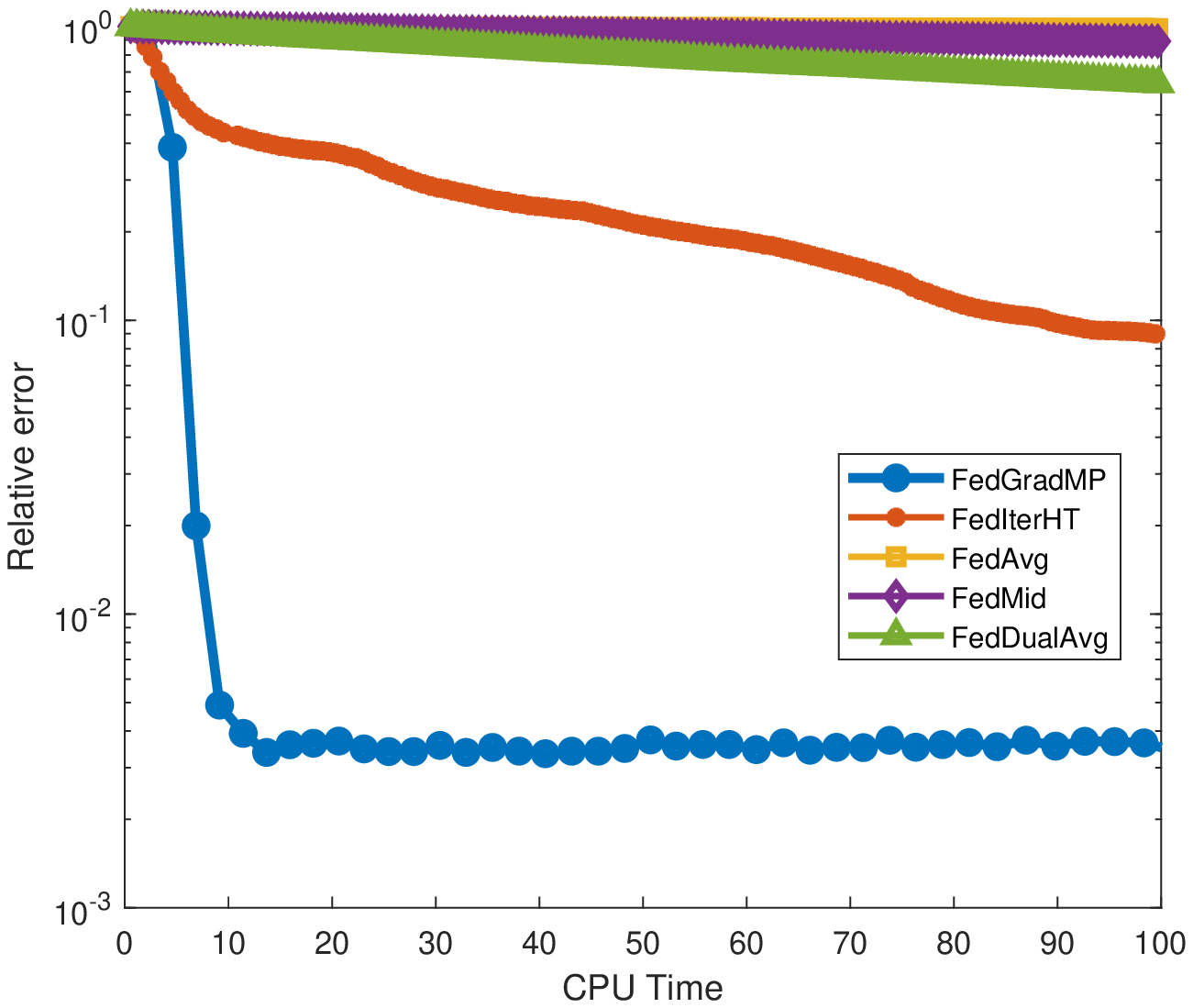}
			\end{minipage}
			\caption[] {FedGradMP outperforms other methods in a high data heterogeneous environment. 
			}
			\label{fig:Comparison_noniid}
		\end{figure}

		\begin{figure}
			\centering
			\begin{minipage}{0.45\textwidth}
				\centering 
				\includegraphics[ width=1.0 \textwidth]{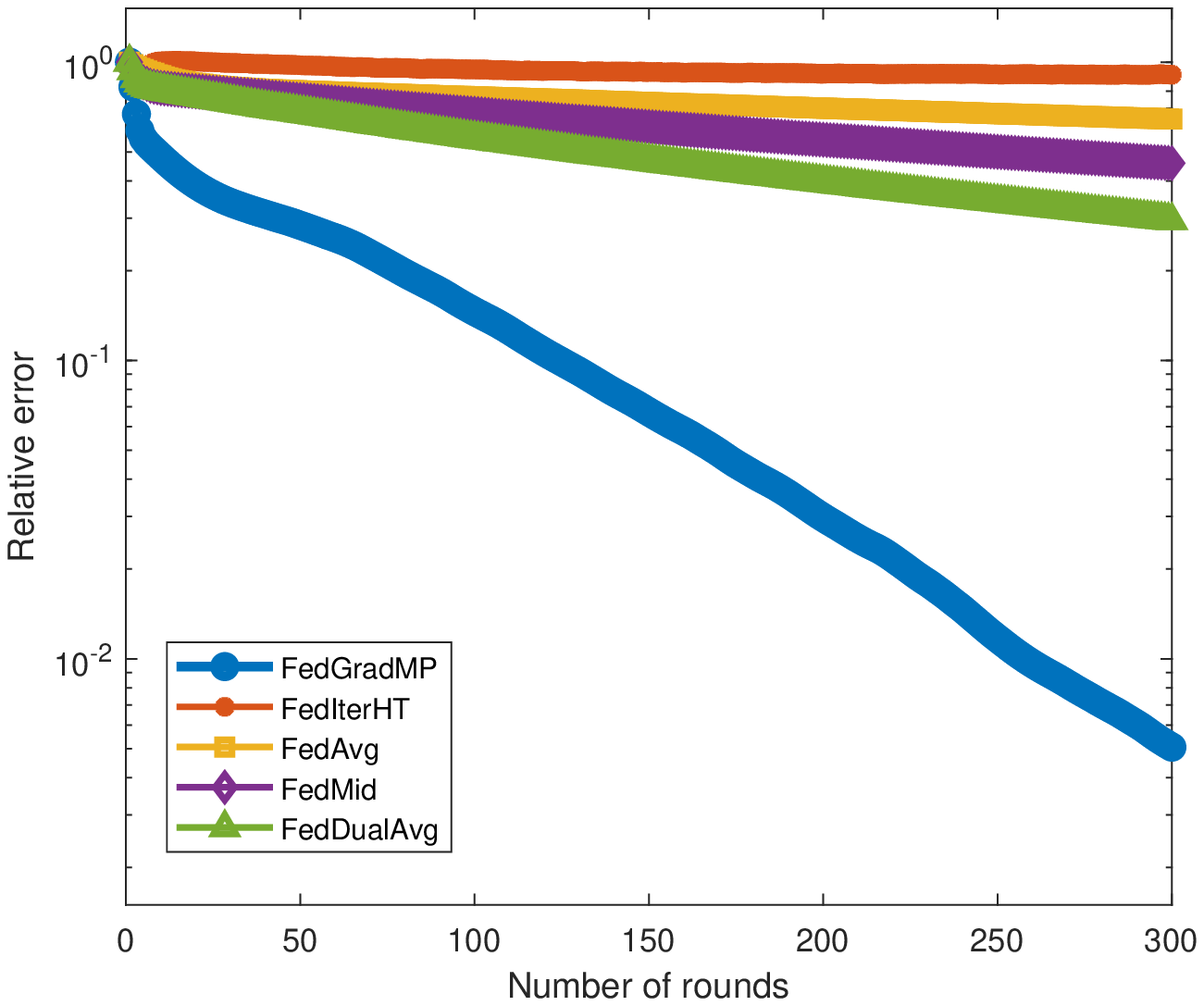}
			\end{minipage}
			\begin{minipage}{0.45\textwidth}
				\centering 
				\includegraphics[ width=1.0 \textwidth]{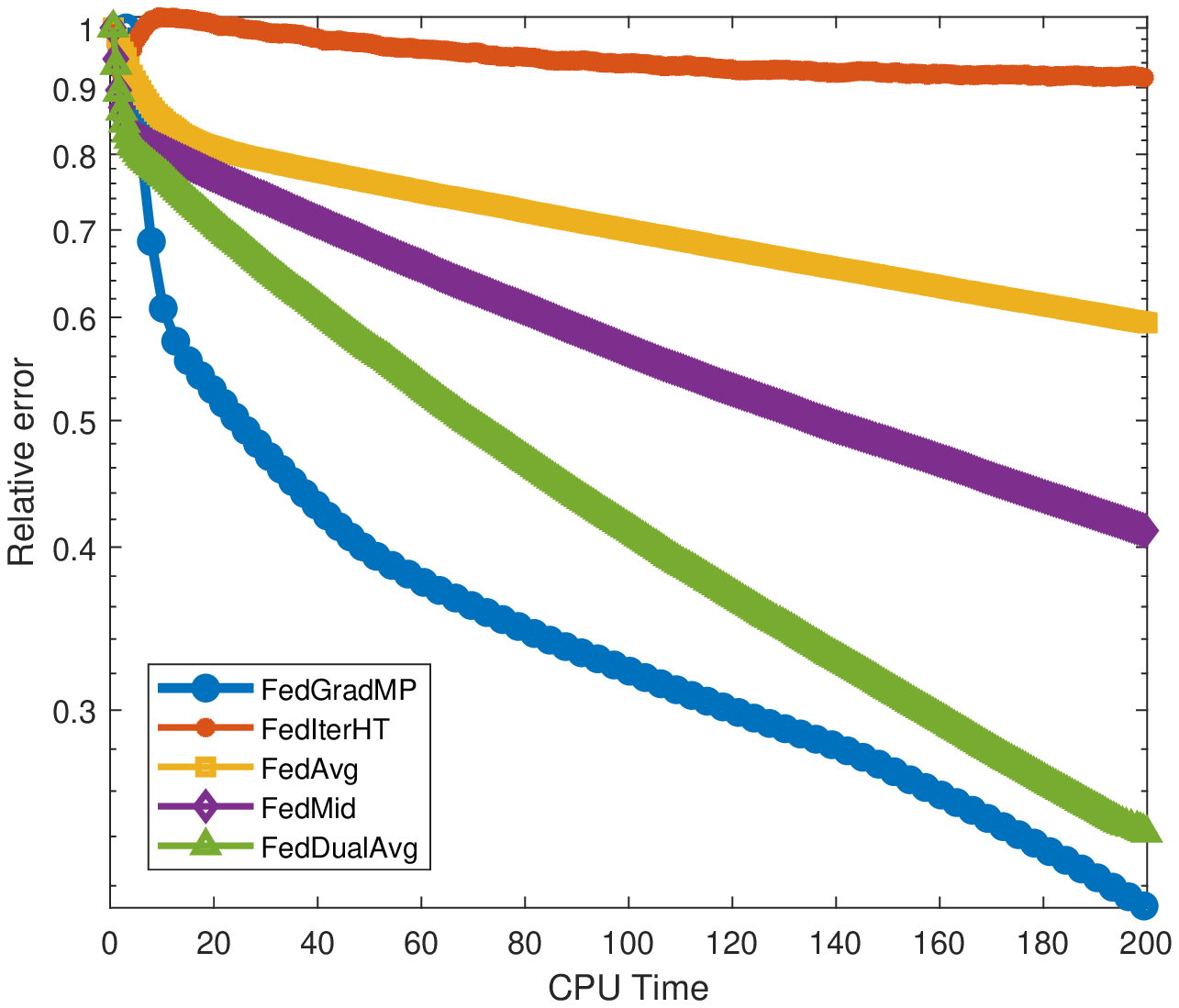}
			\end{minipage}
			\caption[] {FedGradMP outperforms other methods in a high data heterogeneous environment for high sparsity level signals.  
			}
			\label{fig:Comparison_noniid_high_sparsity}
		\end{figure}

		\subsubsection{Logistic regression for Federated EMNIST dataset}
		\label{sec:Logistic regression for FEMNIST}
		
		\subsubsection*{Experiment settings}
		The data set we use is the Federated EMNIST-10 dataset (FEMNIST-10), a commonly used dataset to test FL algorithms. FEMNIST-10 is a collection of handwritten digits and $10$ labels, grouped by writers. Each data point of FEMNIST-10 consists of a $28 \times 28$ gray-scale image and its label belongs to one of the $10$ classes. 
		Note that the dimension of solution space is $28 \times 28 = 784$. 
		
		In the experiment, we use $350$ clients, which is about $10\%$ of the original dataset with $100$ examples each. We split the data into a training dataset with $300$ clients and a test dataset with $50$ clients. The number of participating clients per round is $10$ and the mini-batch size is $50$.
		This is similar to the standard settings used for FL algorithm benchmark \cite{yuan2021federated, bao2022fast}. We run the Inexact-FedGradMP with an $\ell_2$ norm constraint with $20$ local iterations, in which we solve the sub-optimization problem in FedGradMP by SGD with $2$ iterations. The number of local iterations for FedIterHT, FedAvg, FedMid, FedDualAvg is $40$. Note that the total number of the effective number of local iterations for all the algorithms is the same, $40$ iterations. The number of communication rounds is $1000$.
		
		The local objective function $f_i (x^{(1)},x^{(2)},\dots,x^{(N)})  = {1 \over |D_i|} \sum\limits_{j=1}^{|D_i|} \left[\sum\limits_{i=1}^{10}  -y_{ij}  b_j^Tx^{(i)} + \ln \left(\exp \left( \sum\limits_{i=1}^{10}   b_j ^Tx^{(i)}  \right)\right)\right]$, the multiclass logistic regression function. Additionally, we use $\ell_1$ regularization with  hyperparameter $\lambda$ is chosen to be $0.0001$ for FedMid and FedDualAvg as in \cite{yuan2021federated, chen2020optimal} and the $\ell_2$ ball constraint $\|x\| \le 10^5$ for FedGradMP.
		
		\subsubsection*{Simulation results}
		Figure \ref{fig:Comparison_FedGradMP} demonstrates that FedGradMP outperforms the baseline algorithms in terms of prediction accuracy on training and test datasets. 
		
		\begin{figure}
			\centering
			\begin{minipage}{0.45\textwidth}
				\centering 
				\includegraphics[ width=1.0 \textwidth]{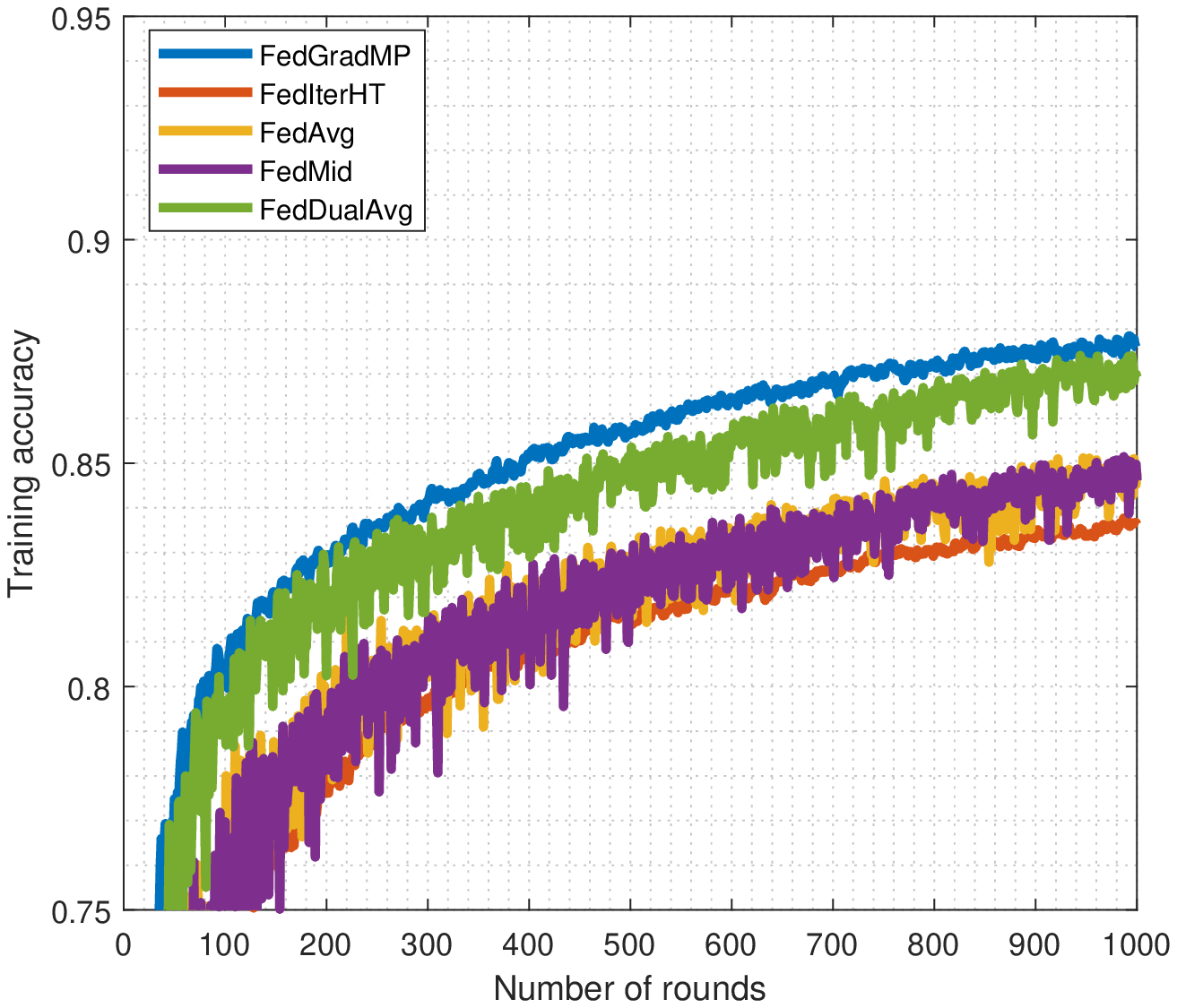}
			\end{minipage}
			\begin{minipage}{0.45\textwidth}
				\centering 
				\includegraphics[ width=1.0 \textwidth]{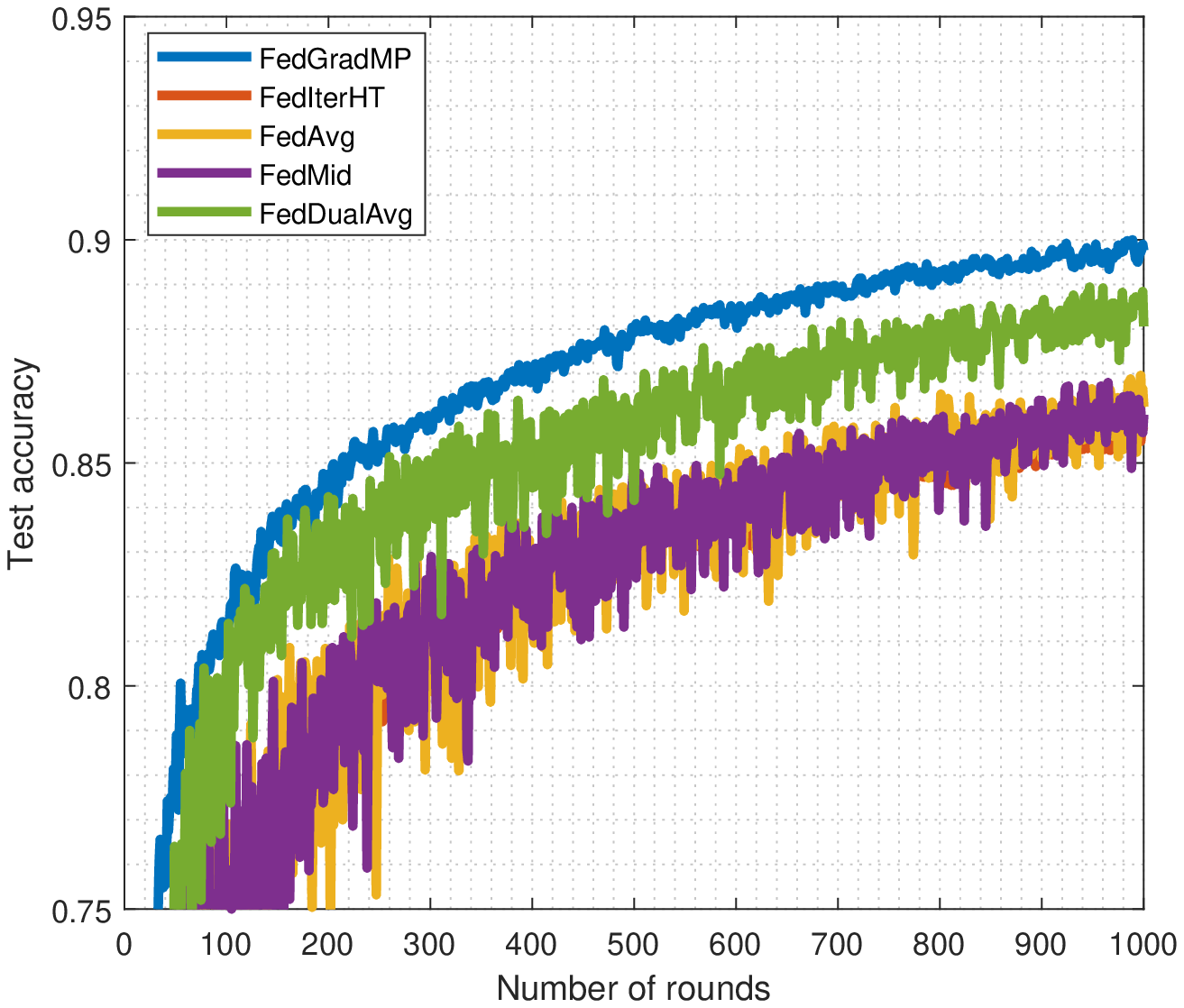}
			\end{minipage}
			\caption[] {In both of experiments using the traning dataset (on the left) and the test dataset (on the right), the performance of FedGradMP is better than other baseline methods. 
			}
			\label{fig:Comparison_FedGradMP}
		\end{figure}

		\subsubsection*{Improving FedGradMP performance using random dictionaries}

		\label{sec:Improving_RSC_for_FEMNIST}
		In this section, we show that FedGradMP combined with a random Gaussian dictionary empirically outperforms the one with the standard basis. The experiment settings are the same as the ones in Section \ref{sec:Logistic regression for FEMNIST} except we use the random Gaussian dictionary of size $200 \times 784$. As a comparison, we have also included the prediction accuracy curves of FedGradMP in Figure \ref{fig:Comparison_FedGradMP}. 
		
		The plot in Figure
		\ref{fig:FedGradMP_FEMINST} indicates that FedGradMP + random Gaussian dictionary outperforms FedGradMP + the standard basis, supporting our theory in Section \ref{sec:Improving_RSC_using_random_dictionary}. 
		
		\begin{figure}
			\centering
			\begin{minipage}{0.45\textwidth}
				\centering 
				\includegraphics[ width=1.0 \textwidth]{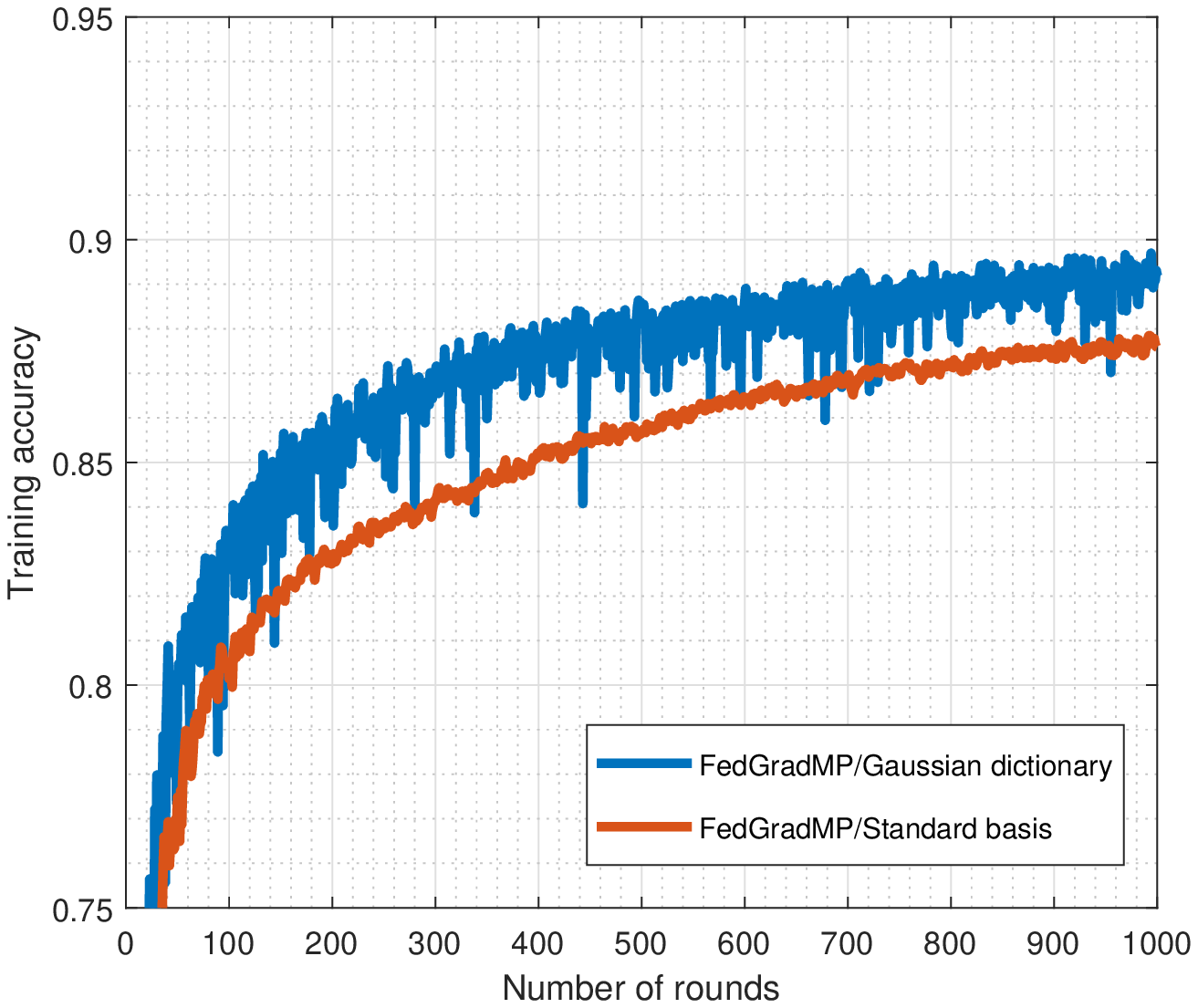}
			\end{minipage}
			\begin{minipage}{0.45\textwidth}
				\centering 
				\includegraphics[ width=1.0 \textwidth]{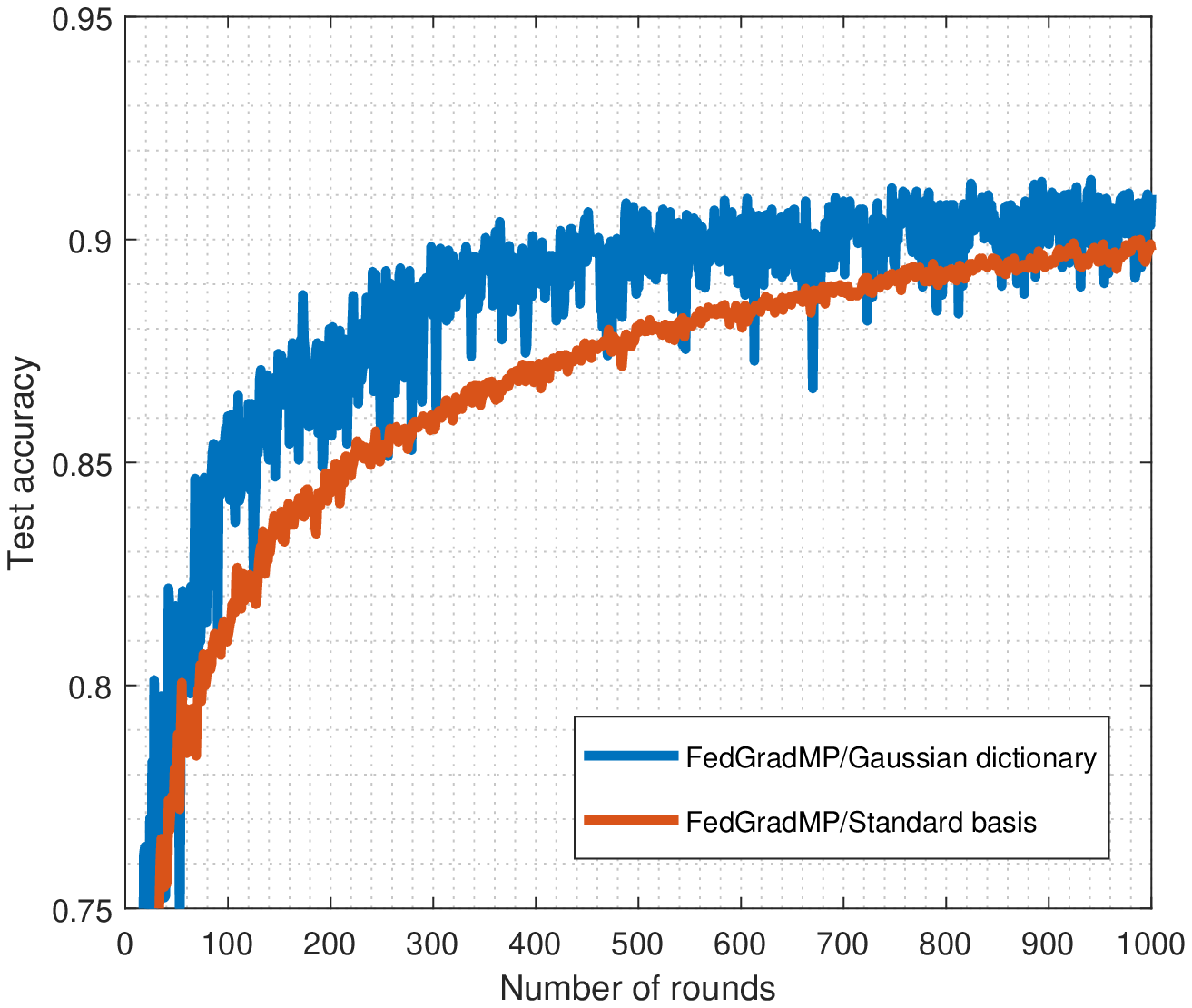}
			\end{minipage}
			\caption[] {Training accuracy curves of FedGradMP for the FEMINST dataset with respect to the random Gaussian dictionary and the standard basis. 
			}
			\label{fig:FedGradMP_FEMINST}
		\end{figure}

		\subsection{Difficulties of tuning learning rates for FL methods}
		\label{section:Challenges in tuning LR}
		As we saw in the numerical experiments, Section \ref{section:Comparison of FedGradMP}, other FL methods suffer especially in a highly heterogeneous environment. This can be  alleviated by tuning hyperparameters individually for each client such as learning rates, but it could be challenging or at least time-consuming. To showcase the difficulties of tuning the learning rates of FL methods, we study FedIterHT but we empirically observed the same phenomenon for other baseline algorithms. Another reason we tested FedIterHT is that it is actually the only method among baseline that aims to solve the sparsity-constrained problem \eqref{eq:main_problem} as ours.

		The convergence of FedIterHT in \cite{tong2020federated} strongly depends on the learning rates. Although they provide the learning rates that depend on the dissimilarity parameter and restricted strong convexity/smoothness parameters at the clients, they are quite often not available and difficult to estimate in practice since the data at clients are non i.i.d.. FedGradMP is free from this issue at least for sparse linear regression and is often still computationally efficient since clients only solve optimization problems over smaller spaces after the support estimation. 
		
		\subsubsection*{Experiment settings}
		We run FedIterHT for the squared loss function with a randomly generated $10$-sparse vector as ground truth. The local loss function $f_i = {1 \over 2\|D_i|} \|A_{D_i}x - y_{D_i}\|_2^2$ where  $A_{D_i}$ is the client $i$ data matrix in $\R^{100 \times 1000}$ whose elements are synthetically generated according to $\mathcal{N}(\mu_i, 1/i^{1.1})$ with randomly generated mean $\mu_i$ from the mean-zero Gaussian with variance  $\alpha = 1.0$. This setting is similar to the synthetic dataset in \cite{tong2020federated} except we have common sparse ground truth. The number of clients is $30$ with mini-batch size $40$. The number of total data points $m = 3000$ and the dimension of solution space $n = 1000$.
		The client learning rate combinations for the experiment are $\{0.0001, 0.0005, 0.001, 0.002, 0.004, 0.01, 0.02 \}$.
		
		\subsubsection*{Simulation results}

		\begin{figure}
			\centering
			\begin{minipage}{0.45\textwidth}
				\centering 
				\includegraphics[ width=1.0 \textwidth]{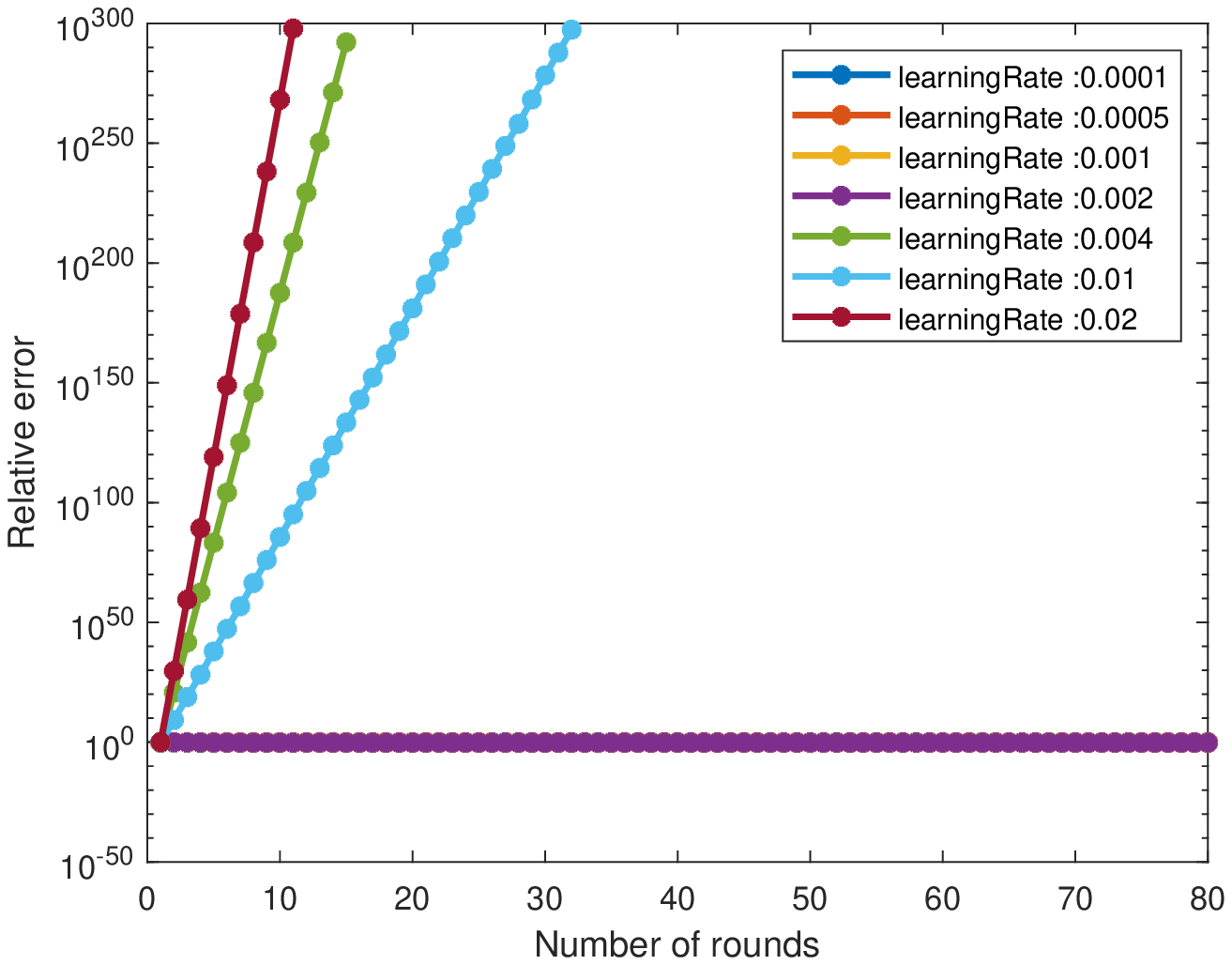}
			\end{minipage}
			\begin{minipage}{0.45\textwidth}
				\centering 
				\includegraphics[ width=1.0 \textwidth]{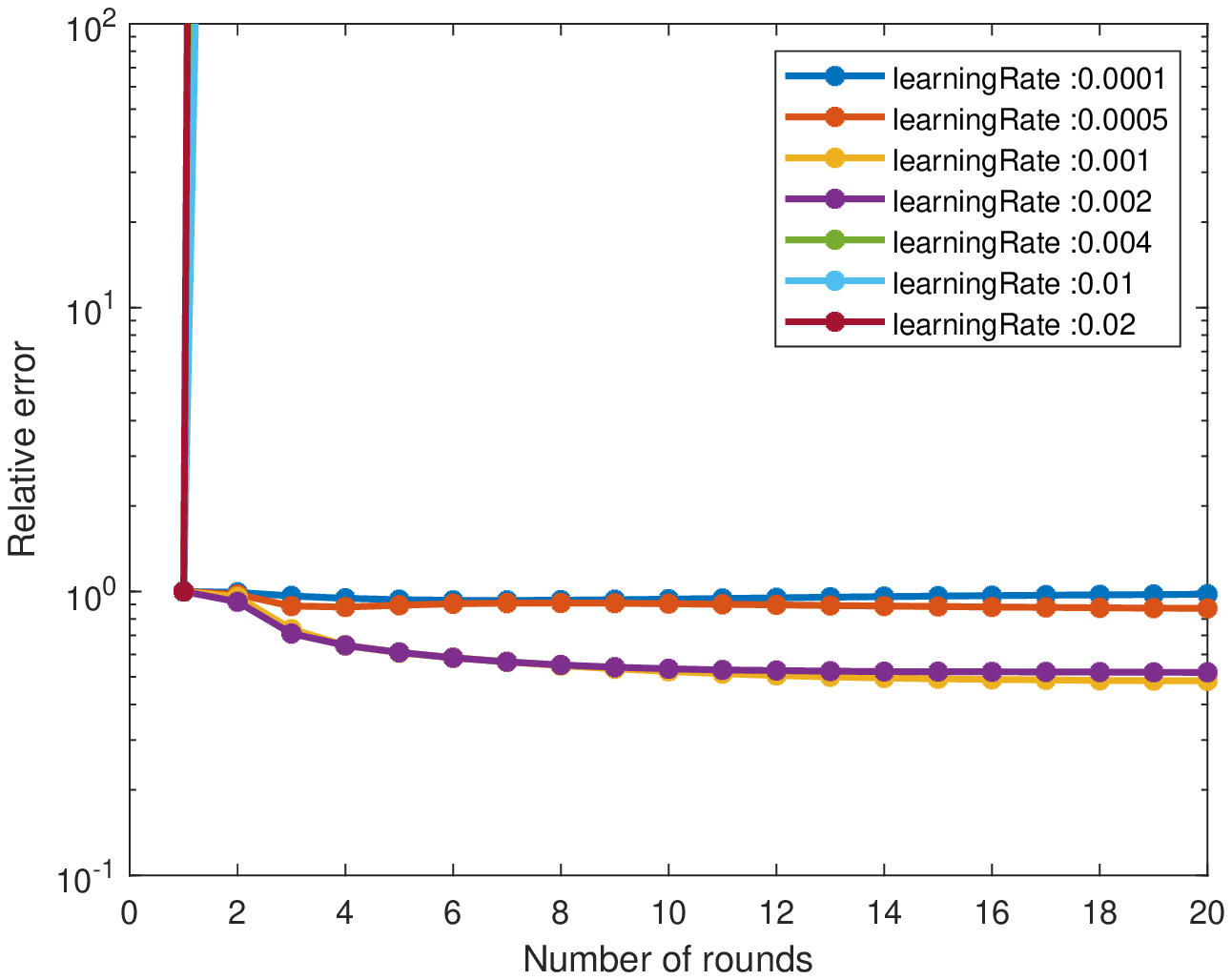}
			\end{minipage}
			\caption[] {FedIterHT with learning rates  $\{0.0001, 0.0005, 0.001, 0.002, 0.004, 0.01, 0.02 \}$ for non i.i.d. data sets.  
			}
			\label{fig:FedIterHT_noniid_saturated_diverege}
		\end{figure}
		
		\begin{figure}
			\centering 
			\includegraphics[ height=6cm]{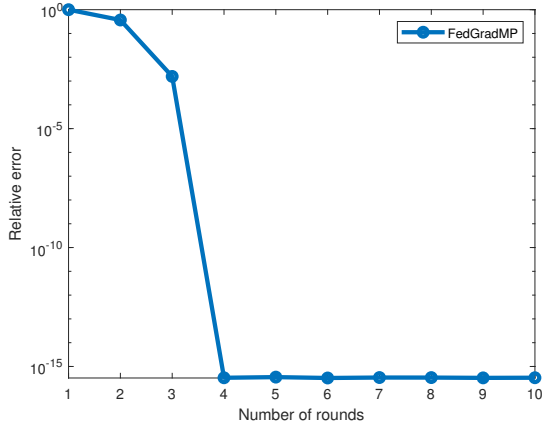}
			\caption[] {FedGradMP for non i.i.d. data sets.  
			}
			\label{fig:FedGradMP_noniid_new}
		\end{figure}
		
		If the learning rates are chosen from $\{ 0.004, 0.01, 0.02\}$, then the left plot in Figure \ref{fig:FedIterHT_noniid_saturated_diverege} show that they quickly diverge from the optimal solution.  
		
		On the other hand, the right panel in Figure \ref{fig:FedIterHT_noniid_saturated_diverege} shows the relative error and squared loss curves for FedIterHT when the learning rate is in $\{0.0001, 0.0005, 0.001, 0.002 \}$. For these smaller learning rates, the iterates of FedIterHT tend to converge to a highly suboptimal local solution.
		It has been observed in the literature \cite{aghazadeh2018mission}  that approaches based on stochastic gradient descents combined with hard-thresholding (such as FedIterHT) suffer from such phenomena when the learning rates are chosen to be too small.
		
		Hence, our numerical experiments indicate that the learning rates should be chosen very carefully for each client. Working learning rates should depend on the statistics and heterogeneity of the local data set at the client. Obtaining this information could be challenging because it might not be available in general, so usually, a grid search is performed to find  learning rates. 
		
		
		On the other hand, the iterates of FedGradMP converge to the ground truth up to (almost) machine precision as shown in Figure \ref{fig:FedGradMP_noniid_new}
		under the same setting, only in four rounds with three local iterations at the clients. Unlike FedIterHT, FedGradMP does not require fine tuning of learning rates per client. 
		
		
		\subsection{Impact of the number of local iterations}
		\label{section:Impact of the number of local iterations}
		
		We provide numerical evidence supporting Theorem \ref{thm:main_theorem1} about how the number of local iterations at clients affects the convergence rate and the residual error of FedGradMP. 
		
		\subsubsection*{Experiment settings}
		The number of clients is $50$, the dimension of solution space is $1000$, the number of data points of each client is $100$, the mini-batch size of each client is $30$, and the cohort size is $50$. The local objective function $f_i$ is the squared loss with associated data matrix $A_{D_i}$ for client $i$, similar to the one used for the heterogeneous case with $\alpha = 2.5$ in Section \ref{section:FedGradMP on synthetic Heterogeneous data sets}. 
		We run FedGradMP with local iterations $3, 6, 9, 12, 15$ for noiseless and noisy setup ($y_{D_i} = A_{D_i}x^{\#} + e$, where $e$ is a Gaussian noise where each component are independently generated according to $\mathcal{N}(0, 4 \times 10^{-6}$) ). 
		
		\subsubsection*{Simulation results}
		
		We display the relative error curves of iterates of FedGradMP 
		on the left and right panels of Figure \ref{fig:FedGradMP_LOCAL_ITERATIONS} for noiseless and noisy case respectively.  
		
		The error decay curves in the left plot for the  noiseless case demonstrate that as we increase the number of local iterations at clients, FedGradMP converges faster or the convergence rates improve. The plot on the right for the noisy case also exhibits a similar pattern but with a few exceptions probably due to the noise. This supports our theory about the dependence of convergence rate $\kappa$ on the number of local iterations in Theorem \ref{thm:main_theorem1} as explained in Remark \ref{rem:main_remark1}.
		
		As for the residual error of FedGradMP, we observe a general trend in the right panel that increasing the local iterations decreases the residual error, but this effect is not as noticeable as the convergence rate. This is somewhat expected since the residual error term in  \ref{thm:main_theorem1} depends on the local iteration numbers complicated way as explained in Remark \ref{rem:main_remark1}.
		
		\begin{figure}
			\centering
			\begin{minipage}{0.45\textwidth}
				\centering 
				\includegraphics[ width=1.0 \textwidth]{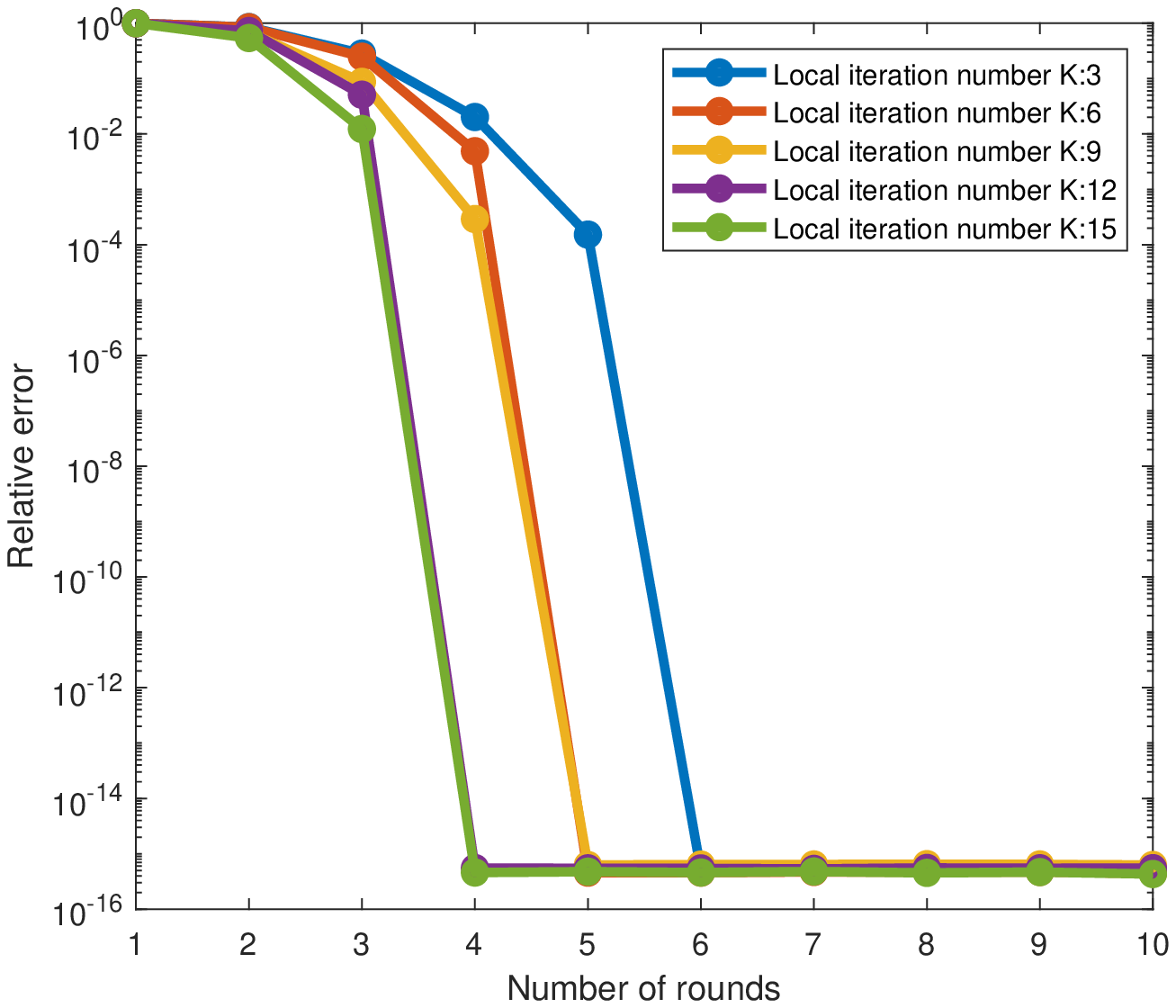}
			\end{minipage}
			\begin{minipage}{0.45\textwidth}
				\centering 
				\includegraphics[ width=1.0 \textwidth]{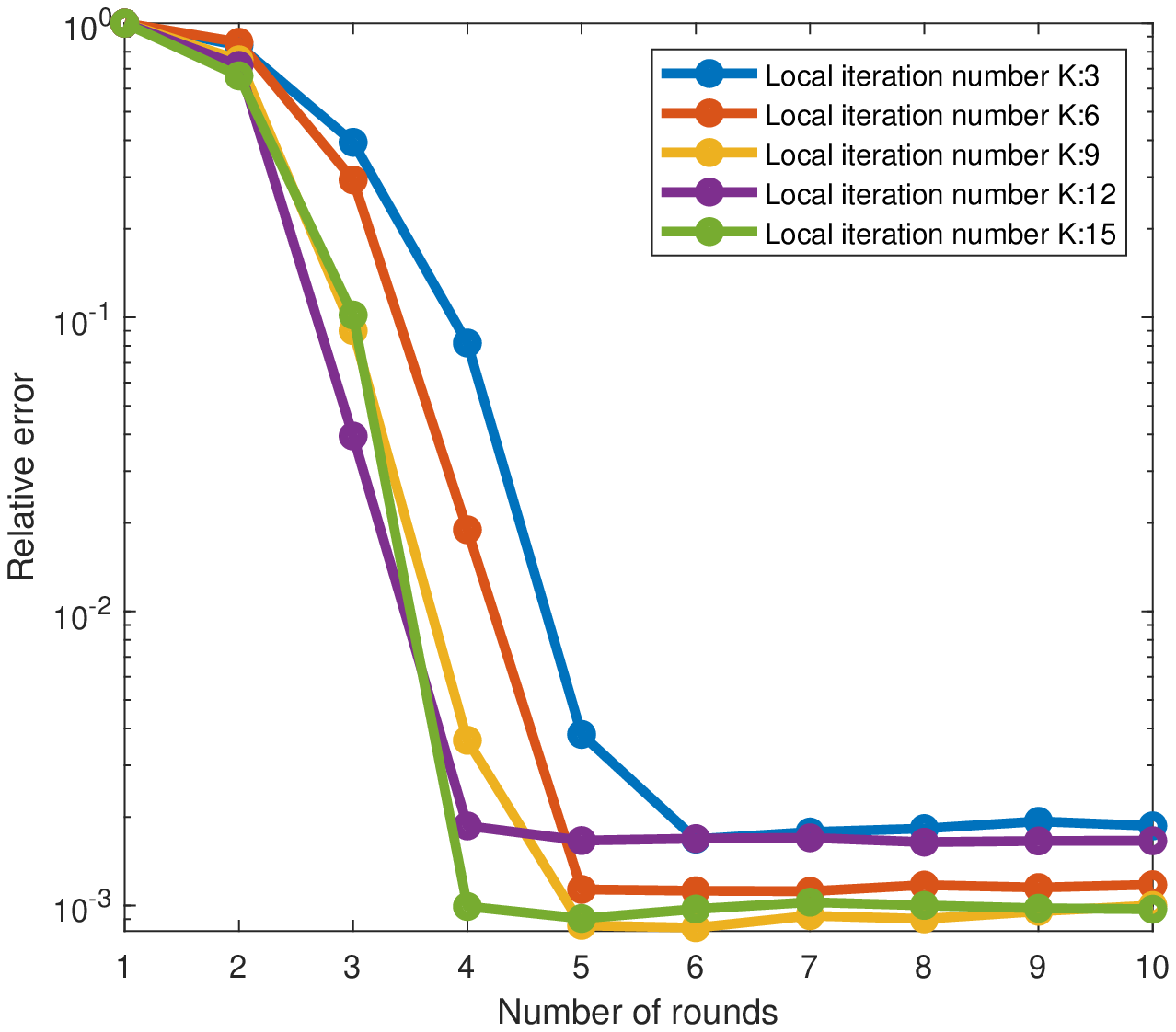}
			\end{minipage}
			\caption[] {The convergence rate improves as the local iterations at clients increase in Theorem \ref{thm:main_theorem1}. 
			}
			\label{fig:FedGradMP_LOCAL_ITERATIONS}
		\end{figure}

		\subsection{Impact of cohort size}
		\label{section:cohort_size}
		The next experiment illustrates how well FedGradMP performs when cohort size (the number of participating clients per round) varies. We notice that Figure \ref{fig:FedGradMP_COHORT_SIZE} provides numerical evidence supporting Theorem \ref{thm:main_theorem2} about how the cohort size affects the convergence rate and the residual error of FedGradMP. 
		
		\subsubsection*{Experiment settings}
		The number of clients is $50$, the dimension of solution space is $50$ and we set the mini-batch size $30$. The local objective function $f_i$ is the squared loss with associated non iid data matrix $A_{D_i}$ for client $i$, similar to the one used for the heterogeneous case with $\alpha = 2.5$ in Section \ref{section:FedGradMP on synthetic Heterogeneous data sets}. We run FedGradMP with cohort size $10, 15, 20, 25, 30$ for noiseless and noisy setup.

		\subsubsection*{Simulation results}
		The relative error curves of iterates of FedGradMP are given on the left panel (noiseless case) and right panels (noisy case) of Figure \ref{fig:FedGradMP_COHORT_SIZE}. These error plots indicate that the convergence rate improves as we increase the cohort size, for both noiseless and noisy cases as predicted in Theorem \ref{thm:main_theorem2}. 
		On the other hand, a careful reader might have noticed that the residual error actually slightly increases as the cohort size increases. This implies that the dependence of our residual error bound on the cohort size in Theorem \ref{thm:main_theorem2} is pessimistic and may not capture the true dependence as most of the other works in FL algorithm analysis. For more details, see the discussion and criticism on the gap between the current theoretical analyses of the impact of cohort size in FL algorithms and their empirical performance \cite{charles2021large}.

		\begin{figure}
			\centering
			\begin{minipage}{0.45\textwidth}
				\centering 
				\includegraphics[ width=1.0 \textwidth]{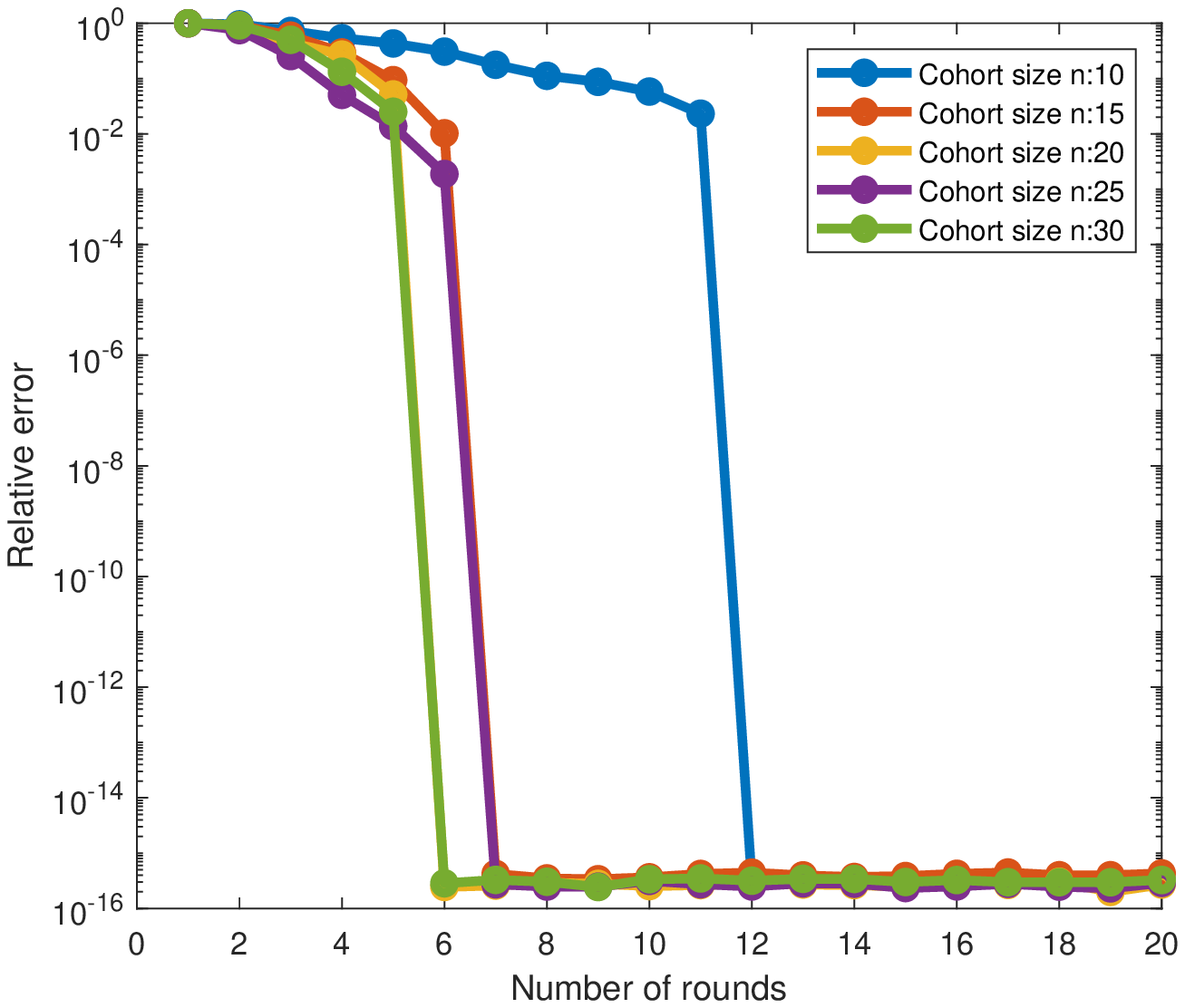}
			\end{minipage}
			\begin{minipage}{0.45\textwidth}
				\centering 
				\includegraphics[ width=1.0 \textwidth]{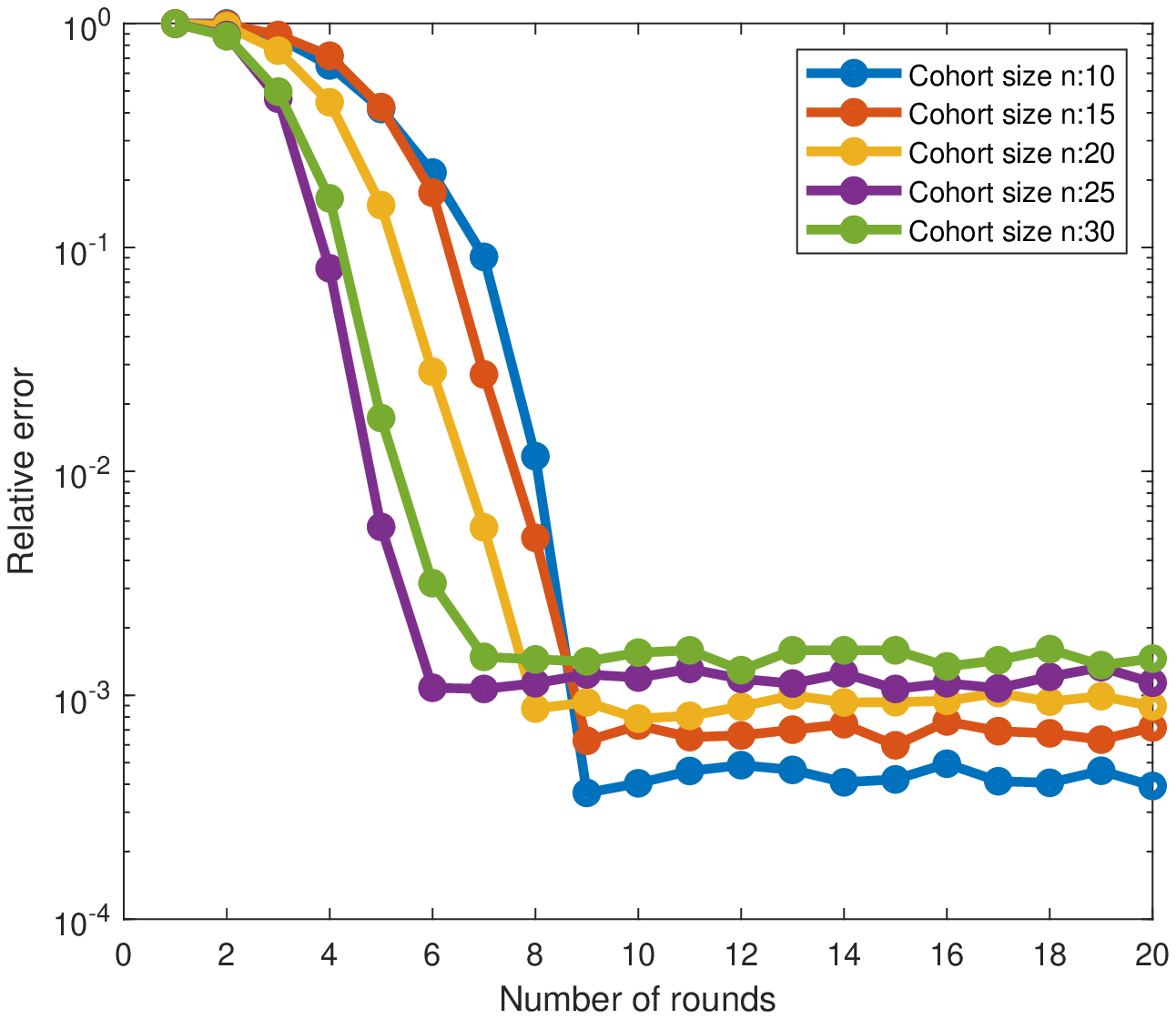}
			\end{minipage}
			\caption[] {The convergence rates improve as the cohort size increases as predicted in Theorem \ref{thm:main_theorem2}. Note that the residual errors in the right panel decay to zero (up to the machine precision) since all the non i.i.d. measurements are noiseless with the squared loss function. 
			}
			\label{fig:FedGradMP_COHORT_SIZE}
		\end{figure}

		\section{Conclusion}
		\label{sec:con}
		In this paper, we propose a novel federated stochastic gradient matching pursuit algorithm framework and show the linear convergence in expectation under certain assumptions of the objective function, including the dictionary restricted-RSS/RSC 
		conditions and the unbounded dissimilarity assumption. For the sparse linear regression problem, our method does not require learning rate tuning at the client side, which could be challenging for existing baseline algorithms in highly heterogeneous data environments. 
		Numerical experiments on large scale heterogeneous data sets such as EFMINIST and videos have shown the effectiveness of the proposed approach over the state-of-the-art federated learning algorithms. Our analysis reveals the benefits of adopting random dictionaries such as Gaussian random dictionary, which is also confirmed by our numerical experiments. 
		
		\section*{Acknowledgements}
		DN is supported by NSF DMS 2011140 and NSF DMS 2108479. The research of Qin is supported by the NSF grant DMS-1941197.

		\appendices
		
		\section{Proofs}
		
		\begin{proof}[Proof of Corollary~\ref{cor:7}]
			First, we recall that the global objective function $f(x) = \sum_{i=1}^N p_i f_i(x)$ and $f_i(x) = {1 \over M} \sum_{j=1}^M g_{i,j}(x)$. From Assumption \ref{assumption:main_assumptions2} on the  $\mathcal{A}$-RSS property of $g_{i,j}$ with constant $\rho^{+}_\tau(i,j) $, we have
			\begin{align*}
				\| \nabla g_{i,j}(x_1) - \nabla  g_{i,j} (x_2) \|_2 \le \rho^{+}_\tau(i,j) \|x_1 - x_2\|_2
			\end{align*} for all $x_1, x_2 \in \R^n$ with $\|x_1 - x_2\|_{0, \mathcal{A}} \le \tau$. By Lemma \ref{lem:consequence_RSS}, we have 
			\begin{align*}
				&\inner{\nabla g_{i,j}(x_1), x_2} \ge g_{i,j}(x_1 + x_2) - g_{i,j}(x_1) - {\rho^{+}_\tau(i,j)  \over 2} \|x_2\|_2^2.
			\end{align*}	
			Taking average $g_{i,j}$ over $j$ to recover $f_i$ and over $i$ with probability $p_i$ to recover $f$, the above inequality implies that 
			\begin{align*}
				\inner{\nabla f(x_1), x_2} \ge f(x_1 + x_2) - f(x_1) - {1 \over 2} \sum_{i=1}^N p_i \bar{\rho}^{+{(i)}}_\tau \|x_2\|_2^2.
			\end{align*} 
			Denote $ \sum_{i=1}^N p_i \bar{\rho}^{+{(i)}}_\tau$ by $\rho$. 
			Setting $x_2 = x_{t+1} - x^*$ and $x_1 = x^*$ in the above inequality yields
			\begin{align*}
				f(x_{t+1}) 
				&\le f(x^*) + \inner{\nabla f(x^*), x_{t+1} - x^*} + {\rho\over 2}  \|x_{t+1} - x^*\|_2^2\\
				&\le f(x^*) + \|\nabla f(x^*)\|_2 \| x_{t+1} - x^*\|_2 + {\rho\over 2}  \|x_{t+1} - x^*\|_2^2\\
				&\le f(x^*) + {1\over 2\rho} \|\nabla f(x^*)\|_2^2 + {\rho\over 2} \| x_{t+1} - x^*\|_2^2 + {\rho\over 2}  \|x_{t+1} - x^*\|_2^2\\
				&\le f(x^*) + {1\over 2\rho} \|\nabla f(x^*)\|_2^2 + \rho \| x_{t+1} - x^*\|_2^2.
			\end{align*} Here the third inequality follows from the AM-GM inequality. 
			Taking the expectation to the last inequality, we have
			\begin{align*}
				\mathbb{E} f(x_{t+1}) 
				&\le f(x^*) + {1\over 2\rho} \|\nabla f(x^*)\|_2^2 + \rho \mathbb{E} \| x_{t+1} - x^*\|_2^2.
			\end{align*}
			Finally, we apply Theorem \ref{thm:main_theorem1} to the above inequality to establish the statement in the corollary. 
		\end{proof}
		
		\begin{proof} [Proof of Theorem \ref{thm:main_theorem1_inexact_solver}]
			We follow the same arguments used in the first few steps of the proof of Theorem \ref{thm:main_theorem1} and obtain the following inequality.
			\begin{align}
				\nonumber
				& \mathbb{E} \|x_{t+1} - x^*\|_2^2 \le  (2 \eta_3^2 + 2)   \sum_{i=1}^N p_i \mathbb{E}^{(i)}_{J_K} \left \|  x_{t,K+1}^{(i)}   - x^*  \right \|_2^2. 
			\end{align} 
			
			Because we are solving $b^{(i)}_{t,k} = \argmin{x} f_i(x)$ for $ x \in R(D_{\widehat{\Gamma}})$ with an accuracy $\delta$, 
			we have 
			\begin{align}
				\nonumber
				& \sum_{i=1}^N p_i \mathbb{E}^{(i)}_{J_K} \left \|  x_{t,K+1}^{(i)}   - x^*  \right \|_2^2   \\
				\nonumber
				& \le (1 + \eta_2)^2 \sum_{i=1}^N p_i \mathbb{E}^{(i)}_{J_K}\left \|b^{(i)}_{t,K} - x^* \right \|_2^2 \\
				\nonumber
				& \le (1 + \eta_2)^2 \sum_{i=1}^N p_i \left[ 2\mathbb{E}^{(i)}_{J_K}\left \|b^{(i, \text{opt})}_{t,K} - x^* \right \|_2^2 + 2\mathbb{E}^{(i)}_{J_K}\left \|b^{(i, \text{opt})}_{t,K-1} - b^{(i)}_{t,K} \right \|_2^2  \right] \\
				\nonumber
				& \le (1 + \eta_2)^2 \sum_{i=1}^N p_i \left[ 2\mathbb{E}^{(i)}_{J_K}\left \|b^{(i, \text{opt})}_{t,K} - x^* \right \|_2^2 + 2 \delta^2 \right] \\
				\nonumber
				& \le  2(1 + \eta_2)^2 \sum_{i=1}^N  p_i \left[  \beta_1(i)  \mathbb{E}^{(i)}_{J_K} \|P^\perp_{\widehat{\Gamma}}( b^{(i)}_{t,K} - x^*)\|_2^2 + \xi_1(i) +  \delta^2 \right].
			\end{align}
			The rest of the proof is similar to that of Theorem \ref{thm:main_theorem1}. 
		\end{proof}

		\begin{proof} [Proof of Theorem \ref{thm:main_theorem2}]
			As in the proof of Theorem \ref{thm:main_theorem1}, let $\mathcal{F}^{(t)}$ be the filtration by all the randomness up to the $t$-th communication round, but in this case, it is all the selected participating clients and the selected mini-batch indices at all these clients  up to the $t$-th round. Let us denote the client subset selected at round $t$ by $I_t$. Note that $I_t$ is chosen uniformly at random over all possible subsets of cardinality $L$ whose elements belong to $[N]$, so $|I_t| = L$.
			Again, as we did in the proof of Theorem \ref{thm:main_theorem1}, by abusing the notation slightly, $\mathbb{E} \left[ \cdot | \mathcal{F}^{(t)} \right]$ will be denoted $\mathbb{E}_{(I_t)} \left[\mathbb{E} [\cdot] \right]$, where $\mathbb{E}_{(I_t)}$ is the expectation taken over the randomly selected participating clients at round $t$.
			
			We first consider the case for $\eta_1 > 1$.
			By following the same argument for the first step of the proof for Theorem \ref{thm:main_theorem1}, we have
			\begin{align}
				\nonumber
				\mathbb{E}_{(I_t)} \mathbb{E} \|x_{t+1} - x^*\|_2^2  
				&  = \mathbb{E}_{(I_t)} \mathbb{E} \left \| P_{\Lambda_s} \left( \sum_{i \in I_t} {1 \over L}  x_{t,K+1}^{(i)} \right) - \sum_{i \in I_t}  {1 \over L}  x_{t,K+1}^{(i)} + \sum_{i \in I_t} {1 \over L}  x_{t,K+1}^{(i)}  - x^* \right \|_2^2 \\ \nonumber
				&  \le 2 \mathbb{E}_{(I_t)} \mathbb{E} \left \| P_{\Lambda_s} \left( \sum_{i \in I_t} {1 \over L}  x_{t,K+1}^{(i)} \right) - \sum_{i \in I_t}  {1 \over L}  x_{t,K+1}^{(i)} \right \|_2^2 +  2 \mathbb{E}_{(I_t)} \mathbb{E} \left \| \sum_{i \in I_t} {1 \over L}  x_{t,K+1}^{(i)}  - x^* \right \|_2^2 \\ \nonumber
				& = (2 \eta_3^2 + 2) \mathbb{E}_{(I_t)} \mathbb{E} \left \|  \sum_{i \in I_t} {1 \over L}  x_{t,K+1}^{(i)}   -  \sum_{i \in I_t}  {1 \over L}  x^*  \right \|_2^2 \\
				& \le  (2 \eta_3^2 + 2) \mathbb{E}_{(I_t)}  \left[ \sum_{i \in I_t}  {1 \over L} \mathbb{E} \left \|  x_{t,K+1}^{(i)}   - x^*  \right \|_2^2 \right] \\
				& \le  (2 \eta_3^2 + 2) \mathbb{E}_{(I_t)}  \left[ \sum_{i \in I_t}  {1 \over L}  \mathbb{E}^{(i)}_{J_K} \left \|  x_{t,K+1}^{(i)}   - x^*  \right \|_2^2 \right].
			\end{align} 
			
			Moreover, the argument used in the proof of Theorem \ref{thm:main_theorem1} yields
			
			\begin{align*}
				&  \sum_{i \in I_t}  {1 \over L}  \mathbb{E}^{(i)}_{J_K} \left \|  x_{t,K+1}^{(i)}   - x^*  \right \|_2^2    
				\le (1 + \eta_2)^2  \sum_{i \in I_t}  {1 \over L}  \beta_1(i) \beta_2(i) \mathbb{E}^{(i)}_{J_{K-1}}  \| x^{(i)}_{t,K} - x^* \|_2^2 \\
				&\qquad + (1 + \eta_2)^2  \max_i \left(  {8\beta_1(i) \over (\rho_{4\tau}^{-}(i) )^2} + {4 \over \bar{\rho}^{+(i)}_{4\tau} (2\rho^{-}_{4\tau}(i) - \bar{\rho}^{+(i)}_{4\tau} ) }  + {\beta_1(i) \over \rho_{4\tau}^{-}(i) }{6\sqrt{\eta_1^2-1} \over \eta_1}   \right) \zeta_*^2 \\
				&\qquad +  (1 + \eta_2)^2  \sum_{i \in I_t} {1 \over L}  \left[{\beta_1(i) \over \rho_{4\tau}^{-}(i) }  \left({2\theta^2 } + {6\sqrt{\eta_1^2-1} \over \eta_1} \right) + {4 \over \bar{\rho}^{+(i)}_{4\tau} (2\rho^{-}_{4\tau}(i) - \bar{\rho}^{+(i)}_{4\tau} ) } \right]\sigma_i^2.
			\end{align*}
			
			We define $\nu (I_t)$ that depends on the random index set $I_t$ as follows:
			\begin{align*}
				\nu (I_t)  & =   (1 + \eta_2)^2  \max_i \left(  {8\beta_1(i) \over (\rho_{4\tau}^{-}(i) )^2} + {4 \over \bar{\rho}^{+(i)}_{4\tau} (2\rho^{-}_{4\tau}(i) - \bar{\rho}^{+(i)}_{4\tau} ) }  + {\beta_1(i) \over \rho_{4\tau}^{-}(i) }{6\sqrt{\eta_1^2-1} \over \eta_1}   \right) \zeta_*^2 \\
				&\qquad +  (1 + \eta_2)^2  \sum_{i \in I_t} {1 \over L}  \left[{\beta_1(i) \over \rho_{4\tau}^{-}(i) }  \left({2\theta^2 } + {6\sqrt{\eta_1^2-1} \over \eta_1} \right) + {4 \over \bar{\rho}^{+(i)}_{4\tau} (2\rho^{-}_{4\tau}(i) - \bar{\rho}^{+(i)}_{4\tau} ) } \right]\sigma_i^2.
			\end{align*} After rewriting the previous inequality, we obtain
			\begin{align}
				&  \sum_{i \in I_t}  {1 \over L} \mathbb{E}^{(i)}_{J_K} \left \|  x_{t,K+1}^{(i)}  - x^*  \right \|_2^2  \le   \sum_{i \in I_t}   {1 \over L} \mu (i) \mathbb{E}^{(i)}_{J_{K-1}} \| x^{(i)}_{t,K} - x^* \|_2^2 + \nu (I_t).
			\end{align}
			
			Hence, by the induction on $K$ and using the fact that the cohort set $I_t$ is fixed while the local iterations are running, we obtain a similar upper bound on $	\mathbb{E} \|x_{t+1} - x^*\|_2^2$ as follows. 
			\begin{align*}
				\mathbb{E} \|x_{t+1} - x^*\|_2^2 
				& \le  (2 \eta_3^2 + 2) \mathbb{E}_{(I_t)}  \sum_{i \in I_t} {1 \over L} \left( \mu(i)^K \left[  \mathbb{E}^{(i)} \|x^{(i)}_{t,1} - x^*\|_2^2 \right] +  {\nu (I_t) (1 - \mu(i)^K) \over 1 - \mu(i)} \right) \\
				&=  (2 \eta_3^2 + 2) \mathbb{E}_{(I_t)} \left(  \left( \sum_{i \in I_t} {1 \over L} \mu(i)^K \right) \mathbb{E} \|x_t - x^*\|_2^2+ \nu (I_t)  \sum_{i \in I_t} {1 \over L}  {(1 - \mu(i)^K) \over 1 - \mu(i)} \right)\\
				&\le  (2 \eta_3^2 + 2) \mathbb{E}_{(I_t)} \left(  \left( \sum_{i \in I_t} {1 \over L} \mu(i)^K \right) \mathbb{E} \|x_t - x^*\|_2^2+ \nu (I_t)   {(1 - \mu^K) \over 1 - \mu} \right). 
			\end{align*} 
			Recall that the index set $I_t$ is a subset of $[N]$, uniformly selected at random, for the communication round $t$. By taking the maximum of $\sum\limits_{i \in I_t} {1 \over L} \mu(i)^K$ over all  possible subsets, we have 
			\begin{align*}
				\mathbb{E} \|x_{t+1} - x^*\|_2^2 
				&\le   \kappa \mathbb{E} \|x_t - x^*\|_2^2 +   {(2 \eta_3^2 + 2) (1 - \mu^K) \over 1 - \mu} \mathbb{E}_{(I_t)} [\nu (I_t)]\\
				&\le   \kappa \mathbb{E} \|x_t - x^*\|_2^2 +   {(2 \eta_3^2 + 2) \tilde{\nu} (1 - \mu^K) \over 1 - \mu} 
			\end{align*}
			where 
			\[
			\kappa =  (2 \eta_3^2 + 2)  \max_{\substack{ S \subset [N] \\ |S| = L}} {1 \over L} \sum_{z \in S}  \left[(1 + \eta_2)^2   \beta_1(z) \beta_2(z) \right]^K,
			\]
			and
			\begin{align*}
				&\tilde{\nu}  = (1 + \eta_2)^2  \max_i \left(  {8\beta_1(i) \over (\rho_{4\tau}^{-}(i) )^2} + {4 \over \bar{\rho}^{+(i)}_{4\tau} (2\rho^{-}_{4\tau}(i) - \bar{\rho}^{+(i)}_{4\tau} ) }  + {\beta_1(i) \over \rho_{4\tau}^{-}(i) }{6\sqrt{\eta_1^2-1} \over \eta_1}   \right) \zeta_*^2 \\
				&\qquad+  (1 + \eta_2)^2 {1 \over L} \sum\limits_{i=1}^N \left[{\beta_1(i) \over \rho_{4\tau}^{-}(i) }  \left({2\theta^2 } + {6\sqrt{\eta_1^2-1} \over \eta_1} \right) + {4 \over \bar{\rho}^{+(i)}_{4\tau} (2\rho^{-}_{4\tau}(i) - \bar{\rho}^{+(i)}_{4\tau} ) } \right]\sigma_i^2. 
			\end{align*}  
			
			Hence, by the induction on $t$, we have
			\[
			\mathbb{E} \|x_{t+1} - x^*\|_2^2 
			\le  \kappa^{t+1} \mathbb{E} \|x_0 - x^*\|_2^2  +  {(2 \eta_3^2 + 2) \tilde{\nu} (1 - \mu^K)   \over (1-\kappa)(1 - \mu)}.
			\]
			
			The case for $\eta_1 = 1$ follows from a similar argument. 
		\end{proof}
		
		\section{FedGradMP convergence without the bounded variance condition of stochastic gradients} 
		
		We start with the following lemma replacing the bounded variance condition of stochastic gradients \eqref{asump:local_gradients2} in Assumption \ref{asump:local_gradients} only under the $\mathcal{A}$-RSS condition. 
		
		\begin{lem}
			\label{lem:variance_stochastic_gradients}
			Let $\mathbb{E}_{j}$ be the expectation over the uniform distribution on all possible mini-batches. Then, for all $\tau$-sparse vectors $x$, we have
			\begin{align*}
				&\mathbb{E}_{j} \|\nabla g_{i,j}(x) - \nabla f_i \left(x \right)\|_2^2 \le 3 \mathbb{E}_{j} ((\rho^{+}_\tau(i,j))^2 + \bar{\rho}^{+(i)}_{4\tau}) \| \Delta \|_2^2 + 12 \mathbb{E}_{j} \|\nabla g_{i,j}(x^*)\|_2^2
			\end{align*} and 
			\begin{align*}
				&\mathbb{E}_{j} \|P_{\Gamma} ( \nabla g_{i,j}(x) - \nabla f_i \left(x \right) )\|_2^2 \le 3 \mathbb{E}_{j} ((\rho^{+}_\tau(i,j))^2 + \bar{\rho}^{+(i)}_{4\tau}) \| \Delta \|_2^2 + 12 \mathbb{E}_{j} \|P_{\Gamma} \nabla g_{i,j}(x^*)\|_2^2,
			\end{align*}
			where $x^*$ is a solution to \eqref{eq:main_problem} and $\Delta = x - x^*$.
		\end{lem}
		
		\begin{proof} [Proof of Lemma]
			\begin{align*}
				&\mathbb{E}_{j} \|\nabla g_{i,j}(x) - \nabla f_i \left(x \right)\|_2^2 \\
				&\le 3 \mathbb{E}_{j} \|\nabla g_{i,j}(x) - \nabla g_{i,j}(x^*)  \|_2^2 
				+ 3 \mathbb{E}_{j} \|\nabla g_{i,j}(x^*) - \nabla f_i \left(x^* \right)\|_2^2 
				+ 3 \mathbb{E}_{j} \| \nabla f_i \left(x^* \right) - \nabla f_i \left(x\right)\|_2^2 \\
				&\le 3  \mathbb{E}_{j} (\rho^{+}_\tau(i,j))^2 \|x - x^* \|_2^2 + 6 \mathbb{E}_{j} \|\nabla g_{i,j}(x^*)\|_2^2 + 6 \| \nabla f_i \left(x^* \right)\|_2^2 + 3 \mathbb{E}_{j} \bar{\rho}^{+(i)}_{4\tau} \|x  - x^* \|_2^2\\ 
				&= 3  \mathbb{E}_{j} (\rho^{+}_\tau(i,j))^2 \| \Delta \|_2^2 + 6 \mathbb{E}_{j} \|\nabla g_{i,j}(x^*)\|_2^2 + 6 \| \nabla f_i \left(x^* \right)\|_2^2 + 3 \mathbb{E}_{j} \bar{\rho}^{+(i)}_{4\tau} \| \Delta  \|_2^2\\ 
				&\le 3 \mathbb{E}_{j} ( (\rho^{+}_\tau(i,j))^2 + \bar{\rho}^{+(i)}_{4\tau}) \| \Delta \|_2^2 + 6 \mathbb{E}_{j} \|\nabla g_{i,j}(x^*)\|_2^2 + 6 \| \nabla f_i \left(x^* \right)\|_2^2\\
				&\le 3 \mathbb{E}_{j} ((\rho^{+}_\tau(i,j))^2 + \bar{\rho}^{+(i)}_{4\tau}) \| \Delta \|_2^2 + 12 \mathbb{E}_{j} \|\nabla g_{i,j}(x^*)\|_2^2.
			\end{align*}
			The second inequality follows from the  $\mathcal{A}$-RSS condition for $\nabla g_{i,j}$ with constant $\rho^{+}_\tau(i,j)$, and the fact that $\nabla f_i$ is the average of $\nabla g_{i,j}$. The last inequality is from the Jensen's inequality. This proves the first part of the lemma and the second part follows from a similar argument. 
		\end{proof}
		
		This lemma allows us to prove a similar statement as in Lemma \ref{lem:GradMP_lem3} without the bounded variance condition \eqref{asump:local_gradients2}. Since the underlying argument of the proof of the following lemma is the same, we only point out the difference from the proof for Lemma \ref{lem:GradMP_lem3}. 
		
		\begin{lem}
			\label{lem:GradMP_lem3_without_bounded_variance}
			Let $\widehat{\Gamma}$ be the set obtained from the $k$-th iteration at client $i$. Then, for any $\theta > 0$, we have
			\[
			\mathbb{E}^{(i)}_{j_k}  \|P^\perp_{\widehat{\Gamma}}( b^{(i)}_{t,k} - x^*)\|_2^2 \le \beta_2(i) \| x^{(i)}_{t,k} - x^* \|_2^2 + \xi_2(i),
			\] where 
			\begin{align*}
				\beta_2(i) &= \left(4{(2\eta_1^2 -1) \left( {\bar{\rho}^{+(i)}_{4\tau} } +  {1 \over \theta^2} \right)  - \eta_1^2 \rho^-_{4\tau}(i)   \over \eta_1^2 \rho_{4\tau}^{-}(i)} + \left({3\theta^2} \mathbb{E}_{j_k} (\rho^{+}_\tau(i,j_k) + \bar{\rho}^{+(i)}_{4\tau}) \over \rho_{4\tau}^{-}(i) \right)  + {{2(\eta^2_1 - 1)} \over \eta_1^2} \right) \\
				\xi_2(i)  &=
				{8 \over (\rho_{4\tau}^{-}(i) )^2}   \max\limits_{\substack{ \Omega \subset [d] \\ |\Omega| = 4\tau}}  \| P_{\Omega}  \nabla f_i(x^*) \|_2^2  + 2\left({6\theta^2} + {15\sqrt{\eta_1^2-1} \over 2\eta_1} \right)  \mathbb{E}_{j_k} \|\nabla g_{i,j_k}(x^*)\|_2^2.
			\end{align*} Note that if $\eta_1=1$, then the projection operator is exact. 
			Here $\mathbb{E}^{(i)}_{j_k}$ is the expectation taken over the randomly selected index $j_k$ at the $k$-th step of the local iterations of the $i$-th client.
		\end{lem}
		
		\begin{proof}
			We follow the same steps in the proof of Lemma \ref{lem:GradMP_lem3} for the bound $f_i(x^*) - f_i \left(x^{(i)}_{t,k} \right) -  {\rho_{4\tau}^{-}(i) \over 2} \| x^* - x^{(i)}_{t,k} \|_2^2$ but apply Lemma \ref{lem:GradMP_lem3_without_bounded_variance} to the inequality \ref{eq:bound2_Lemma_mod} as follows.

			\begin{align}
				\nonumber
				&f_i(x^*) - f_i \left(x^{(i)}_{t,k} \right) -  {\rho_{4\tau}^{-}(i) \over 2} \| x^* - x^{(i)}_{t,k} \|_2^2 \\
				\nonumber
				& \ge  \mathbb{E}_{j_k} \inner { \nabla f_i\left(x^{(i)}_{t,k} \right),  z } -  {\theta^2 \over 2} \|\nabla g_{i,j_k}(x^{(i)}_{t,k}) - \nabla f_i \left(x^{(i)}_{t,k} \right)\|_2^2 \\
				\nonumber
				&\qquad - {1 \over 2\theta^2} \mathbb{E}_{j_k} \|z\|_2^2 -  {\bar{\rho}^{+(i)}_{4\tau} \over 2} \mathbb{E}_{j_k} \| z\|_2^2  - {\sqrt{\eta_1^2-1} \over 2 \eta_1} \left(\mathbb{E}_{j_k}  \|P^\perp_{\Gamma} \nabla g_{i,j_k}(x^{(i)}_{t,k}) \|_2^2 + \mathbb{E}_{j_k} \|  \Delta \|_2^2  \right) \\
				\label{eq:bound2_Lemma_mod}
				& \ge  \mathbb{E}_{j_k} \inner { \nabla f_i\left(x^{(i)}_{t,k} \right),  z } -  {\theta^2 \over 2} ( 3 \mathbb{E}_{j} ((\rho^{+}_\tau(i,j))^2 + \bar{\rho}^{+(i)}_{4\tau}) \| \Delta \|_2^2 + 12 \mathbb{E}_{j} \|\nabla g_{i,j}(x^*)\|_2^2) \\
				\nonumber
				&\qquad - {1 \over 2\theta^2} \mathbb{E}_{j_k} \|z\|_2^2 -{\bar{\rho}^{+(i)}_{4\tau} \over 2} \mathbb{E}_{j_k} \| z\|_2^2  - {\sqrt{\eta_1^2-1} \over 2 \eta_1} \left(\mathbb{E}_{j_k}  \|P^\perp_{\Gamma} \nabla g_{i,j_k}(x^{(i)}_{t,k}) \|_2^2 + \mathbb{E}_{j_k} \|  \Delta \|_2^2  \right).
			\end{align}
			
			Similarly, we obtain the upper bound for $\mathbb{E}_{j_k}  \|P^\perp_{\Gamma} \nabla g_{i,j_k}(x^{(i)}_{t,k}) \|_2^2$ as follows.
			\begin{align*}
				&\mathbb{E}_{j_k}  \|P^\perp_{\Gamma} \nabla g_{i,j_k}(x^{(i)}_{t,k}) \|_2^2 \\
				&\le \mathbb{E}_{j_k}  \|\nabla g_{i,j_k}(x^{(i)}_{t,k}) \|_2^2 \\
				&\le 3 \mathbb{E}_{j_k} \|\nabla g_{i,j_k}(x^{(i)}_{t,k}) - \nabla g_{i,j_k}(x^*)  \|_2^2 
				+ 3 \mathbb{E}_{j_k} \|\nabla g_{i,j_k}(x^*) - \nabla f_i \left(x^* \right)\|_2^2 
				+ 3 \mathbb{E}_{j_k} \| \nabla f_i \left(x^* \right) \|_2^2 \\
				&\le 3  \mathbb{E}_{j_k} (\rho^{+}_\tau(i,j_k))^2 \|x^{(i)}_{t,k}  - x^* \|_2^2 + 6 \mathbb{E}_{j_k} \|\nabla g_{i,j_k}(x^*)\|_2^2 + 6 \| \nabla f_i \left(x^* \right)\|_2^2 +  3  \| \nabla f_i \left(x^* \right) \|_2^2\\ 
				&= 3  \mathbb{E}_{j_k} (\rho^{+}_\tau(i,j_k))^2 \| \Delta \|_2^2 + 6 \mathbb{E}_{j_k} \|\nabla g_{i,j_k}(x^*)\|_2^2 + 9 \| \nabla f_i \left(x^* \right)\|_2^2\\
				&= 3  \mathbb{E}_{j_k} (\rho^{+}_\tau(i,j_k))^2 \| \Delta \|_2^2 + 15 \mathbb{E}_{j_k} \|\nabla g_{i,j_k}(x^*)\|_2^2,
			\end{align*} where the last inequality is from the Jensen's inequality. 
			
			Applying this bound for $\mathbb{E}_{j_k}  \|P^\perp_{\Gamma} \nabla g_{i,j_k}(x^{(i)}_{t,k}) \|_2^2$ yields 
			\begin{align*}
				&f_i(x^*) - f_i \left(x^{(i)}_{t,k} \right) -  {\rho_{4\tau}^{-}(i) \over 2} \| \Delta \|_2^2  \\
				& \ge  \mathbb{E}_{j_k} \inner { \nabla f_i\left(x^{(i)}_{t,k} \right),  z } -  {\theta^2 \over 2} ( 3 \mathbb{E}_{j} ((\rho^{+}_\tau(i,j))^2 + \bar{\rho}^{+(i)}_{4\tau}) \| \Delta \|_2^2 + 12 \mathbb{E}_{j} \|\nabla g_{i,j}(x^*)\|_2^2) - {\bar{\rho}^{+(i)}_{4\tau} \over 2} \mathbb{E}_{j_k} \| z\|_2^2  \\
				\nonumber
				&\qquad - {1 \over 2\theta^2} \mathbb{E}_{j_k} \|z\|_2^2  - {\sqrt{\eta_1^2-1} \over 2 \eta_1} \left(3  \mathbb{E}_{j_k} (\rho^{+}_\tau(i,j_k))^2 \| \Delta \|_2^2 + 15 \mathbb{E}_{j_k} \|\nabla g_{i,j_k}(x^*)\|_2^2 + \mathbb{E}_{j_k} \|  \Delta \|_2^2  \right).
			\end{align*}
			
			Following the same argument in the proof of Lemma \ref{lem:GradMP_lem3}, we have
			\begin{align*}
				&\left( {\bar{\rho}^{+(i)}_{4\tau} \over 2} +  {1 \over 2\theta^2} \right) \mathbb{E}_{j_k} \| z\|_2^2 - {1 \over 2} \left(\rho_{4\tau}^{-}(i)  -  {3\theta^2} \mathbb{E}_{j_k} ((\rho^{+}_\tau(i,j_k))^2 + \bar{\rho}^{+(i)}_{4\tau}) - {\sqrt{\eta_1^2-1} \over \eta_1} (3 \mathbb{E}_{j_k} \rho^{+}_\tau(i,j_k))^2 + 1)  \right) \|\Delta\|_2^2 \\
				& \qquad + \left({6\theta^2} + {15\sqrt{\eta_1^2-1} \over 2\eta_1} \right)  \mathbb{E}_{j_k} \|\nabla g_{i,j_k}(x^*)\|_2^2  \\
				&\ge \mathbb{E}_{j_k} f_i \left(x^{(i)}_{t,k} +  z \right) -  f_i \left(x^* \right) \\
				& \ge {\rho^-_{4\tau}(i) \over 2}  \|\Delta - z  \|_2^2  - \max_{\substack{ \Omega \subset [d] \\ |\Omega| = 4\tau}}  \| P_{\Omega}  \nabla f_i(x^*) \|_2 \mathbb{E}_{j_k} \| \Delta -z \|_2.
			\end{align*}
			
			Let $u = \mathbb{E}_{j_k} \|\Delta - y\|_2 $, $a =  \rho^-_{4\tau}(i) $, $b = \max\limits_{\substack{ \Omega \subset [d] \\ |\Omega| = 4\tau}}  \| P_{\Omega}  \nabla f_i(x^*) \|_2$, and 
			\small
			\begin{align*}
				c &=\left( {\bar{\rho}^{+(i)}_{4\tau} \over 2} +  {1 \over 2\theta^2} \right) \mathbb{E}_{j_k} \| z\|_2^2 - {1 \over 2} \left(\rho_{4\tau}^{-}(i)  -  {3\theta^2} \mathbb{E}_{j_k} ((\rho^{+}_\tau(i,j_k))^2 + \bar{\rho}^{+(i)}_{4\tau}) - {\sqrt{\eta_1^2-1} \over \eta_1} (3 \mathbb{E}_{j_k} \rho^{+}_\tau(i,j_k))^2 + 1)  \right) \|\Delta\|_2^2 \\
				& \qquad + \left({6\theta^2} + {15\sqrt{\eta_1^2-1} \over 2\eta_1} \right)  \mathbb{E}_{j_k} \|\nabla g_{i,j_k}(x^*)\|_2^2.
			\end{align*}
			\normalsize
			Then above inequality can be rewritten in $au^2 - 2bu - c \le 0$ and solving it gives
			\[
			\mathbb{E}_{j_k} \|\Delta - y\|_2 \le \sqrt{c \over a} + {2b \over a}.
			\]
			Again, following the same argument for the proof of Lemma \ref{lem:GradMP_lem3}, we get
			\begin{align*}
				\mathbb{E}^{(i)}_{j_k}\|\Delta - P_\Gamma \Delta \|_2^2 
				&\le {2c \over a} + {8b^2 \over a^2}.
			\end{align*}
			Thus, 
			\begin{align*}
				&\mathbb{E}^{(i)}_{j_k}\|\Delta - P_\Gamma \Delta \|_2^2 \\
				& \quad \le \left(2{\left( {\bar{\rho}^{+(i)}_{4\tau} } +  {1 \over \theta^2} \right)  - \eta_1^2 \rho^-_{4\tau}(i)   \over \eta_1^2 \rho_{4\tau}^{-}(i)} + \left({3\theta^2} \mathbb{E}_{j_k} ((\rho^{+}_\tau(i,j_k))^2 + \bar{\rho}^{+(i)}_{4\tau}) \over \rho_{4\tau}^{-}(i) \right)  +{\sqrt{\eta_1^2-1} \over \eta_1} (3 \mathbb{E}_{j_k} \rho^{+}_\tau(i,j_k))^2 + 1) \right) \|\Delta\|_2^2 \\
				&\qquad {8 \over (\rho_{4\tau}^{-}(i) )^2}   \max_{\substack{ \Omega \subset [d] \\ |\Omega| = 4\tau}}  \| P_{\Omega}  \nabla f_i(x^*) \|_2^2  + 2\left({6\theta^2} + {15\sqrt{\eta_1^2-1} \over 2\eta_1} \right)  \mathbb{E}_{j_k} \|\nabla g_{i,j_k}(x^*)\|_2^2.
			\end{align*}
		\end{proof}
		
		We follow the idea of the proof for Theorem \ref{thm:main_theorem1} but use Lemma \ref{lem:GradMP_lem3_without_bounded_variance} to show the following convergence theorem of FedGradMP. 
		
		\begin{thm}
			\label{thm:main_theorem3}
			Under the same notations and assumptions but without the bounded variance condition \eqref{asump:local_gradients2}, the expectation of the recovery error at the $(t+1)$-th round of FedGradMP described in Algorithm \ref{alg:FedGradMP} is upper bounded by
			\[
			\mathbb{E} \|x_{t+1} - x^*\|_2^2 
			\le  \kappa^{t+1} \|x_0 - x^*\|_2^2  +  {(2 \eta_3^2 + 2) \nu \over 1-\kappa} \sum_{i=1}^N p_i { 1 - \mu(i)^K   \over 1 - \mu(i)},	
			\]
			where 
			\[
			\kappa = (2 \eta_3^2 + 2)  \sum_{i=1}^N p_i \left[(1 + \eta_2)^2  \beta_1(i)\beta_2(i) \right]^K.
			\]
			Here 
			\begin{align*}
				&\beta_1(i) = { \bar{\rho}^{+(i)}_{4\tau} \over 2\rho^{-}_{4\tau}(i) -  \bar{\rho}^{+(i)}_{4\tau} }, \\
				&\beta_2(i) 
				= \left(2{\left( {\bar{\rho}^{+(i)}_{4\tau} } +  {1 \over \theta^2} \right)  - \eta_1^2 \rho^-_{4\tau}(i)   \over \eta_1^2 \rho_{4\tau}^{-}(i)} + \left({3\theta^2} \mathbb{E}_{j_k} ((\rho^{+}_\tau(i,j_k))^2 + \bar{\rho}^{+(i)}_{4\tau}) \over \rho_{4\tau}^{-}(i) \right)  +{\sqrt{\eta_1^2-1} \over \eta_1} (3 \mathbb{E}_{j_k} \rho^{+}_\tau(i,j_k))^2 + 1) \right), 
			\end{align*}
			\begin{align*}
				&\nu  = 
				(1 + \eta_2)^2  \max_i \left(  {8\beta_1(i) \over (\rho_{4\tau}^{-}(i) )^2}    \right) \sum\limits_{i=1}^N p_i   \max\limits_{\substack{ \Omega \subset [d] \\ |\Omega| = 4\tau}}  \| P_{\Omega}  \nabla f_i(x^*) \|_2^2 \\
				&\qquad+  (1 + \eta_2)^2  \sum\limits_{i=1}^N p_i \left[{\beta_1(i) \over \rho_{4\tau}^{-}(i) }  \left({2\theta^2 } + {6\sqrt{\eta_1^2-1} \over \eta_1} \right) + {4 \over \bar{\rho}^{+(i)}_{4\tau} (2\rho^{-}_{4\tau}(i) - \bar{\rho}^{+(i)}_{4\tau} ) } \right]\sum\limits_{i=1}^N p_i  \mathbb{E}_{j} \| \nabla g_{i,j}(x^*) \|_2^2.
			\end{align*}
		\end{thm}

		\bibliographystyle{plain}
		\bibliography{Federated_Gradient_Matching_Pursuit}

	\end{document}